\PassOptionsToPackage{usenames,dvipsnames}{xcolor}
\documentclass{article}

\usepackage{microtype}
\usepackage{graphicx}
\usepackage{subcaption}
\usepackage{booktabs} %
\usepackage{minitoc}
\usepackage{hyperref}

\usepackage[preprint]{icml2026}

\usepackage{amsmath}
\usepackage{amssymb}
\usepackage{mathtools}
\usepackage{amsthm}

\usepackage[capitalize,noabbrev]{cleveref}

\theoremstyle{plain}
\newtheorem{theorem}{Theorem}[section]
\newtheorem{proposition}[theorem]{Proposition}

\theoremstyle{definition}
\newtheorem{definition}[theorem]{Definition}

\theoremstyle{remark}

\newtheorem{example}[theorem]{Example}

\usepackage[textsize=tiny]{todonotes}

\usepackage{graphicx}
\usepackage{amsthm}
\usepackage{enumitem}
\usepackage{placeins}

\newcommand{\av}[1]{}
\newcommand{\lk}[1]{}
\newcommand{\gm}[1]{}

\usepackage{amsmath,amssymb}
\usepackage{import} %
\usepackage{notation}
\usepackage{stmaryrd}

\usepackage{tikz}
\usepackage{multirow}
\usetikzlibrary{trees}
\usetikzlibrary{decorations.pathreplacing}
\usetikzlibrary{positioning}
\usepackage{researchpack}

\usepackage{algorithm}
\usepackage{algorithmic} %
\newcommand{\RETURN}{\STATE \textbf{return} }
\usepackage{xspace}

\newcommand{\smt}{SMT\xspace}
\newcommand{\lra}{\mathcal{LRA}}
\newcommand{\mapal}{{%
  \ifmmode%
\text{MAP}(\ensuremath{\lra})%
  \else%
MAP(\ensuremath{\lra})\xspace%
  \fi
}}
\newcommand{\smtlra}{%
  \ifmmode
    \text{SMT}(\lra)%
  \else
    SMT(\ensuremath{\lra})\xspace%
  \fi
}
\newcommand{\mpmapclass}{%
  \ifmmode
    \mathcal{MP{-}MAP}%
  \else
    $\mathcal{MP{-}MAP}$\xspace%
  \fi
}

\Crefname{equation}{Eq.}{Eqs.}
\Crefname{figure}{Fig.}{Figs.}
\Crefname{tabular}{Tab.}{Tabs.} 
\Crefname{section}{Sec.}{Secs.}
\Crefname{algorithm}{Alg.}{Algs.} 
\Crefname{example}{Ex.}{Exs.}
\Crefname{definition}{Def.}{Defs.}

\newcommand{\factor}[1]{F_{#1}}
\renewcommand{\neigh}{\ensuremath{\mathsf{neigh}}}
\newcommand{\parent}{\ensuremath{\mathsf{pa}}}
\newcommand{\msg}[3]{\mathsf{m}_{{#1} \rightarrow {#2}}^{#3}}
\DeclareMathOperator*{\argmaxc}{arg\,max\,coord}
\newcommand{\formula}{\ensuremath{\Delta}\xspace}
\newcommand{\formulaExampleOpt}{\ensuremath{\Phi}\xspace}
\newcommand{\formulaExampleMP}{\formula}%
\newcommand{\formulaExampleMPDetailFst}{\formula}%
\newcommand{\formulaExampleMPDetailSnd}{\formula}%

\newcommand{\wfamily}[1]{\ensuremath{\boldsymbol{\Omega}^{#1}}\xspace}
\newcommand{\wfun}{\ensuremath{p}}
\newcommand{\densExampleMP}{\ensuremath{p}}

\newcommand{\varscope}{\ensuremath{\mathcal{S}}}

\newcommand{\mpmap}{\textsc{MpMap}}

\renewcommand{\paragraph}[1]{\vspace{0pt}{\textbf{#1}}}

\newcommand{\append}{\ensuremath{\mathsf{append}}}
\newcommand{\bounds}{\ensuremath{\mathsf{bounds}}}
\newcommand{\breakpoints}{\ensuremath{\mathsf{breakpoints}}}
\newcommand{\computemsgs}{\ensuremath{\mathsf{compute\text{-}msgs}}}
\newcommand{\criticalpoints}{\ensuremath{\mathsf{critical\text{-}points}}}
\newcommand{\dom}{\ensuremath{\mathsf{dom}}}
\newcommand{\dominatingpoly}{\ensuremath{\mathsf{dominating\text{-}poly}}}
\newcommand{\emptypiecewise}{\ensuremath{\mathsf{empty\text{-}piecewise}}}
\newcommand{\exclusiveorinclusivestart}{\ensuremath{\mathsf{exclusive\text{-}or\text{-}inclusive\text{-}start}}}
\newcommand{\extractpolypiece}{\ensuremath{\mathsf{extract\text{-}poly\text{-}piece}}}
\newcommand{\extremepoints}{\ensuremath{\mathsf{extreme\text{-}points}}}
\newcommand{\findsimbolicboundsin}{\ensuremath{\mathsf{find\text{-}symbolic\text{-}bounds\text{-}in}}}
\newcommand{\gathermsgs}{\ensuremath{\mathsf{gather\text{-}msgs}}}
\newcommand{\gatherroot}{\ensuremath{\mathsf{gather\text{-}root}}}
\newcommand{\getfiniteelem}{\ensuremath{\mathsf{get\text{-}finite\text{-}elem}}}
\newcommand{\getmsgpieces}{\ensuremath{\mathsf{get\text{-}msg\text{-}pieces}}}
\newcommand{\highestnvderivativeat}{\ensuremath{\mathsf{highest\text{-}non\text{-}vanishing\text{-}derivative\text{-}at}}}
\newcommand{\innermax}{\ensuremath{\mathsf{inner\text{-}max}}}
\newcommand{\intersection}{\ensuremath{\mathsf{intersection}}}
\newcommand{\intervals}{\ensuremath{\mathsf{intervals}}}
\newcommand{\intervalsfrompoints}{\ensuremath{\mathsf{intervals\text{-}from\text{-}points}}}
\newcommand{\isempty}{\ensuremath{\mathsf{is\text{-}empty}}}
\newcommand{\ispiecewiseexponentianted}{\ensuremath{\mathsf{is\text{-}piecewise\text{-}exponentiated}}}
\newcommand{\maxout}{\ensuremath{\mathsf{max\text{-}out}}}
\newcommand{\maxpieces}{\ensuremath{\mathsf{max\text{-}pieces}}}
\newcommand{\overallbounds}{\ensuremath{\mathsf{overall\text{-}bounds}}}
\newcommand{\piece}{\ensuremath{\mathsf{piece}}}
\newcommand{\preparebreaks}{\ensuremath{\mathsf{prepare\text{-}breaks}}}
\newcommand{\randompolyunivar}{\ensuremath{\mathsf{random\text{-}poly\text{-}univar}}}
\newcommand{\roots}{\ensuremath{\mathsf{roots}}}
\newcommand{\simplify}{\ensuremath{\mathsf{simplify}}}
\newcommand{\topieces}{\ensuremath{\mathsf{to\text{-}pieces}}}

\hypersetup{
colorlinks=true,linkcolor=teal,citecolor=OliveGreen
}

\icmltitlerunning{The Theory and Practice of MAP Inference over Non-Convex Constraints}

\begin{document}

\twocolumn[
\icmltitle{The Theory and Practice of MAP Inference over Non-Convex Constraints}

\icmlsetsymbol{equal}{*}

\begin{icmlauthorlist}
\icmlauthor{Leander Kurscheidt}{equal,edi}
\icmlauthor{Gabriele Masina}{equal,trento}
\icmlauthor{Roberto Sebastiani}{trento}
\icmlauthor{Antonio Vergari}{edi}
\end{icmlauthorlist}

\icmlaffiliation{edi}{School of Informatics, University of Edinburgh, UK}
\icmlaffiliation{trento}{DISI, University of Trento, Trento, Italy}

\icmlcorrespondingauthor{Leander Kurscheidt}{l.kurscheidt@sms.ed.ac.uk}
\icmlcorrespondingauthor{Gabriele Masina}{gabriele.masina@unitn.it}

\icmlkeywords{Machine Learning, ICML}

\vskip 0.3in
]

\doparttoc %
\faketableofcontents %

\printAffiliationsAndNotice{\icmlEqualContribution} %

\begin{abstract}

    In many safety-critical settings,  probabilistic ML systems have to make predictions subject to algebraic constraints, e.g., predicting the most likely trajectory that does not cross obstacles.
    These real-world constraints are rarely convex, nor the densities considered are (log-)concave.
    This makes computing this constrained maximum a posteriori (MAP) prediction efficiently and reliably  extremely challenging.
    In this paper, we first investigate under which conditions we can perform constrained MAP inference over continuous variables exactly and efficiently and devise a scalable message-passing algorithm for this tractable fragment.
    Then, we devise a general constrained MAP strategy that interleaves partitioning the domain into convex feasible regions with numerical constrained optimization.
    We evaluate both methods on synthetic and real-world benchmarks, showing our %
    approaches outperform constraint-agnostic baselines, and scale to complex densities intractable for SoTA exact solvers.

\end{abstract}

\section{Intro}

Making predictions with probabilistic machine learning (ML) models can be mapped to performing \textit{maximum a posteriori} (MAP; \citealt{bishop2006pattern}) \textit{inference}, i.e., computing the output configuration with the highest likelihood according to the distribution learned by the model.
However, in several real-world scenarios, from 
physics applications~\cite{hansen2023learning, cheng2024gradient} to fair predictions \citep{ghandi2024probabilistic} and ``what-if'' time-series analysis~\cite{narasimhan2024time},
distributions are  \textit{constrained}, i.e., some configurations are infeasible and  should never be predicted nor sampled \citep{grivas2024taming}.
This is mandatory if ML models are deployed in safety-critical scenarios \citep{giunchiglia2023road,bortolotti2024neuro}.

\begin{figure}[t!]
    \centering
    \begin{subfigure}[T]{0.41\linewidth}
        \includegraphics[width=\linewidth]{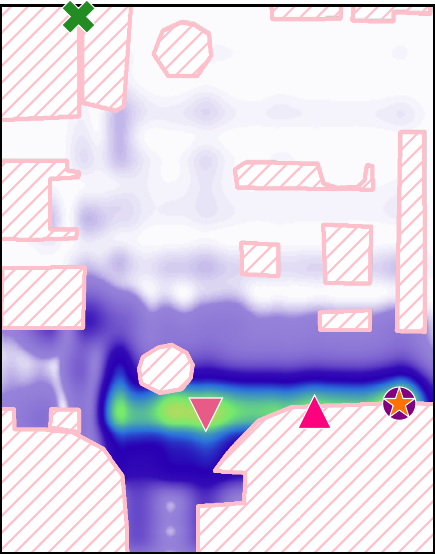}
    \end{subfigure}
    \hfill
    \begin{subfigure}[T]{0.57\linewidth}
        \includegraphics[width=\linewidth]{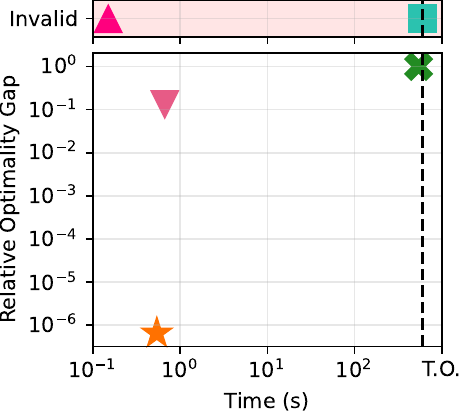}
    \end{subfigure}
    
    \vspace{1mm}
    \begin{subfigure}{\linewidth}
        \includegraphics[width=\linewidth]{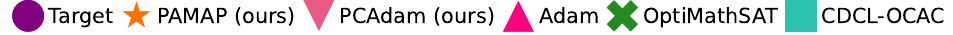}
    \end{subfigure}
    \caption{
        \textbf{{Our structure-aware solver \pamap{} is able to correctly and efficiently perform MAP inference over non-convex constraints and non-log-concave densities}} while classical optimizers (Adam), even when being constrained-aware (\pcadam{}; \cref{sec:experiments}), are imprecise and slower and when exact solvers  (OptiMathSAT \citep{sebastiani2020optimathsat} and CDCL-OCAC~\cite{jia2025complete}) timeout.
        If the density factorizes as a tree, our \mpmap{} solver (\cref{sec:mp-map}) can be exact and even faster (see \cref{fig:mp_map_timeout_overview}).
    }%
    \label{fig:intro-ex}
\end{figure}

Constrained MAP inference is  well understood from an optimization perspective when it comes to discrete variables \citep{marinescu2004and,martins2011augmented} or when  the problem has a simple form, i.e., constraints are convex \citep{dantzig2002linear,jaggi2013revisiting} and distributions are log-concave \citep{doss2019inference}, as we discuss in \cref{sec:related}.
However, these assumptions are generally not met in real-world applications \citep{de2023neural,kurscheidt2025probabilistic,stoianbeyond} and understanding how to perform efficient MAP inference over non-convex constraints and non-log-concave distributions is an open and challenging problem.
In this paper, we reduce this gap by advancing a number of contributions, discussed next.

First, \textbf{C1)} we theoretically trace a non-trivial fragment of tractable constrained MAP problems over non-convex constraints and distributions represented as tree-factorized piecewise (exponentiated) polynomials.
We then prove that it can be solved by an efficient message passing scheme (\mpmap; \cref{sec:mp-map}).
Our \mpmap{} is inspired by message passing schemes to compute the probability of non-convex constraints \citep{ZengICML20,zeng2020probabilistic}, but differently from them, performing MAP inference yields different challenges and different complexity results.

Second, \textbf{C2)} we investigate 
how to approximate constrained MAP via optimization for general constraints and distributions for which \mpmap{} is not applicable.
To this end, we design \pamap{}, a general scheme that decomposes global optimization into a series of MAP inference problems over convex constraints which can be efficiently solved by calling local optimizers and can provide approximation guarantees for polynomial densities \citep{powers1998algorithm,lasserre2001global}.
Lastly, \textbf{C3)} we rigorously evaluate \mpmap{} and \pamap{} over a set of synthetic and real-world benchmarks, reporting  that our custom optimizers, are able to outperform a number of SoTA optimizers \citep{sebastiani2020optimathsat,jia2025complete} both in terms of approximation quality and time to achieve it (see \cref{fig:intro-ex}).
As such, we set the first milestone to tackle the challenging problem of constrained optimization from both theory and practice.

\section{Maximum A Posteriori Inference Under Non-Convex Algebraic Constraints}

\paragraph{Notation.}
We denote random variables by uppercase letters (e.g., $X, Y$), and their assignments with lowercase ones (e.g., $x, y$). Bold symbols denote sets of variables (e.g., $\vX, \vY$), and their joint assignments (e.g., $\vx, \vy$). Greek letters such as $\formula, \formulaExampleOpt$ or $\formulaExampleMP$ denote logical formulas that map real values to binary values (false, true). We say that assignment $\vx$ satisfies the constraint $\formula$, and denote it as $\vx \models \formula$, if substituting $\vx$ into $\formula$ makes $\formula$ true. 
So, the indicator function $\id{\vx \models \formula}$ is 1 when $\vx$ satisfies $\formula$, 0 otherwise.

\begin{figure}[t!]
    \centering
    \begin{subfigure}{0.32\linewidth}
        \includegraphics[width=\linewidth]{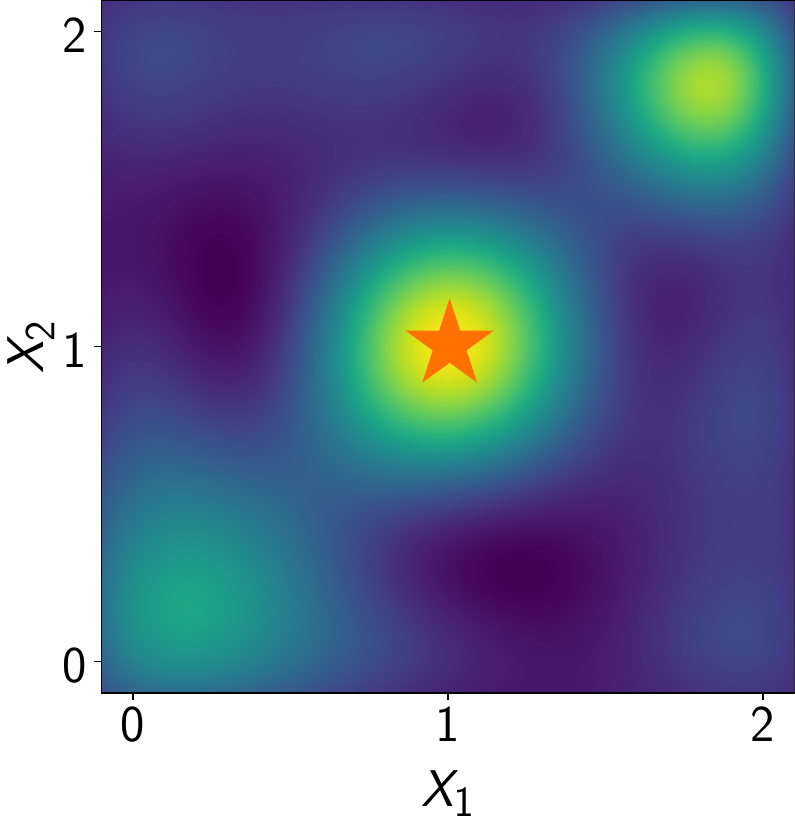}
    \end{subfigure}
    \begin{subfigure}{0.32\linewidth}
        \includegraphics[width=\linewidth]{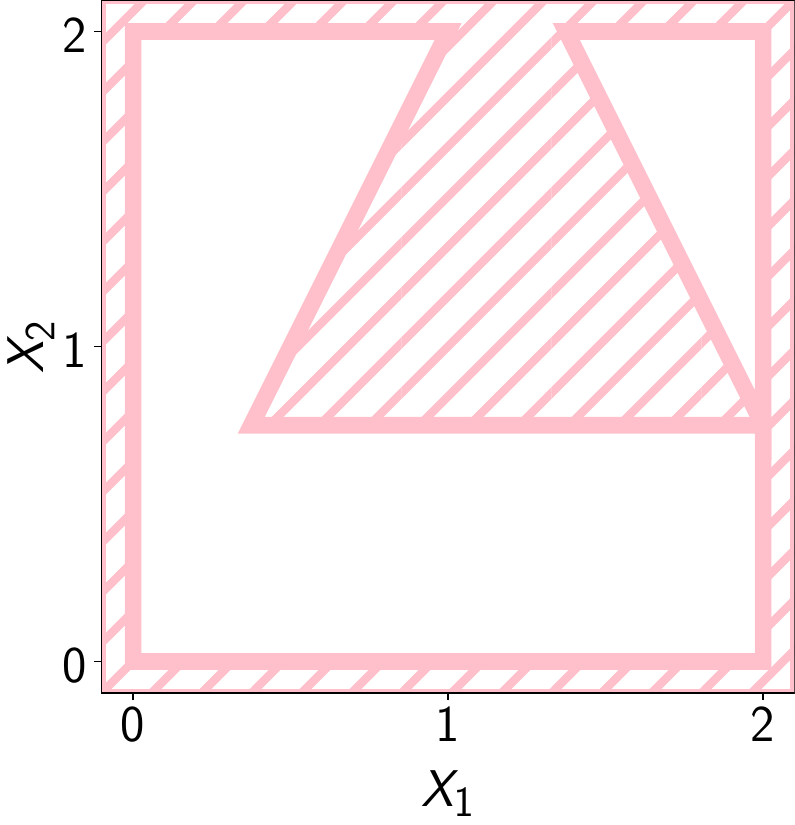}
    \end{subfigure}
    \begin{subfigure}{0.32\linewidth}
        \includegraphics[width=\linewidth]{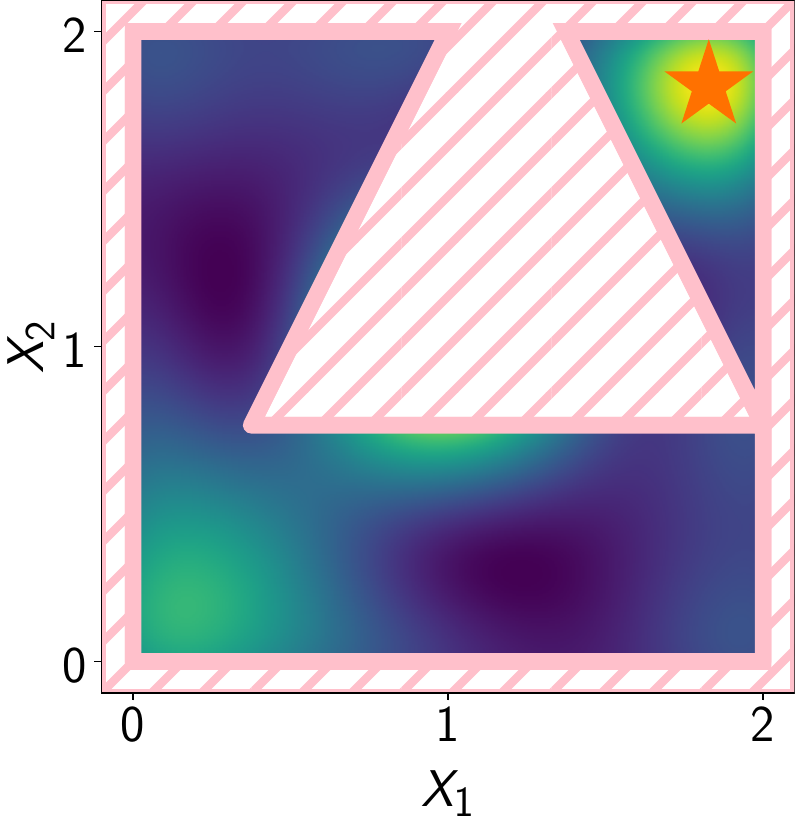}
    \end{subfigure}
    \caption{\textbf{An example of \ref{eq:constrained-map}{} inference} over non-convex constraints and non-log-concave density. Left: an unconstrained density in 2D.  Center: non-convex constraints. 
    Right: constrained density. Orange stars indicate the solutions for the MAP problem for the unconstrained and constrained densities.
    }
    \label{fig:intro-ex-cmap}
\end{figure}

\paragraph{SMT($\Tlarat$) Constraints.}
We consider algebraic constraints over $\vX$ representing a collection of non-convex polytopes.
We express them in the language of \textit{satisfiability modulo theory over linear real arithmetic} (SMT($\Tlarat$); \citealt{barrett2021satisfiability}), hereafter just SMT for short.
We consider quantifier-free SMT-formulas over continuous variables $\vX$, where linear (in)equalities $\sum_i a_i X_i \bowtie{} b$, with $\bowtie\,\in\{\le, =\}$, are connected via Boolean connectives ---i.e., conjunctions ($\wedge$), disjunctions ($\vee$), and negations ($\neg$).
As we show in \cref{fig:intro-ex}, SMT is flexible enough to represent rich real-world constraints.
We now provide a simpler example.

\begin{example}[SMT formula]
    Consider the following SMT formula over  variables $\vX=\{X_1, X_2\}$:
    \begin{align*}
    \formula(\vX)&=X_1\in[0,2]\wedge X_2\in[0,2]\wedge\\
        &\left(X_2\le 1\vee X_2> 2 X_1 \vee X_2 > 4.75 -2 X_1\right).
    \end{align*}
It denotes a square feasible area with an inner infeasible triangle area, as shown in \cref{fig:intro-ex-cmap} (center).
\label{ex:smt}
\end{example}

\paragraph{MAP under SMT constraints.}
Given a probability density function $p(\vX)$ over continuous random variables $\vX = \{X_1, \ldots, X_n\}$ and some SMT constraints $\formula$ over $\vX$, the goal of constrained MAP inference is to find the assignment $\vx^*$ that maximizes $p(\vX)$ while satisfying $\formula$:
\begin{equation}%
    \tag{\mapal}
    \label{eq:constrained-map}
    \argmax_{\vx\models\formula}{p(\vx)} = \argmax_{\vx}\widetilde{p}(\vx)\id{\vx\models\formula}
\end{equation}
where $\widetilde{p}$ can be an unnormalized density, i.e., a nonnegative function. 
Clearly, the solution to \ref{eq:constrained-map} can drastically differ from the one for MAP inference over an unconstrained distribution.
\begin{example}[Constrained MAP]
    Consider the SMT formula of \cref{ex:smt}  and the unnormalized polynomial shown in \cref{fig:intro-ex-cmap} (left). 
    The unconstrained optimum is $(1.0, 1.0)$, but the constrained one is $(1.83, 1.83)$, shown on the right.
    Note also that the new optimum is in a disconnected polytope, highlighting how the optimization problem for \ref{eq:constrained-map} has a combinatorial nature that makes it challenging for traditional continuous optimizers.
\end{example}

Examples of \ref{eq:constrained-map} problems can be found in several recent works where SMT constraints are defined over the output of a neural network \citep{de2023neural,
kurscheidt2025probabilistic}.
As in those scenarios we deal with conditional distributions, i.e., we want to solve $\argmax_{\vy\models\formula}{p(\vy\mid\vx)}$, we have to solve one \ref{eq:constrained-map} problem for each datapoint $\vx$.
This motivates us to find fast and reliable solvers
that can  be parallelized. 
We review next, when it is already known how to solve \ref{eq:constrained-map} for simple forms of distribution $p$ and constraints $\formula$.

\section{The SoTA of Constrained MAP Inference}%
\label{sec:related}

Solving \ref{eq:constrained-map} exactly is generally computationally hard, as even optimizing a quadratic function under linear constraints is NP-hard~\cite{sahni1974computationally}.
However, efficient methods for specific problem classes exist, each one trading off expressiveness, efficiency, and optimality guarantees in a different way, as we review next.

\paragraph{Convex constraints, log-concave densities.}
If the constraint $\formula$ is \textit{convex}, e.g., it is defined as a \emph{conjunction} of linear constraints, and the objective is \textit{(log-)concave}, any local optimum is also a global optimum. This property ensures that algorithms such as sequential quadratic programming~\cite{boggs1995sequential} and interior-point methods~\cite{conn2000trustregion} are guaranteed to efficiently converge to the global optimum.

\paragraph{Convex constraints and piecewise densities.}
When the density decomposes in several pieces of simple form, one can go beyond log-concavity and use symbolic bucked elimination for \textit{constrained piecewise functions} as in \citet{10.1007/978-3-319-93031-2_42}. It tackles the problem of optimizing a sum of piecewise linear or univariate quadratic (LUQF) functions, i.e., only one variable can be quadratic in each piece, under convex, linear constraints. 
The approach is based on the XADD-data structure \cite{10.5555/1642293.1642513} and leverages the symbolic, partial maximization algorithm of \citet{Zamani_Sanner_Fang_2021}. 
A similar line of work by \citet{10.1007/978-3-031-33271-5_6} deals with mixed-integer linear programming over convex, and linear constraints which decompose into two sets of constraints over disjoint sets of variables. 

\paragraph{General polynomial densities.}
For non-convex problems, exact global constrained optimization is very challenging, even 
if $p$ is a polynomial.
For instance, the \textit{cylindrical algebraic decomposition} (CAD)~\cite{collins1975quantifier,arnon1984cylindrical,wolfram2025exact}  has doubly-exponential worst-case complexity in the number of variables.
A classical way to approximate 
\ref{eq:constrained-map} with guarantees for polynomials, comes from 
\citet{lasserre2001global} where a hierarchy of semidefinite programming relaxations is used.
As the relaxation order increases, the solution provides increasingly tight upper bounds, converging to the global optimum in the limit. With linear constraints, it is possible to compute a lower bound on the global maximum.
The dual formulation of the problem can be solved via \textit{sum-of-squares} optimization (SoS), which also provides a sequence of improving upper bounds converging to the global maximum.

\paragraph{Local optimizers for general densities and convex constraints.}
A practical way to approximate \ref{eq:constrained-map} when $\Delta$ is convex, is to run a local optimizer starting from multiple initial points (particles), to increase the chances of finding the global optimum.
E.g., Basin Hopping~\cite{wales1997global} combines local minimization with stochastic perturbations of the current solution, whereas \shgo~\cite{endresSimplicialHomologyAlgorithm2018} exploits a simplicial complex to identify locally convex subdomains corresponding to distinct minima, so to systematically explore of the objective landscape.
These methods scale well with problem dimensionality and make minimal assumptions about the shape of the constraints or objective function ---typically requiring only smoothness or Lipschitz continuity--- but provide no guarantees of finding the global optimum in the general case.
We leverage these optimizers in our \pamap{} in \cref{sec:pamap} when we decompose a global constraint into convex polytopes.

\paragraph{Optimization under general SMT constraints.}
Optimization under SMT constraints is known in the literature as \textit{optimization modulo theories} (OMT;~\citealt{nieuwenhuis2006sat,sebastiani2012optimization,sebastiani2015optimization}).
In OMT, however, both the objective function (our density) and the constraint, need to belong to the same family.
For linear constraints and  objectives, OMT can be solved efficiently using an SMT solver extended with optimization capabilities~\cite{sebastiani2012optimization,li2014symbolic,bjorner2015nz}.
The extension to non-linear polynomial constraints and objective functions, however, is non-trivial and remains an active area of research.
OptiMathSAT~\cite{sebastiani2020optimathsat} uses an incomplete approach, solving a linear overapproximation of the original non-linear problem which is iteratively refined~\cite{bigarella2021optimization}.
Recently, \citet{jia2025complete} proposed CDCL-OCAC, a complete algorithm based on a variant of CAD, hence suffering from the same scalability limitations.
Finally, we remark that the extension of common constrained optimization methods ---such as penalty or barrier methods \citep{nocedal2006numerical}--- to SMT constraints with a complex Boolean structure is not straightforward. Indeed, to the best of our knowledge, no such extension exists in the literature, making these methods not directly applicable to \ref{eq:constrained-map}.

\section{A Scalable Message-Passing Algorithm For Constrained MAP}
\label{sec:mp-map}

As discussed in the previous section, tractable fragments of \ref{eq:constrained-map} are restricted to 
convex constraints and piecewise linear or univariate quadratic functions \citep{10.1007/978-3-319-93031-2_42}. 
As a first contribution \textbf{C1)}, we now relax these requirements to cover a larger class of tractable constrained MAP problems: those involving non (log-)concave densities  and constraints factorizing as a tree.
We do so by building on prior work that enables exact integration over algebraic constraints, also known as weighted model integration (WMI)~\citep{belle2015probabilistic,morettin-wmi-ijcar17}, via a message-passing scheme~\citep{ZengICML20,zeng2020probabilistic}.
While our approach takes inspiration from WMI, we show that the class of tractable densities for WMI is not suitable for tractable \ref{eq:constrained-map}.

\paragraph{A tree-shaped \mapal.} 
Similarly to \citet{ZengICML20,zeng2020probabilistic}, we consider problems where the product between the density and constraints $\id{\vx\models\formula}\cdot p(\vx)$ decomposes with a tree-shaped graph structure. %
We first start with the SMT formula $\formula$, which we assume is in conjunctive normal form (CNF) and with at most bivariate clauses. As a consequence, the indicator function factorizes as
\[
    \!\! \llbracket \vx \models \formula \rrbracket %
    = \mkern-15mu \prod_{X_i,X_j \in \mathcal{E}_{\formula}} \mkern-5mu \llbracket (x_i,x_j) \models \formula_{ij} \rrbracket \prod_{v \in V_{\formula}} \llbracket x_v \models \formula_{v} \rrbracket
    \label{eq:indicator_factorization}
\]
where $\mathcal{E}_{\formula}$ (resp. $V_{\formula}$)  is the set of pairs of variables appearing in
a same bivariate clause (resp. univariate clauses).
clauses, 
and $\formula_\varscope$ is the restriction of $\formula$ to the clauses over the variables in $\varscope$. 
Furthermore, we require the graph $(\vX, \mathcal{E}_{\formula})$, also called the primal graph of $\formula$, to have a treewidth of one, thus encoding a tree (or a forest), as shown next. 
\begin{example}[Primal graph of SMT formula]
    \label{ex:primal_graph_smt}
    The following SMT formula over variables $\vX=\{X_1, X_2, X_3\}$ (left) exhibits a tree-shaped primal graph $(\vX, \mathcal{E}_{\formulaExampleMP})$ (right):
    \begin{minipage}{0.48\columnwidth}
        \begin{equation*}
            \begin{aligned}
                \formulaExampleMPDetailFst_i  & = |X_i| \leq 1                                                  \\
                \formulaExampleMPDetailSnd_{1i} & = 1 \leq |X_1 - X_i| \leq 2                                     \\
                \formulaExampleMP    & = \formulaExampleMPDetailSnd_{12} \land \formulaExampleMPDetailSnd_{13} \land \bigwedge\nolimits_{i=1}^3 \formulaExampleMPDetailFst_i
            \end{aligned}
        \end{equation*}
    \end{minipage}\hfill
    \begin{minipage}{0.48\columnwidth}
    \raggedright
        \hspace{3em} \begin{tikzpicture}[sibling distance=10mm, level distance=10mm,
                every node/.style={circle, draw},grow=right, line width=1pt]
            \node {$X_1$}
            child { node {$X_2$} }
            child { node {$X_3$} };
        \end{tikzpicture}
    \end{minipage}
\end{example}
Additionally, a similar structure is assumed for $p$, which we also assume to factorize into at most bivariate components:
\begin{equation}
    p(\vx)=\prod_{X_i,X_j \in \mathcal{E}_{p}} \wfun_{ij}(x_i,x_j) \prod_{X_i \in \vX} \wfun_i(x_i)
    \label{eq:density_factorization}
\end{equation}
Here, the set $\mathcal{E}_{p}$ denotes the set of variable pairs appearing in the domains of the bivariate functions $\wfun_{ij}$. 
We also assume the graph $(\vX, \mathcal{E}_{p})$ to have a treewidth of one.
This is not enough however, in order to tractably compute \mapal, we need a few tractable operations and properties over the functions $\wfun_{ij}$ and $\wfun_{i}$:%

\begin{definition}[Tractable MAP Conditions]
\label{def:tractable_mp_conditions}
We say that the \emph{tractable MAP conditions} (TMC) hold for a family of functions $\wfamily{}$ if we have:
\begin{enumerate}[label=(\roman*),noitemsep,topsep=0pt]
    \item \textbf{\textit{Closedness under product}:} $\forall f, g \in \wfamily{}:f \cdot g \in \wfamily{}$;
    
    \item\label{item:tractable-supremum} \textbf{\textit{Tractable symbolic supremum}:} For any bivariate $f \in \wfamily{}$ and bounds $l(x_i), u(x_i)$ in $\lra$,
    $
        m(x_i) \coloneq (\arg)\sup_{x_j \in [l(x_i),u(x_i)]} f(x_i,x_j)
    $
    belongs to $\wfamily{}$ and can be computed tractably; %
    
    \item \textbf{\textit{Tractable pointwise maximum}:} For any univariate $f,g \in \wfamily{}$,
    $
        o(x) \coloneq (\arg)\max\{f(x),g(x)\}
    $
    belongs to $\wfamily{}$ and can be computed tractably. 
\end{enumerate}
\end{definition}

We note that the function family identified by \citet{ZengICML20} that enables tractable integration over SMT formulas (tractable WMI), does not necessarily satisfy our TMC conditions.
For example, general piecewise polynomials enable tractable WMI, but can violate property \ref{item:tractable-supremum} of \cref{def:tractable_mp_conditions}, as we discuss in \cref{app:incomparability-wmi-map}. 
To pinpoint a function class satisfying \cref{def:tractable_mp_conditions}, we have to add further properties to polynomials.
To this end, we identify two major families of functions for which TMC holds: $\wfamily{\mathsf{PP}}$ and $\wfamily{\mathsf{PEP}}$.
$\wfamily{\mathsf{PP}}$ is the family of piecewise polynomial functions, where the bounds for the finite number of pieces are defined by conjunctions of linear inequalities and each polynomial \textit{factorizes} into a products of two univariate polynomials.
\begin{example}[Example of a density in $\wfamily{\mathsf{PP}}$]
    \label{ex:example_wfamily_pp}
    The density $p$ over $\vX{=}\{X_1, X_2, X_3\}{\in}[-1,1]^3$ that factorizes as%
\begin{equation*}
\begin{aligned}
    \densExampleMP_{1}(x_1)       &= 0.05
    & \densExampleMP_{2}(x_2)       &= (x_2 + 1)
    \\
    \densExampleMP_{1,3}(x_1,x_3) &= (1 - x_1)(3 - x_3) & \densExampleMP_{3}(x_2)         & = (1 - x_3) \\
    \densExampleMP_{1,2}(x_1,x_2) &= \mathrlap{0.2\cdot(x_1 {-} 0.9)^2 (x_2 {+} 0.9)^2 \id{x_1 {-} x_2 {<} 0}} \\
    \densExampleMP_{1,2}(x_1,x_2) &= \mathrlap{(x_1 {+} 1)  \id{x_1 {-} x_2 {\geq} 0 \wedge x_1 \leq 0.5}} & & 
\end{aligned}
\end{equation*}
\end{example}
\begin{figure}[!h]
    \begin{minipage}{.55\linewidth}
        \includegraphics[width=.9\linewidth]{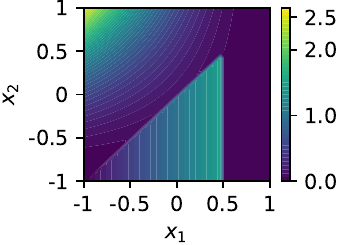}
    \end{minipage}\hfill\begin{minipage}{.44\linewidth}
            \caption{
            \textbf{Densities in $\wfamily{\mathsf{PP}}$ can express complex and multimodal densities} as shown here for the density  $\densExampleMP_{1,2}(x_1,x_2)$ from example \ref{ex:example_wfamily_pp}. Despite factorizing into univariate polynomials, different pieces can recover correlations.
    }
    \end{minipage}
    \label{fig:mp_example_dens_p12}
\end{figure}

Similarly, $\wfamily{\mathsf{PEP}}$ is the family of piecewise exponentiated polynomials, 
where each piece factorizes into a product of exponentiated univariate polynomials. 
A prominent example of $\wfamily{\mathsf{PEP}}$ is the multivariate Gaussian with independent components.
As such, our TMC conditions cover functions that do not support tractable integration according to \citet{ZengICML20}, hinting to the fact that WMI and \ref{eq:constrained-map} are incomparable problems (see \cref{app:incomparability-wmi-map}).

To properly define a tractable fragment of \ref{eq:constrained-map}, we need to enforce the global problem structure, defined next, to be a tree. 
That, in turn, leads to our first main result.
\begin{definition}
\label{def:mp_map_graph}
The global structure of a \mapal{} problem over an SMT formula $\formula$ and density $p$ factorizing as \eqref{eq:indicator_factorization} and \eqref{eq:density_factorization}
is the graph $\mathcal{G}=(\vX,\mathcal{E})$ with $\mathcal{E} \coloneq \mathcal{E}_{\formula} \cup \mathcal{E}_{p}$.
\end{definition}

\begin{theorem}[Tractability of \mapal{}]
\label{th:mp_tractability}
If the global graph of \mapal{} has treewidth one and bounded diameter, and the density fulfills the TMC (\cref{def:tractable_mp_conditions}), then \mapal{} can be solved tractably.
\end{theorem}

The proof is by construction and detailed in \cref{app:mp-map-algo}. 
It proceeds by building a fixed-parameter tractable message-passing algorithm, which we discuss next.

\paragraph{\mpmap.}
The key idea behind our message-passing algorithm for constrained MAP (\mpmap) is to exploit the tree structure of the global graph. 
By iteratively conditioning on a variable node, we render its children independent, and allow maximization to be performed independently over each sub-tree. 
We can therefore decompose the computation of the overall maximum into smaller problems until we arrive at the leaves.
We start by reordering $\llbracket \vx \models \formula \rrbracket \cdot p(\vx)$ by grouping them via $\mathcal{E}$ and introduce the factor representation:
\[
    \factor{\varscope}(\vx_\varscope) \coloneq \llbracket \vx_{\varscope} \models \formula_\varscope \rrbracket \cdot \wfun_S(\vx_S)
\]
with the scope $\varscope$ being over both variables in edges in $\mathcal{E}$ (e.g.\ $\factor{13}$) and single variable indices in $\vX$ (e.g.\ $\factor{2}$).
We set $\wfun_\varscope$ to $1$ and $\formula_\varscope$ to True if not previously defined in \formula and $p$. This results in $\llbracket \vx \models \formula \rrbracket \cdot p(\vx) = \prod_\varscope \factor{\varscope}(\vx_\varscope)$.

If we manage to ``maximize out'' (akin to marginalizing out) the variables one by one, e.g., in the example calculate $\max_{x_i} \factor{i1}(x_i, x_1) \factor{i}(x_i)$ as a function of $x_1$, we have a scalable algorithm to compute exactly the constrained MAP even for high-dimensional problems. 
This can be extended to the argmax by not only tracking the value but also the position at which the value is attained and then backtrack, detailed in the \cref{app:mp-map-maxout}. 
We start from a root node $X_r$, that can be chosen so to minimize the number of computations, and then recursing into the children of the directed factor graph until we hit the leaves: %
\begin{align}
\msg{X_i}{\factor{\varscope}}{}(x_i) &\coloneq \prod_{c \in \mathsf{child}(X_i)}\msg{\factor{ci}}{X_i}{}(x_i) \cdot \factor{i}(x_i)\label{eq:msg-var-to-fac} \\
\msg{\factor{ij}}{X_j}{}(x_j) &\coloneq \max_{x_i} \factor{ij}(x_i, x_j) \cdot \msg{X_i}{\factor{ij}}{}(x_i) \label{eq:msg-fac-to-var} \\
\max_{\vx} \prod_\varscope \factor{\varscope}(\vx_\varscope) &= \max_{x_{r}} \factor{r}(x_r) \cdot \!\!\!\!\!\!\! \prod_{c \in \mathsf{child}(X_r)} \!\!\!\!\!\!\! \msg{\factor{cr}}{X_r}{}(x_r) \label{eq:msg-final}
\end{align}
\cref{alg:mp_map} shows the pseudo-code for the message passing, with the procedure computing message \cref{eq:msg-var-to-fac} in \cref{alg:mp_map_gather} and \cref{eq:msg-fac-to-var} in \cref{alg:mp_map_compute_generic}.
Next,
we illustrate one run of \mpmap{} by considering \cref{ex:example_wfamily_pp} restricted to the formula in \cref{ex:primal_graph_smt}.
\begin{example}[\mpmap{} in action]
    \label{ex:mp_max_via_factor_rep}
    Consider the factors for \cref{ex:example_wfamily_pp}, and $\formula$ from \cref{ex:primal_graph_smt}. \mpmap{} computes the following operations, leading to the computational graph below.
\begin{align*}
    & \max_{\vx} (\factor{3}(x_3)\factor{21}(x_2, x_1) \factor{2}(x_2))(\factor{31}(x_3,x_1) \factor{3}(x_3))                                                                                                               \\
                      & = \max_{x_1} \frac{9}{100} \prod_{i\in \{2,3\}} \underbrace{\max_{x_i} \factor{i1}(x_i, x_1) \underbrace{\factor{i}(x_i)}_
                      {=\msg{\factor{X_i}}{\factor{i1}}{}(x_1)}
                      }_{=\msg{\factor{i1}}{X_1}{}(x_1)}
\end{align*}
\scalebox{.9}{
    \begin{tikzpicture}[sibling distance=21mm, level distance=16.5mm,
            var/.style={circle, draw},
            fac/.style={rectangle, draw},
            facimg/.style={inner sep=0pt, anchor=center}, %
            grow=right,
            edge from parent/.style={draw, <-, >=latex}, %
            line width=1pt
        ]
        \node[facimg] (finalmsg) {
            \includegraphics[height=14mm]{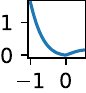}
        }
        child{
                node[var] {$X_1$}
                child {
                        node[facimg] (x2tox1){
                                \includegraphics[height=14mm]{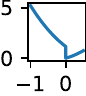}
                            }
                        child {
                                node[var] {$X_2$}
                                child { node[facimg] (margx2) {
                                                \includegraphics[height=14mm]{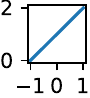}
                                            }
                                        node[above=-0.5mm of margx2] { $\factor{2}$}
                                    }
                            }
                        node[above=-0.5mm of x2tox1] { $\msg{\factor{21}}{X_1}{}$}
                    }
                child {
                        node[facimg] (x3tox1) {
                                \includegraphics[height=14mm]{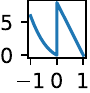}
                            }
                        child {
                                node[var] {$X_3$}
                                child { node[facimg] (margx1) {
                                                \includegraphics[height=14mm]{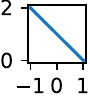}
                                            }
                                        node[above=-0.5mm of margx1] { $\factor{3}{}$}
                                    }
                            }
                        node[above=-0.5mm of x3tox1] { $\msg{\factor{31}}{X_1}{}$}
                    }
            node[above=-0.5mm of finalmsg] {final message}
            };
    \end{tikzpicture}
}
\end{example}

\begin{algorithm}[!t]
    \caption{\textbf{\mpmap}($\formula$, $p$)
    }
    \label{alg:mp_map}
    \begin{algorithmic}[1]
        \INPUT $\formula$: SMT formula as in \eqref{eq:indicator_factorization}, $p$ density as in \eqref{eq:density_factorization}%
        \OUTPUT $m^*$: max.\ density, $\vx^*$: the coordinates
        \STATE $\vV_{\mathsf{up}} \leftarrow$ sort variable nodes according to $\mathcal{G}$
        \FOR[upward pass]{\textbf{each} $X_{i}\in\vV_{\mathsf{up}}$}
        \STATE \gathermsgs($X_i$,$\factor{i, \parent(i)}$) \COMMENT{$\msg{X_i}{\factor{i, \parent(i)}}{}$}%
        \STATE \computemsgs($\factor{i, \parent(i)}$, $X_{\parent(i)}$) %
        \COMMENT{$\msg{\factor{i, \parent(i)}}{X_{\parent(i)}}{}$}
        \ENDFOR
        \STATE $m \gets \gatherroot(X_\mathit{roots})$ \COMMENT{final message}
        \RETURN $\max_x m(x), \argmaxc_x m(x)$
    \end{algorithmic}
\end{algorithm}

\begin{algorithm}[!t]
    \caption{$\bm{\gathermsgs}$($X_i$,$\factor{i, j}$)
    }
    \label{alg:mp_map_gather}
    \begin{algorithmic}[1]
        \INPUT $X_i$: variable, $\factor{i, j}$: factor
        \OUTPUT $\msg{X_i}{\factor{ij}}{}$: message
        \STATE $Q \gets \{ \msg{\factor{j',i}}{X_i}{} \forall j' {\in} \neigh(i): j' {\neq} j\} \cup \{ \msg{\factor{i}}{X_i}{} \}$\\
        \RETURN $\prod_i Q_i$ \COMMENT{point-wise product}
    \end{algorithmic}
\end{algorithm}

The main challenge in \mpmap{} is how to calculate $\msg{\factor{ij}}{X_j}{}(x_j)$, in particular $\max_{x_i} \factor{ij}(x_i, x_j) \cdot \msg{X_i}{\factor{ij}}{}(x_i)=\llbracket \vx_{\varscope} \models \formula_\varscope \rrbracket \cdot f_{\wfamily{}}(x_\varscope)$ with $f_{\wfamily{}} \in \wfamily{}$ being the product between $p_\varscope(\vx_\varscope)$ and the incoming messages. 
First, we reduce the problem to a a number of symbolic maximas over linear upper and lower bounds which we derive similarly to \citet{ZengICML20} (lines \ref{line:mp_map_compute_generic:P}-\ref{line:mp_map_compute_generic:I} in \cref{alg:mp_map_compute_generic}).
The second step is challenging, as symbolic maximization over linear bounds is unexplored. We identify a way to calculate this function explicitly for univariate polynomials (detailed in \cref{app:mp-map-maxout}), and reduce our function class to this operation via $\max_{x_i \in I(x_j)}f_1(x_j)f_2(x_i)=f_1(x_j)\max_{x_i \in I(x_j)}f_2(x_i)$ with $I=[l(x_j),u(x_x)]$, thanks to the properties of our function classes $\wfamily{\mathsf{PP}}$ and $\wfamily{\mathsf{PEP}}$.
Finally, we note maximising a univariate polynomial in $\wfamily{\mathsf{PP}}$ and $\wfamily{\mathsf{PEP}}$ can be done with different complexities, depending on the function family as we discuss in \cref{app:mp-map-complexity-wfamilies}.

\begin{algorithm}[!t]
    \caption{$\bm\computemsgs$($\factor{i, j}$, $X_{j}$)
    }%
    \label{alg:mp_map_compute_generic}
    \begin{algorithmic}[1]
        \INPUT $\factor{i, j}$: factor, $X_{j}$: variable
        \OUTPUT $\msg{\factor{ij}}{X_j}{}$: message
        \STATE $\msg{\factor{ij}}{X_j}{} \gets \emptypiecewise()$ %
        \STATE $\mathcal{P} \leftarrow \criticalpoints(\overallbounds(\msg{X_i}{\factor{ij}}{}), \formula_{ij})$\alglinelabel{line:mp_map_compute_generic:P}
        \STATE $\mathcal{I} \leftarrow \intervalsfrompoints(\mathcal{P})$\alglinelabel{line:mp_map_compute_generic:I}
        \FOR{interval $I \in \mathcal{I}$ consistent with formula $\formula_{ij}$}
        \STATE $ \mathcal{Q} \leftarrow$
        \getmsgpieces($\msg{X_j}{\factor{ij}}{}, I, p_{ij}$)\alglinelabel{line:mp_map_compute_generic:get_pieces}\\
        \COMMENT{enumerates every piece that falls into $X_i \in I$}
        \STATE $q^\prime_{h} \leftarrow \max_{x_i}(q_h, l_h, u_h) \forall (l_h, u_h, q_h) {\in} \mathcal{Q}$\alglinelabel{line:mp_map_compute_generic:max_out}
        \STATE $\msg{\factor{ij}}{X_j}{}|_{x_i \in I} \leftarrow \maxpieces(\{q^\prime_{h}, 0 {\leq} h {<} |Q|\})$\alglinelabel{line:mp_map_compute_generic:max-pieces}
        \ENDFOR
        \RETURN $\msg{\factor{ij}}{X_j}{}$
    \end{algorithmic}
\end{algorithm}

\section{\pamap: scaling local convex optimizers}%
\label{sec:pamap}

While our TMC enable tractable \ref{eq:constrained-map},
they might be too restrictive for certain applications.
As such, we now introduce a practical algorithm for approximating \ref{eq:constrained-map} (\textbf{C2}) over arbitrary \smtlra{} formula, and a density for which an optimization algorithm constrained to a \emph{convex polytope} is available.
\newcommand{\cpenum}{\ensuremath{\mathsf{cpenum}}}
\newcommand{\cpmax}{\ensuremath{\mathsf{cpopt}}}
\begin{algorithm}[!t]
    \caption{\textbf{\pamap}$(\formula, p, \cpenum{}, \cpmax{})$}%
    \label{alg:pa_map}
    \begin{algorithmic}[1]
        \INPUT $\formula$: \smt{} formula, $p$: density, $\cpenum{}$: convex polytope enumerator, $\cpmax{}$: convex polytope optimizer 
        \OUTPUT $m^*$: max.\ density found, $\vx^*$: the coordinates
        \STATE $m^* \gets -\infty$ \COMMENT{current lower bound}\alglinelabel{line:pa_map:lower-bound}
        \STATE $\vx^*\ \gets \emptyset$ \COMMENT{current best point}\alglinelabel{line:pa_map:best-point}
        \STATE $\cpenum{}.\mathsf{init}(\formula, p)$\alglinelabel{line:pa_map:partition-init}
        \WHILE{$\cpenum{}.\mathsf{has\text{-}next}()$} \alglinelabel{line:pa_map:loop-start}
            \STATE $\Pi \gets \cpenum{}.\mathsf{next}()$ \alglinelabel{line:pa_map:partition-next}
            \STATE $m_\Pi, \vx_\Pi \gets \cpmax{}.\mathsf{optimize}(\Pi, p)$ \alglinelabel{line:pa_map:maximize}
            \IF{$m_\Pi > m^*$}\alglinelabel{line:pa_map:update-start}
                \STATE $m^*, \vx^* \gets  m_\Pi, \vx_\Pi$
                \STATE $\cpenum{}.\mathsf{update\text{-}lower\text{-}bound}(m^*)$\alglinelabel{line:pa_map:update-lb}
            \ENDIF\alglinelabel{line:pa_map:update-end}
        \ENDWHILE\alglinelabel{line:pa_map:loop-end}
        \RETURN $m^*, \vx^*`$
    \end{algorithmic}
\end{algorithm}
The algorithm, named \pamap, is inspired in name and spirit by the WMI-PA algorithm for WMI computation~\cite{morettin-wmi-ijcar17,morettin-wmi-aij19,spallitta2022smt,spallitta2024enhancing}.
It decomposes a potentially non-convex feasible region into convex polytopes, over which constrained optimization can be performed efficiently ---and sometimes with guarantees--- for many density classes.
The procedure, outlined in \Cref{alg:pa_map}, builds on two key components: an enumerator \cpenum{} that partitions the feasible region of $\formula$ into convex polytopes,
and a constrained optimizer \cpmax{} to maximize $p$ over each convex polytope.
\pamap{} maintains the current best solution (lines~\ref{line:pa_map:lower-bound}-\ref{line:pa_map:best-point}), 
and iteratively considers convex polytopes $\Pi$ in the partition (lines~\ref{line:pa_map:loop-start}-\ref{line:pa_map:loop-end}).
For each $\Pi$, it invokes \cpmax{} to find the maximum of $p$ restricted to $\Pi$ (line~\ref{line:pa_map:maximize}), updating the best solution accordingly (lines~\ref{line:pa_map:update-start}-\ref{line:pa_map:update-end}).
Importantly, unlike exact WMI computation, finding the maximum of $p$ does not necessarily require enumerating \emph{all} convex polytopes. 
Therefore, after updating the best solution, the enumerator is informed of the new lower bound $m^*$ (line~\ref{line:pa_map:update-lb}), allowing it to \textit{prune} polytopes that are known not to contain better solutions.
\pamap{} is a family of optimizers, we discuss next how different choices for the base optimizer \cpmax{} and polytope enumeration \cpenum{} can impact performance.

\paragraph{Optimization over convex polytopes.}
As discussed in \Cref{sec:related}, optimization over a single convex polytope is a well-understood problem.
Thus, for \cpmax{} in \pamap{}, we can leverage  off-the-shelf constrained optimizers such as \shgo, which can find optima efficiently, albeit without formal guarantees.
Note that numerical optimizers rely on floating-point arithmetic, and as such the returned optimum may be slightly infeasible; nevertheless, such points can be easily projected onto the polytope if needed.
For polynomial densities, we can also employ moment-based global optimization~\cite{lasserre2001global}, which provides reliable upper and lower bounds on the global maximum at the cost of a higher runtime, as we quantify empirically in \cref{sec:experiments}.

\paragraph{Enumeration of convex polytopes.}
Partitioning the feasible region of an \smtlra{} formula into convex polytopes is a task known as \emph{AllSMT}~\cite{lahiri2006SMTTechniques,masina2025cnf}. This involves enumerating truth assignments to linear (in)equalities such that each assignment yields a non-empty convex polytope, and their union covers the feasible region. For this step, we adopt the efficient enumeration techniques used for WMI computation~\cite{spallitta2024enhancing}.
\begin{figure}[!t]
    \centering
    \begin{minipage}{.35\linewidth}
        \includegraphics[width=0.88\linewidth]{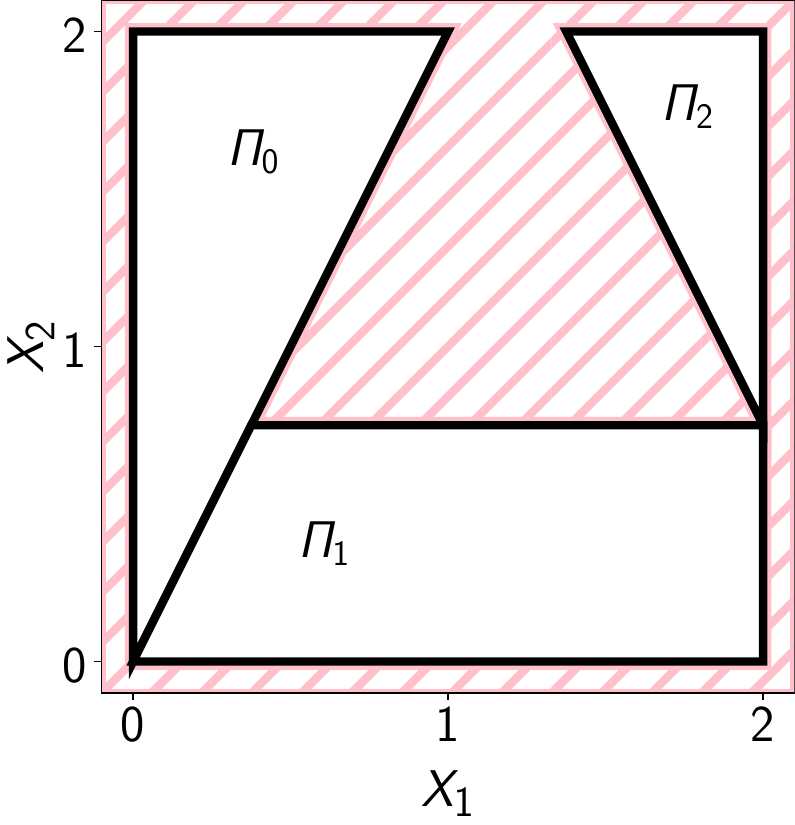}
    \end{minipage}\hfill\begin{minipage}{.62\linewidth}
            \caption{
    \textbf{{\pamap{} can decompose non-convex feasible regions into convex polytopes}}, as shown for the example in \cref{fig:intro-ex}.
    Note that the partitioning only depends on the constraints \formula.
    As such, for a conditional density $p(\vy\,|\vx)$, it needs to be computed only once for all datapoints $\vx$ as in the SDD experiments (\cref{sec:experiments}).  
    }
    \end{minipage}
    \label{fig:pt-decompose}
\end{figure}
Often, the density $p$ is defined as a piecewise function, which some constrained optimizers (e.g.,~\citealt{lasserre2001global}) cannot handle. In such cases, the structure-aware enumeration by~\citet{spallitta2024enhancing} ensures that each enumerated polytope lies entirely within a single piece of $p$. 

\paragraph{Pruning via upper bounds.}
To prune polytopes that won't contribute to improving the maximum, 
we can compute an upper bound on $p$ for a given polytope, and if this bound is lower than the current best value $m^*$, that region can be safely discarded.
While computing tight upper bounds is non-trivial in the general case, we can often exploit the specific shape of $p$ to derive efficient bounding techniques.
For example, in Appendix~\ref{sec:app_details_pa_map}, we provide additional details on how to compute upper bounds for the case of PAL densities~\cite{kurscheidt2025probabilistic} used in our experiments.

\paragraph{Relation with OMT-solvers.}
The idea of decomposing the feasible region into convex polytopes is reminiscent of the lazy OMT loop~\cite{bjorner2015nz,sebastiani2015optimization}. %
There are, however, several key differences that make \pamap{} more practical for solving \mapal{} problems.
First, current OMT-solvers only support polynomial densities, whereas the modularity of \pamap{} allows handling any density which admits a constrained optimizer. 
E.g, mixtures of Gaussians, and also stochastic densities estimated via Monte Carlo sampling, as we show in \cref{sec:experiments}.
Second, even restricting to polynomial densities, OMT-solvers are not specialized for the case of \smtlra{} constraints, and thus address a computationally harder problem. In contrast, \pamap{} decouples the enumeration of convex polytopes from the optimization step, combining efficient enumeration for \smtlra{} and specialized optimizers and pruning techniques for convex polytopes.
Finally, this decoupling allows for straightforward parallelization, as different polytopes can be optimized concurrently.

\section{Experiments}%
\label{sec:experiments}

We now empirically evaluate our \mpmap{} and \pamap{} on several real-world and synthetic benchmarks (\textbf{C3}).
Specifically, we aim to answer these research questions: \textbf{Q1)} How much can  \mpmap{} scale and how does it compare to approximate optimizers for exact constrained MAP? \textbf{Q2)} How does \pamap{} trade-off  solution quality and runtime  on real-world problems?
We describe our setup, baselines, and results for both algorithms next.
Experimental settings are detailed in \cref{sec:app_experiments} and
the code to reproduce experiments is attached to the submission.

\paragraph{Baselines and comparison.}%
\label{sec:baselines}
We compare against OMT-solvers for non-linear real arithmetic, namely OptiMathSAT~\cite{sebastiani2020optimathsat} and CDCL-OCAC~\cite{jia2025complete}, see \cref{sec:related}.
As a first baseline, we compare against the classical \adam{} optimizer~\cite{DBLP:journals/corr/KingmaB14} used to maximize $p(\vx)$ without considering the constraints $\formula$. 
This approach has two  limitations: first, it may return infeasible solutions that do not satisfy $\formula$; second, it is prone to getting stuck in local optima, especially in high-dimensional, non-convex landscapes.
Therefore, and as a side contribution (\textbf{C2}), we introduce a more competitive baseline, \textit{a particle-based, constraints-aware version of \adam{}}, which we call \pcadam{}.
\pcadam{} maintains a set of $N$ particles (i.e., candidate solutions) that are iteratively updated using \adam{}.
Crucially, at each iteration, we update the best result found so far among all particles that \textit{satisfy} the constraints $\formula$.
\cref{sec:app_pcadam} provides further details.

\paragraph{Q1) Scalability of \mpmap.} 
\label{sec:experiments_mp_map}
To benchmark our message-passing scheme, we follow~\citet{ZengICML20} and generate three different kinds of tree-shaped problems in varying dimensions and diameters: \emph{STAR} (star-shaped), \emph{SNOW} (ternary-tree), or \emph{PATH} (linear-chain) trees, 
which represent real-world applications like phylogenetic trees~\cite{nei2000molecular} and fault tree analysis~\cite{vesely1981fault}. 
We couple these constraints with random, unnormalized densities in $\wfamily{\mathsf{PP}}$ (\cref{sec:mp-map}) after we sample random $N$-variable SMT formulas for a given shape among STAR, SNOW or PATH.
\cref{sec:app_details_tree_shaped_primal} further details the setup . In total, this procedure yields $1059$ benchmark instances.
Examples of these instances (for 2 dimensions) are visualized in %
the appendix in \cref{fig:tree_primal_density_path,,fig:tree_primal_density_snow,,fig:tree_primal_density_star}, highlighting both the non-convex constraints and the multimodal densities which render the optimization problem challenging.

We compare $3$ algorithms on these problems: our message-passing algorithm \mpmap, \pamap{} using \shgo~\cite{endresSimplicialHomologyAlgorithm2018} as optimizer, and \pcadam. We do not compare to the OMT-solvers in this experiment due to their limited scalability. %
Since both \pamap{} and \pcadam{} are optimizers without optimality guarantees, they can return arbitrarily fast but potentially very poor solutions. To control for solution quality, we therefore introduce a simple grid-search on the enumerated polytopes %
that these methods must outperform (see ~\cref{sec:appendix_hyperparames_opt_tree_primal} for details) until either the time-budget is exhausted or they surpass the baseline. 

\begin{figure}
    \begin{minipage}{0.34\columnwidth}
        \centering
        \hspace{20pt}{\scriptsize STAR}\\
        \includegraphics[width=\textwidth]{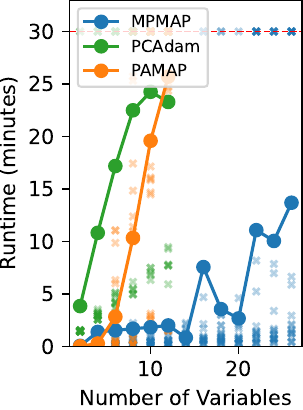}
    \end{minipage}
    \begin{minipage}{0.27\columnwidth}
        \centering
        {\scriptsize SNOW}\\
        \includegraphics[width=\textwidth]{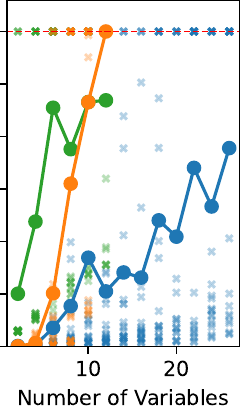}
    \end{minipage}
    \begin{minipage}{0.27\columnwidth}
        \centering
        {\scriptsize PATH}\\
        \includegraphics[width=\textwidth]{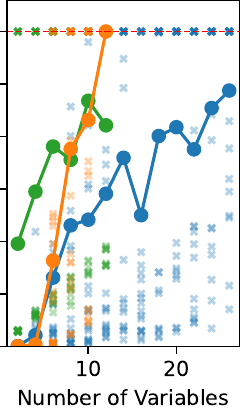}
    \end{minipage}
    \caption{\textbf{{When \mapal{} has tree-structure,  \mpmap{} outperforms competitors}} such as \pamap{} and \pcadam{} on graphs with different diameter, including PATH graphs with maximal diameter $d{=}N{-}1$, where the worst-case complexity of \mpmap{} would scale exponentially. Details in~\cref{sec:app_details_tree_shaped_primal}.}
    \label{fig:mp_map_timeout_overview}
\end{figure}

\cref{fig:mp_map_timeout_overview} shows the results, demonstrating the superior scalability of exact message passing compared to approximate optimizers. 
Here, we scale \pamap{} and \pcadam{} only up to dimension $12$, as runtimes increase rapidly beyond this point and render the experiments prohibitively long. Notably, the advantage of \mpmap{} persists for \emph{PATH} problems, which have maximal graph diameter $d {=} N {-} 1$ and for which the theoretical complexity of \mpmap{} scales exponentially in $d$ (\cref{th:mp_complexity_pp}). %
Detailed results are in \cref{tab:tree_primal_results_snow_long,,tab:tree_primal_results_star_long,,tab:tree_primal_results_path_long}. %

\paragraph{Q2) Trajectory prediction with \pamap{}.}
We evaluate \pamap{} on a first real-world application, and we consider the Stanford drone dataset (SDD)~\cite{robicquet2016learning}, a dataset of trajectories of multiple agents, captured from a drone,
that can only move in walkable areas, modeled as SMT constraints over the 2D space.
Following \citet{kurscheidt2025probabilistic}, we learned a predictive density for an agent's future position, conditioned on its past trajectory and the scene layout.
We then generated 50 test instances by sampling different agent trajectories.
In \cref{fig:sdd_density_overapprox} we show an example of such a density.
The figure also shows an execution of \pamap{} on this instance, demonstrating how the computation of upper bounds is crucial for pruning the vast majority of the polytopes, and thus improving efficiency. In \cref{sec:app_details_pa_map} we show further examples (\cref{fig:trajectory_pruning_grid}) and provide details on upper bound computation.

We run \pamap{} using two different optimizers: a numerical optimizer (\shgo), and an optimizer based on the SoS-Moment hierarchy (see \cref{sec:related})). 
We compare it against \adam{}, \pcadam{} with different number of particles $N$, OptiMathSAT run in anytime mode, and CDCL-OCAC
 (see Appendix~\ref{sec:app_details_pa_map} for details).
Results are shown in \cref{fig:pa-map-sdd}.
Here solution quality is measured in terms of relative optimality gap, computed as $\max(0, v^* - v)/v^*$, between the value found $v$ and the best known value $v^*$ computed via a grid-search.
From the plot we see that \pamap(\shgo) achieves the best trade-off between runtime and solution quality.
\pcadam{} can be competitive in terms of solution quality, but only if enough particles are used, leading to significantly higher runtimes.
\pamap(SoS), here used with a relaxation order of $7$, is significantly slower than the \shgo{} variant. \adam{} is not reliable at finding good, feasible solutions. OMT-solvers struggle in this setting: CDCL-OCAC solved no instances, while OptiMathSAT found only poor-quality solutions or reported errors.

\begin{figure}
    \centering
    \begin{subfigure}{0.24\columnwidth}
        \includegraphics[width=\textwidth]{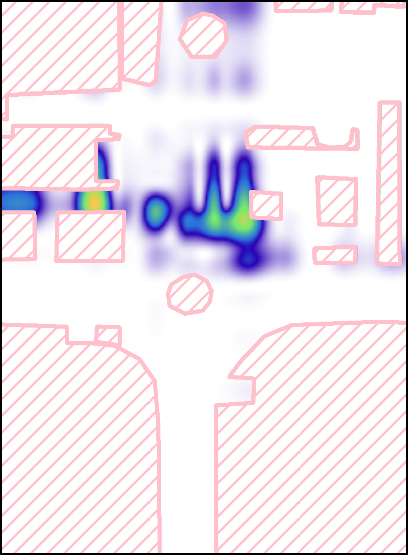}
    \end{subfigure}
    \begin{subfigure}{0.24\columnwidth}
        \includegraphics[width=\textwidth]{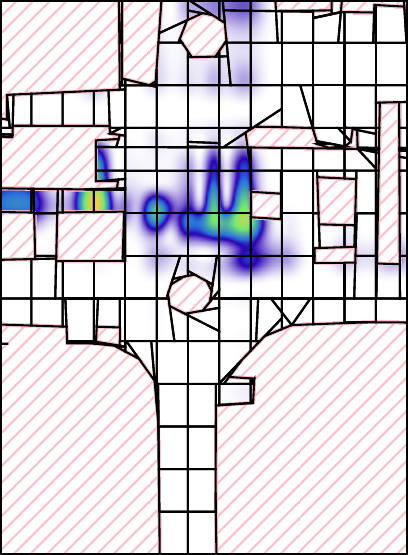}
    \end{subfigure}
    \begin{subfigure}{0.24\columnwidth}
        \includegraphics[width=\textwidth]{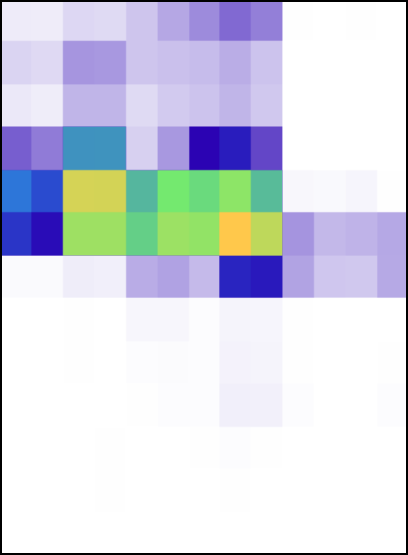}
    \end{subfigure}
    \begin{subfigure}{0.24\columnwidth}
        \includegraphics[width=\textwidth]{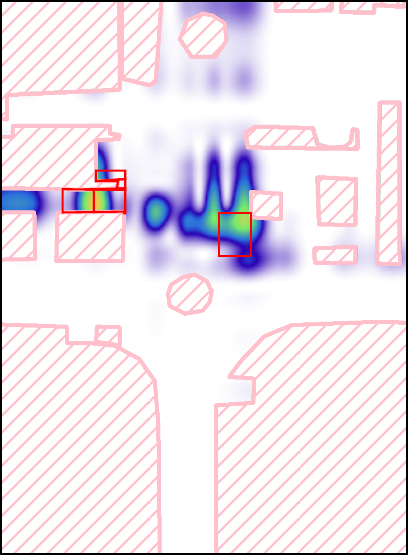}
    \end{subfigure}
    \caption{\textbf{On SDD, our upper-bound pruning strategy drastically reduces the number of polytopes optimized by \pamap{}.} Left: constrained density for an agent's next position. Center-left: 257 convex polytopes considered by \pamap{} without upper-bound-based pruning. Center-right: piecewise-constant upper bound. Right: only 9 convex polytopes considered by \pamap{} with pruning. More examples in \cref{fig:trajectory_pruning_grid} in \cref{sec:app_details_pa_map}.
    }%
    \label{fig:sdd_density_overapprox}
\end{figure}

\begin{figure}
    \centering
    \begin{subfigure}[T]{.57\columnwidth}
        \centering
        \includegraphics[width=\textwidth]{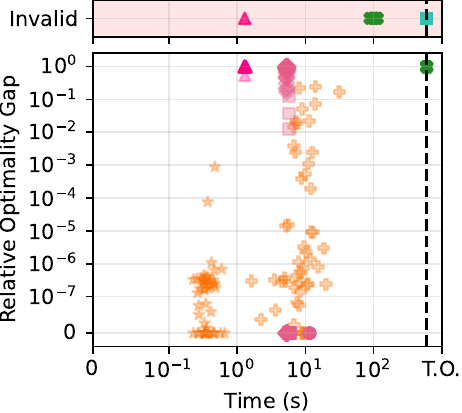}
    \end{subfigure}
    \begin{subfigure}[T]{.41\columnwidth}
        \scriptsize
        \setlength{\tabcolsep}{1.5pt} %
        \begin{tabular}{@{}lcccc@{}}
            \textbf{Method}   & \textbf{S} & \textbf{U} & \textbf{E} & \textbf{TO} \\\midrule
            \pamap(\shgo)     & 50         & 0          & 0          & 0           \\
            \pamap(SoS)       & 50         & 0          & 0          & 0           \\
            \pcadam($N$=10)   & 50         & 0          & 0          & 0           \\
            \pcadam($N$=100)  & 50         & 0          & 0          & 0           \\
            \pcadam($N$=1000) & 50         & 0          & 0          & 0           \\
            \adam             & 46         & 4          & 0          & 0           \\
            OptiMathSAT       & 21         & 0          & 29         & 21          \\
            CDCL-OCAC         & 0          & 0          & 0          & 50          \\
        \end{tabular}
        \vspace{1em}
        \includegraphics[width=.65\textwidth]{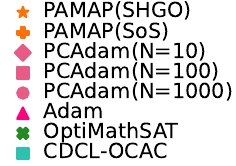}
    \end{subfigure}
    \caption{
    \textbf{On SDD, \pamap{} outperforms other solvers in both time and solution quality.}
    Left: Scatter plot of runtime vs.\ relative optimality gap. Right: Summary table reporting (\textbf{S}) instances solved with a point satisfying $\formula$, (\textbf{U}) instances where the returned point does not satisfy $\formula$, (\textbf{E}) instances terminated with an error, and (\textbf{TO}) number of timeouts (600s).
    Note that OptiMathSAT is anytime and returns the best solution found before timeout. %
    }%
    \label{fig:pa-map-sdd}
\end{figure}

\paragraph{Q2) Data imputation of a VAE with \pamap{}.}
As a second real-world application, we apply \pamap{} to data imputation with a density encoded in a VAE \citep{kingma2014vaes}. 
In this experiment, we use the \emph{House-Sales} tabular dataset with the provided constraints from \citet{stoianbeyond}, which we follow to train an unconstrained VAE on the train-dataset and mask random features on the test-dataset.
Constraints here are in the form of SMT rules over features of a house, such as squared meters, number of rooms and cost.
To evaluate the benefit of constrained imputation, masking is restricted to variables for which constraints are available. For each of $400$ test samples, every eligible feature is masked independently with probability $10\%$, and we compute a MAP estimate over the masked variables subject to the constraints. 
As we cannot compute a VAE density exactly, we approximate it via Monte Carlo samples from the latent prior $\vz$, so 
$p(\vx_{m})\approx1/N \sum_{\vz \sim \mathcal{N}(0,I)} p_{\mathsf{dec}}(\vx_{m}, \vx_{o} | \vz)$
, where $\vx_{m}$ (resp. $\vx_{o}$) are the missing (resp. observed) features, $p_{\mathsf{dec}}$ is the decoder architecture, defined as a neural network outputting an isotropic Gaussian, and $\vz$ is the latent code of the VAE. 
We measure the relative error defined as $\mathrm{RE}(\vx^*){=}\mathsf{avg}(|\vx^* {-} \vx_{gt}|/(|\vx_{gt}| {+} 1))$ where $\vx_{gt}$ is the ground truth value for the missing features that are predicted as $\vx^*$ by \pamap{}.
Since OMT-solvers don't support this kind of densities, we only
compare against a particle version of \adam{} with $N\in\{10, 100\}$ as a baseline.\gm{introduce this in baselines paragraph?} %
For \pamap{}, we need a convex-polytope optimizer that can handle a stochastic objective. Since the numerical optimizers we used in the trajectory experiment are not designed for this, we use our \pcadam{} as convex-polytope optimizer within \pamap{}. 
Further details are provided in \cref{app:details_tabular_dataset}. 
Despite the stochastic objective, \pamap{} substantially improves imputation accuracy: averaged over the dataset, 
obtaining a lower error 
and outperforming particle \adam{} in $78\%$ of samples in our experiment (\cref{fig:experiments_tabular_pamap_adam}), showing that MAP inference under nonconvex constraints is practical and effective for real-world tabular data imputation. We provide further statistics in  \cref{app:tabular_results}.

\begin{figure}
    \centering
    \begin{minipage}{.38\linewidth}
        \includegraphics[width=\linewidth]{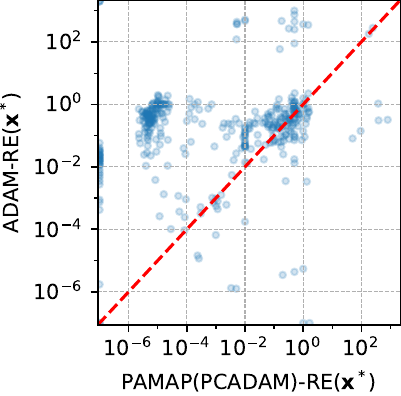}
    \end{minipage}\hfill\begin{minipage}{.55\linewidth}
    \vspace{-15pt}
            \caption{
            \textbf{Combining the VAE prior %
            with domain knowledge constraints, \pamap{} substantially surpasses the unconstrained prediction of \adam.} We compare the relative error on the data-imputation task. Further statistics are reported in \cref{app:tabular_results}.
    \label{fig:experiments_tabular_pamap_adam}
    }
    \end{minipage}
\end{figure}

\section{Conclusion}

In this work, traced the theoretical and practical foundations of MAP inference in the presence of non-convex \smtlra{} constraints and non-log-concave densities.
To this end, we introduced two new solvers that substantially advance the SoTA for this challenging inference task.
The first, \mpmap{} can deal with a new and non-trivial fragment of tractable \ref{eq:constrained-map} problems and can advance our understanding of the complexity of this task, as well as help us design further approximation schemes based on it \citep{zeng2020probabilistic}.
The second, \pamap{}, is a practical and modular framework combining SMT-based convex polytope enumeration, with constrained optimization over each such polytope.
In our rigorous experiments, we showed not only that our solvers can achieve a better trade off between solution quality and execution time, but also that they can be use to perform constrained inference in conjunction with black-box deep generative models.
In the future, we plan to use them to design reliable generative models that satisfy non-convex constraints by design \citep{van2025neurosymbolic,marconato2025symbol}.

\section*{Impact Statement}
We advance both theoretical and empirical understanding of MAP Inference under non-convex constraints. As a general contribution to the field of machine learning, there are many potential societal consequences of our work. Here, we want to highlight one: As constraints enable a more detailed, explicit control over the black-box prediction of machine learning models, they can be used to encode rules and knowledge of the machine learning practitioner. This can be both abused by consciously or unconsciously encoding biases, but can also be used to encode beneficial constraints such as fairness constraints. This example illustrates that constraints in MAP inference provide a mechanism to guide model behavior, with both potential risks and benefits.

\section*{Reproducibility Statement}
To ensure the reproducibility of our results, we have attached our source code as supplementary material. The material includes the implementation of the algorithms, the instructions for setting up the environment, and the scripts and instructions for running the experiments. A detailed description of all experimental settings is given in \cref{sec:app_experiments}.

\section*{Contribution}
GM, LK, AV and RS conceived the initial idea of the paper. 
LK is responsible for all theoretical contributions, illustrations, algorithms and the implementation related to \mpmap{}. 
GM is responsible for all theoretical contributions, illustrations, algorithms and the implementation related to \pamap{}. 
LK conceived and implemented the experiment testing the scalability of \mpmap{} and the missing-value experiment for \pamap{}, while GM conceived and implemented the experiment of trajectory prediction with \pamap{}. 
GM and LK wrote the paper with help from AV and feedback from RS.
AV supervised all phases of the project.

\section*{Acknowledgements}
We thank Dylan Ponsford for valuable feedback on the draft. AV was supported by the “UNREAL: Unified Reasoning Layer for Trustwor-
thy ML” project (EP/Y023838/1) selected by the ERC and
funded by UKRI EPSRC, and acknowledges funds from Huawei TTE-DE Lab Munich.
RS was partially supported by the project ``AI@TN'' funded by the
Autonomous Province of Trento.
RS was partially supported by the MUR PNRR project FAIR - Future AI Research
(PE00000013) funded by the NextGenerationEU;
by the NRRP, Mission 4 Component 2 Investment 1.4,
by the European Union --- NextGenerationEU (proj.\ nr.\ CN 00000013); and
by the TANGO project funded by the EU Horizon Europe
research and innovation program under GA No 101120763, funded by the European
Union.
Views and opinions expressed are however those of the author(s) only and do not
necessarily reflect those of the European Union, the European Health and
Digital Executive Agency (HaDEA) or The European Research Council. Neither the
European Union nor the granting authority can be held responsible for them.

\bibliography{bibliography}
\bibliographystyle{icml2026}

\newpage
\appendix
\onecolumn

\renewcommand{\partname}{}  %

\section{A Scalable Message-Passing Algorithm For Constrained MAP}
\label{sec:mp-map-expanded}

\subsection{Incomparability of tractable MAP and integration over SMT formulas}
\label{app:incomparability-wmi-map}

While the \mapal{} problem is reminiscent of computing the weighted model integral, there is a major difference to the WMI. Our fundamental operation is not integrating out a variable like $\int_{l(x_j)}^{u(x_j)}q(x_j, x_i)dx_i$, but maximizing out a variable like $\max_{x_i\in[l(x_j), u(x_j)]}q(x_j, x_i)$. It is important to note that $l$ and $u$ are symbolic bounds, so affine functions in $x_j$.
In order to general algorithm for the maximization in section \ref{app:mp-map-maxout}, we have to detail the function class we are actually able to handle and how they differ from the class investigated by \citet{ZengICML20}. As we previously mentioned, we are looking two main families of functions, piecewise polynomials $\wfamily{\mathsf{PP}}$ and piecewise exponentiation polynomials $\wfamily{\mathsf{PEP}}$. As we can reduce the task of maximizing exponentiated polynomials to the task of maximizing polynomials in general via going into $\log$-space before maximizing, we will first consider the task of maximizing piecewise-polynomials $\wfamily{\mathsf{PP}}$. Therefore, when thinking about how to maximize out a variable of a piecewise polynomial, the first question to ask is which kind of $q(x_j, x_i)$ stays a piecewise polynomial after maximizing out $x_i$. This is required in order to compute our message-passing algorithm by recusing into the tree. Unfortunately, while \citet{ZengICML20} is able to handle general polynomial over two variable, our choice of functions class is more limited. The issue is that even a simple, general cubic polynomial in two variables already falls outside our polynomial function space when taking the symbolic max; for example, using constant upper and lower bounds, we have:
\begin{align*}
    h^\prime(x)&=\max_{y\in [0,1]} h(x,y)\\
    &=\max_{y\in [0,1]} y^3-1.5xy^2+xy \\
    &=h(x,y)|_{y=\frac{x}{2}-\frac{\sqrt{9x^2-12x}}{6}} \notin \mathbb{R}[x] 
\end{align*}
In this derivation we assume $x \geq 2$, so $\mathsf{dom}(h')=[2,\infty)$. We therefore concentrate on the function-class of piecewise separable non-negative polynomials \wfamily{\mathsf{PP}}, so piecewise polynomials where each piece factors as a product of univariate polynomials:
\begin{align}
    &\wfun_S(\vx_S) = \begin{cases}
\prod_{s \in S} q_1^s(x_s) & \text{if } \vx_S \models \Phi_1,\\
\vdots & \\
\prod_{s \in S} q_k^s(x_s) & \text{if } \vx_S \models \Phi_k,
\end{cases}
\iff \wfun_S(\vx_S) \in \wfamily{\mathsf{PP}}
\end{align}
Here, $q_i^s$ is a univariate polynomial over the variable $x_s$ and $\Phi_i$ a conjunction of literals over $\vx_S$, so a polytope. We require $\Phi_k$ to be non-overlapping. This construction allows us to push the max-operator past the term dependent on other variable while the piecewise-nature retains flexibility. As we later enumerate the pieces in \cref{alg:mp_map_compute_pp}, we can switch it with the maximum per piece:
\begin{align*}
    h(x) &= \max_{y \in [l(x),u(u)]} f(x,y) \cdot q(y)\\
         &= \max_{y \in [l(x),u(u)]} f_1(x)\cdot f_2(y) q(y)\\
         &= f_1(x) \cdot \underbrace{\max_{y \in [l(x),u(x)]} f_2(y) q(y)}_{\text{\maxout{} algorithm (section \ref{app:mp-map-maxout})}}
\end{align*}
In section \ref{app:mp-map-maxout} we derive an algorithm to maximize a univariate, piecewise polynomial under symbolic linear upper and lower bounds, which is the only missing piece we need in order to explicitely compute $h$.

The construction of $\wfamily{\mathsf{PEP}}$ is analogous except that $q_i^s$ are univariate exponentiated polynomials. It is important to note that $\wfamily{\mathsf{PEP}}$ is more general than \citet{ZengICML20} derived for the integral, we the integral can only be performed over exponentiated linear functions but we allow the addition of univariate polynomials of arbitrary degree in the exponent (or as a product before exponentiating).

\subsection{The Message-Passing Algorithm}
\label{app:mp-map-algo}
Here, we will prove the tractability of the \mapal-problem in \cref{th:mp_tractability} by construction, so by explicitly constructing the algorithm \mpmap{} that calculates the $\argmax$ exactly and in polynomial time.

To simplify the derivation, we will w.l.o.g. focus on problem in which the graph $\mathcal{G}$ forms a tree instead of a forest. The approach can be easily extended to forests by looping over the trees.

The key idea behind the message-passing algorithm is that we can exploit the tree-structure of the graph $\mathcal{G}$ (\cref{def:mp_map_graph}). As both density and constraints must be compatible, so $\mathcal{G}$ must have a treewidth of $1$, we can decompose the computation of the overall maximum into smaller and smaller problems. In order to derive the message-passing, we will again focus on calculating the maximum on the example from \cref{ex:primal_graph_smt} with the density \cref{ex:example_wfamily_pp}:
\begin{align*}
    & \max_{\vx} (\factor{3}(x_3)\factor{21}(x_2, x_1) \factor{2}(x_2))(\factor{31}(x_3,x_1) \factor{3}(x_3))                                                                                                               \\
                      & = \max_{x_1} \frac{9}{100} \prod_{i\in \{2,3\}} \underbrace{\max_{x_i} \factor{i1}(x_i, x_1) \underbrace{\factor{i}(x_i)}_
                      {=\msg{\factor{X_i}}{\factor{i1}}{}(x_1)}
                      }_{=\msg{\factor{i1}}{X_1}{}(x_1)} \\
                      & = \max_{x_1} \frac{9}{100} \prod_{i\in \{2,3\}} \msg{\factor{i1}}{X_1}{}(x_1)
\end{align*}

So, if we manage to ``maximize out'' (akin to marginalizing out) the variables one by one, e.g. here calculate $\max_{x_i} \factor{i1}(x_i, x_1) \factor{i}(x_i)$ as a function of $x_1$, we have a scalable algorithm to exactly compute MAP even for high-dimensional problems. We start from a root node $X_r$, that can be chosen so to minimize the number of computations, and then recursing into the children of the directed graph until we hit the leaves:

\begin{align}
\msg{X_i}{\factor{\varscope}}{}(x_i) &\coloneq \prod_{c \in \mathsf{child}(X_i)}\msg{\factor{ci}}{X_i}{}(x_i) \cdot \factor{i}(x_i)\label{eq:mp_map_derivation_msg_to_fac}\\
\msg{\factor{ij}}{X_j}{}(x_j) &\coloneq \max_{x_i} \factor{ij}(x_i, x_j) \cdot \msg{X_i}{\factor{ij}}{}(x_i) \label{eq:mp_map_derivation_msg_max}\\
\max_{\vx} \prod_\varscope \factor{\varscope}(\vx_\varscope) &= \max_{x_{r}} \underbrace{\factor{r}(x_r) \cdot  \prod_{c \in \mathsf{child}(X_r)} \msg{\factor{cr}}{X_r}{}(x_r)}_{\text{univariate function}} \label{eq:mp_map_derivation_final}
\end{align}

In order to compute the overall maximum, we now only need to maximize the rusulting univariate function (line \ref{eq:mp_map_derivation_final}). As we only have to maximize out variable after variables, an operation on single variables, we have a promising direction for a scalable algorithm. It is good to reflect at this point whether how the operations used in the definitions of the messages (line \ref{eq:mp_map_derivation_msg_to_fac} and line \ref{eq:mp_map_derivation_msg_max}) fit our tractable MAP conditions (TMC, see \cref{def:tractable_mp_conditions}). We will first look at equation \ref{eq:mp_map_derivation_msg_to_fac}, here we have the product over the incoming messages and the factor for variable $x_i$. We can reorder this such that we have a product over the SMT-formulas (leading to a logical and) and a product over the functions in $\wfamily{}$. As we are closed under product (and tractable), this message results in a product between the indicator function over an SMT-formula and a function in $\wfamily{}$. In the next equation \ref{eq:mp_map_derivation_msg_max}, we have the product of a factor and the incoming message. We can gain reorder them and group them by a product over constraints (which results in in a single indicator over the and-combination of the formulas) and over functions (which stays again in \wfamily{}). As we later detail in section \ref{app:mp-map-maxout}, we reduce the maximum over the \smt-constraints to multiple symbolic maximums with linear upper and lower bounds and pointwise max-comparisons. Both are required to be tractable and stay in our function-class \wfamily{}. Finally, in line \ref{eq:mp_map_derivation_final}, we have a single maximum, which is just a special case of the symbolic bounds and therefore also tractable (we consider the image of the linear bounds to be the extended reals, so therefore this is the special case of constant upper/lower bounds). We conclude that most of the operations to compute single messages are guaranteed to be tractable directly following from the TMC-conditions, except the enumeration of the linear bounds in \cref{alg:mp_map_compute_generic}. We will later see in the analysis of \mpmap{} (section \ref{app:mp-map-complexity-wfamilies}) that these are maximal polynomially many in the number of atoms per formula.

But if we focus back on these message, we see that there is a remaining piece of the puzzle missing. We do not only need to track of the max.\ value when maximizing out variable after variable, but also the corresponding position this maximum is attained at for the variables maximized out. In comparison to the piecewise message corresponding to the value, which is always a univariate function, this function actually has the signature $\mathbb{R} \rightarrow \mathbb{R}^m$ with $m$ corresponding the the amount of variables maximized out. Therefore, the dimensionality of the space of its image grows during the message-passing. E.g., after obtaining our final message in \cref{ex:mp_max_via_factor_rep} on the left, we can compute the overall argmax by maximizing the remaining message and computing the overall position by evaluating this function at the position the maximum of the remaining message is attained at.

We are now ready to formalize our message-passing scheme in pseudo-code. The high-level routine with the calls to the computation of the messages is given in \cref{alg:mp_map}. As we start with an undirected, but tree-shaped global structure $\mathcal{G}$, we first have to root the tree (or forest). After rooting the tree, the algorithm consists of walking the computed node-order and computing the messages over the factor graph from bottom to top. Computing these messages is composed of two different operations: one, which we call $\gathermsgs$ and is detailed in \cref{alg:mp_map_gather}, collects all the incoming messages at $X_i$ for $\factor{ij}$ and multiplies them; the other is called $\computemsgs$, detailed in \cref{alg:mp_map_compute_generic}, and computes the message from $\factor{ij}$ to $X_j$ by maximizing out $X_i$. %
The messages at the roots correspond then to the density over $X_{root}$ with all the other variables maximized out, and additionally, as we also track the position of the variables maximized out, a function of $X_{root}$ which returns the assignments of these variables maximized out. Doing one final maximization gives us the MAP-prediction we are looking for.

\begin{example}[\mpmap{} in action]
    \label{ex:mp_max_via_factor_rep_app}
    We show the computed messages for the formula of example \ref{ex:primal_graph_smt} combined with the density in example \ref{ex:example_wfamily_pp}. \\
    \begin{tikzpicture}[sibling distance=21mm, level distance=16.5mm,
            var/.style={circle, draw},
            fac/.style={rectangle, draw},
            facimg/.style={inner sep=0pt, anchor=center}, %
            grow=right,
            edge from parent/.style={draw, <-, >=latex}, %
            line width=1pt
        ]
        \node[facimg] (finalmsg) {
            \includegraphics[height=14mm]{figures/message_passing/example/marginal_x1_scaled.pdf}
        }
        child{
                node[var] {$X_1$}
                child {
                        node[facimg] (x2tox1){
                                \includegraphics[height=14mm]{figures/message_passing/example/message_x2_to_x1.pdf}
                            }
                        child {
                                node[var] {$X_2$}
                                child { node[facimg] (margx2) {
                                                \includegraphics[height=14mm]{figures/message_passing/example/potential_x2.pdf}
                                            }
                                        node[above=-0.5mm of margx2] { $\factor{2}$}
                                    }
                            }
                        node[above=-0.5mm of x2tox1] { $\msg{\factor{21}}{X_1}{}$}
                    }
                child {
                        node[facimg] (x3tox1) {
                                \includegraphics[height=14mm]{figures/message_passing/example/message_x3_to_x1.pdf}
                            }
                        child {
                                node[var] {$X_3$}
                                child { node[facimg] (margx1) {
                                                \includegraphics[height=14mm]{figures/message_passing/example/potential_x3.pdf}
                                            }
                                        node[above=-0.5mm of margx1] { $\factor{3}{}$}
                                    }
                            }
                        node[above=-0.5mm of x3tox1] { $\msg{\factor{31}}{X_1}{}$}
                    }
            node[above=-0.5mm of finalmsg] {final message}
            };
    \end{tikzpicture}
\end{example}

So, going back to our example visualized in \cref{ex:mp_max_via_factor_rep} (repeated in \cref{ex:mp_max_via_factor_rep_app}). Running \cref{alg:mp_map} on it could result in $X_1$ being picked as root with the order $X_3$, $X_2$ and then $X_1$. So first, we would call $\gathermsgs(X_3,\factor{3, 1})=\msg{X_3}{\factor{31}}{}=F_3$ and then compute $\computemsgs(\factor{3, 1}, X_1)=\msg{\factor{31}}{X_1}{}=\max_{X_3} \factor{3, 1}(X_1, X_3) \cdot \msg{X_3}{\factor{31}}{}(X_3)$. The $\computemsgs$ function, in \cref{alg:mp_map_compute_generic}, first gathers the change points for $X_j$ induced by the formula $\formula_{ij}$ via $\criticalpoints$ (see \citet{DBLP:conf/uai/ZengB19} for more detail). Inside an interval induced by the change points, so $I=[c,c']$ with $c$ and $c'$ being consecutive change points, we can enumerate the feasible pieces that fall onto interval $I$. As literals cannot intersect in such an interval (as it would generate another change point), these pieces are defined by single linear upper and lower bounds. So, more formal: $(X_i, X_j) \models \formula_{piece} \land X_j \in I \iff l(X_j) \leq X_i \leq u(X_j) \land X_j \in I$ with some linear bound $l$ and $u$. We visualize this process in example \ref{ex:illustration_changepoints}.

\begin{example}[]
\label{ex:illustration_changepoints} We visualize how the critical points lead to a computation of symbolic maximums with linear upper and lower bounds:\\
\begin{minipage}{.45\linewidth}
    \includegraphics[width=.95\linewidth]{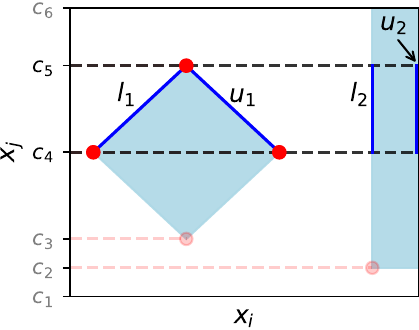}
\end{minipage}\hfill\begin{minipage}{.54\linewidth}
       On the left one can see a visualization of how we compute $\max_{x_i} (\factor{ij} \cdot \msg{X_i}{\factor{ij}}{})(x_i,x_j)$ via the change points, themselves computed via the intersections (red points). This formula generates the change point $(c_1, c_2, c_3, \dots, c_6)$ for $X_j$. We iterate over the induced intervals and now visualize the pieces that fall into the interval $I = [c_4,c_5]$. \\
        For $x_j \in I$, we have two blue pieces that fall into the interval. Here, one can clearly see that on this interval, the these pieces define a single linear upper and lower bound for $x_i$ as a function of $x_j$. These linear bounds are denoted $(l_1, u_1)$ and $(l_2, u_2)$ in the figure. \\
        We can therefore decomposes the overall maximum in this interval first into the maximum over $\wfun_{ij} {\cdot} \msg{X_i}{\factor{ij}}{}$ over $x_i{\in} [l_1(x_j),u_1(x_j)]$ and $x_i{\in} [l_2(x_j),u_2(x_j)]$. This results in another function in \wfamily{} and is tractable according to the TMC. To attain the overall maximum with $x_j$ "maximized out", we have to then do a pointwise max.\ between the results. %
\end{minipage}
\end{example}

We maximize out $X_3$ over the pieces (line~\ref{line:mp_map_compute_generic:max_out}) to obtain the message bound for $X_1$. We repeat the same for $X_3$. Finally, we call $\gatherroot(X_1)=\prod_{i \in \{2,3\}} \msg{\factor{i1}}{X_1}{} = m$ and call the $\max$ and $\argmax$ over $m$ to obtain both the maximal value and the corresponding coordinates.

While \cref{alg:mp_map_compute_generic} provides a generic algorithm that works with any family for which the tractable MAP-conditions hold, we provide a specific version of the algorithm in \cref{alg:mp_map_compute_pp} that shows how we handle densities in \wfamily{\mathsf{PP}}. Here, we do not only enumerate the pieces resulting from $\formula_{ij}$ but also the pieces from the piecewise function $\wfun_{ij}$, in order to perform our $\max_{x_i}p(x_j)p(x_i)=p(x_j)\max_{x_i}p(x_i)$-trick piecewise. The algorithm for \wfamily{\mathsf{PP}} is analogous, except that we go into log-space before calling \maxout{} in order to reuse our polynomial maximization scheme.

The complexity is discussed in section \ref{app:mp-map-complexity-wfamilies}

\begin{algorithm}
    \caption{$\bm\computemsgs^{\wfamily{\mathsf{PP}}}$($\factor{i, j}$, $X_{j}$)
    }%
    \label{alg:mp_map_compute_pp}
    \begin{algorithmic}[1]
        \INPUT $\factor{i, j}$: factor, $X_{j}$: variable
        \OUTPUT $\msg{\factor{ij}}{X_j}{}$: message
        \STATE $\msg{\factor{ij}}{X_j}{} \gets \emptypiecewise()$ %
        \STATE $\wfun_{ij} \gets \mathsf{get\text{-}weight}(\{ij\}, p)$ \COMMENT{bivariate function associated to $\{ij\}$}
        \STATE $\varphi \gets \bigvee_{\varphi_k \in \mathsf{formulas}(\wfun_{ij})} \varphi_k$
        \STATE $\formula_{ij}' \gets \formula_{ij} \land \varphi$\alglinelabel{line:mp_map_compute:formula_and_piecewise}
        \STATE $\mathcal{P} \leftarrow \criticalpoints(\overallbounds(\msg{X_i}{\factor{ij}}{}), \formula_{ij}')$\alglinelabel{line:mp_map_compute:P}
        \STATE $\mathcal{I} \leftarrow \intervalsfrompoints(\mathcal{P})$\alglinelabel{line:mp_map_compute:I}
        \FOR{interval $I \in \mathcal{I}$ consistent with formula $\formula_{ij}'$}
        \STATE $ \mathcal{Q} \leftarrow$
        \getmsgpieces($\msg{X_j}{\factor{ij}}{}, I, \wfun_{ij}$)\alglinelabel{line:mp_map_compute:get_pieces} \\
        \COMMENT{enumerates every piece that falls into $X_i \in I$}
        \STATE $q^\prime_{h} \leftarrow q^i_h {\cdot} \maxout(q^j_h, l_h, u_h) \forall (l_h, u_h, q^i_h, q^j_h) {\in} \mathcal{Q}$\alglinelabel{line:mp_map_compute:max_out}\\
        
        \STATE $\msg{\factor{ij}}{X_j}{}|_{x_i \in I} \leftarrow \maxpieces(\{q^\prime_{h}, 0 {\leq} h {<} |Q|\})$\alglinelabel{line:mp_map_compute:max-pieces}\\
        \COMMENT{The maximum over the pieces forms $\msg{\factor{ij}}{X_j}{}$ at $x_i \in I$}
        \ENDFOR
        \RETURN $\msg{\factor{ij}}{X_j}{}$
    \end{algorithmic}
\end{algorithm}

\subsection{The \maxout{} Algorithm}
\label{app:mp-map-maxout}

While a lot of similarities between computing the $\argmax$ and the integral exists, a major difference that the fundamental operation is not similarly established. As the core operation is not integrating out a variable like $\int_{l(x_j)}^{u(x_j)}q(x_j, x_i)dx_i$, but maximizing out a variable like $\max_{x_i\in[l(x_j), u(x_j)]}q(x_j, x_i)$, we call the corresponding algorithm $\maxout$. It is important to note that $l$ and $u$ are symbolic bounds, so affine functions in $x_j$ and therefore the result of maxing out $x_i$ is a function of $x_j$.
Before introducing our general algorithm for the maximization, we have to detail the function class we are actually able to handle. As we previously described, we are looking two main families of functions, piecewise polynomials $\wfamily{\mathsf{PP}}$ and piecewise exponentiation polynomials $\wfamily{\mathsf{PEP}}$. As we can reduce the task of maximizing exponentiated polynomials to the task of maximizing polynomials via going into $\log$-space (see \cref{alg:max_out}), the rest of the section detailing our \maxout{} algorithm will therefore only consider the task of maximizing piecewise-polynomials ($\maxout^{\mathsf{PP}}$).
We can therefore concentrate on deriving an algorithm to maximize a univariate, piecewise polynomial under symbolic linear upper and lower bounds. As the result will again be a piecewise polynomial, we can recursively apply the algorithm over and over in the message-passing. We will first focus on computing the \emph{value} of the polynomial with a variables maxed-out, and then later generalize this to also include the coordinate of the variable that got maxed out.

\paragraph{Maximizing a univariate piecewise polynomial.} In order to derive our algorithm to explicitly construct $m(y)=\max_{x \in [l(y), u(y)]} q(x)$ for $q(x)$ being a piecewise, but not necessarily continuous, polynomial, we first start by thinking about it point-wise in $y$. We begin with a slight generalization of the extreme-value theorem/interior extremum theorem to piecewise polynomials. Therefore, the general formula to compute the maximum of $q$ over $[l(y), u(y)] \coloneq I(y)$ is:
\begin{align*}
\max_{x \in I(y)} q(x) &{=} 
\max 
\left\{
\begin{array}{lr}
\max \{q(l(y)), q(u(y))\},&(a)\\
\max \{ q(e) \mid e \in \mathrm{E}_q \cap I(y) \}, &(b)\\
\max \{q^*(b) \mid b \in \mathrm{B}_q \cap I(y) \}&(c)\\
\end{array}
\right\} 
\\
\end{align*}
with $\mathrm{E}_q$ denoting the extreme-points, $\mathrm{B}_q$ denoting the boundaries of the pieces of the piecewise polynomial $q$. Finally, $q^*(b)$ is defined as $q^*(b)\coloneq \max\{\lim_{x\rightarrow b^-} q(x), \lim_{x\rightarrow b^+} q(x)\}$ if both the left and right limits are defined otherwise it is $q^*(b)\coloneq q(b)$. As we can see, we have two piecewise constant terms, the maximas over the extreme $(b)$ and boundary-points $(c)$, and only one term, $\max \{q(l(y)), q(u(y))\}$, being directly dependent on $y$. But when can terms $(b)$ and $(c)$ change their values?
We have to look at the union of the feasible pre-images of these points 
$ C_{\hat{e}} = \bigcup_{f \in {l,u}} \{ y \mid f(y) \in \mathrm{E}_q {\cup} \mathrm{B}_q \land l(y) {\le} u(y) \land f  {\neq} \text{const.}\}$. So all the valid points ($l(y) {\le} u(y)$) that map to extreme and boundary points of $q$. Inside an interval spanned by these points, no change of $(b)$ and $(c)$ can occur as no point from $\mathrm{E}_q$ and $\mathrm{B}_q$ can enter $I(y)$, and therefore the terms $(b)$ and $(c)$ are constants. We call this set $C_{\hat{e}}$, which forms our first contribution to the set of breakpoints of the resulting piecewise function $m(y)$. The other term $(a)$ is more tricky to analyze. As we want to explicitly construct our resulting piecewise function $m(y)$, the first question to ask is where the breakpoints originating from the term $\max \{q(l(y)), q(u(y))\}$ are located. A first contribution to the breakpoints, which we call $C_{\mathrm{switch}}$, comes from change in the dominating polynomial, which must occur at the feasible roots of $q(l(y)) {-} q(u(y))$, so roots of $t(y)=q(l(y)) {-} q(u(y))$ with $l(r) {\le} u(r)$. As we are only interested in the real roots here, we can run a real root isolation algorithm like the Vincent-Akritas-Strzeboński continued-fraction method~\cite{vincent1834note, Akritas_Strzebonski_2005}. So inside the intervals spanned by the breakpoints $C_{\hat{e}} \cup C_{\mathrm{switch}}$ we know that term $(b)$ and $(c)$ is a constant, which we can calculate and henceforth call $\hat{e}$, and that only one polynomial  must dominate. We will call the dominating polynomial on interval $I$ $\hat{q}$, so:
\begin{align}
\hat{q} &= (q \circ u) \iff \forall y \in I: (q \circ u) \geq (q \circ l) \\
\hat{q} &= (q \circ l) \text{  otherwise}
\end{align}
We now only need to characterize the relationship between $\hat{q}$ and $\hat{e}$. The careful reader will notice that $C_{\hat{e}}$ also characterizes the behavior of $\hat{q}$ between its breakpoints, as it lets us assume monotonicity for $\hat{q}$ inside the intervals, because it is composed of the extreme points of both $q \circ l$ and $q \circ u$. This observation enables us to enumerate all the possible relationships between $\hat{e}$ and $\hat{q}$ inside an interval $[i_1,i_2]{=}I^\prime$, with $[i_1,i_2]$ being an interval spanned by the breakpoints $C_{\hat{e}} \cup C_{\mathrm{switch}}$, them being:
\begin{enumerate}%
    \item $\min\{\hat{q}(i_1),\hat{q}(i_2)\}{\geq}\hat{e} \Rightarrow \forall i {\in} I^\prime{:} \max\{\hat{q}(i), \hat{e}\}{=}\hat{q}(i)$
    \item $\max\{\hat{q}(i_1),\hat{q}(i_2)\}{\leq}\hat{e} \Rightarrow \forall i {\in} I^\prime{:} \max\{\hat{q}(i), \hat{e}\}{=}\hat{e}$
    \item $\hat{q}(i_1) < \hat{e} < \hat{q}(i_2)$, then there exists 
        $i_{\mathrm{break}} \in I'$ s.t.%
        \begingroup
        \setlength{\abovedisplayskip}{0pt}%
        \setlength{\belowdisplayskip}{0pt}%
        \begin{align*}
        &\forall i \in [i_1, i_{\mathrm{break}}]: 
          &&\max\{\hat{q}(i), \hat{e}\} = \hat{e}, \\[-1pt]
        &\forall i \in [i_{\mathrm{break}}, i_2]: 
          &&\max\{\hat{q}(i), \hat{e}\} = \hat{q}(i)
        \end{align*}
        \endgroup

    \item $\hat{q}(i_1) > \hat{e} > \hat{q}(i_2)$, then there exists a 
    $i_{\mathrm{break}} {\in} I'$ s.t.
    \begingroup
        \setlength{\abovedisplayskip}{0pt}%
        \setlength{\belowdisplayskip}{0pt}%
        \begin{align*}
        &\forall i \in [i_1, i_{\mathrm{break}}]: 
          &&\max\{\hat{q}(i), \hat{e}\} = \hat{q}(i), \\
        &\forall i \in [i_{\mathrm{break}}, i_2]: 
          &&\max\{\hat{q}(i), \hat{e}\} = \hat{e}
        \end{align*}
    \endgroup
\end{enumerate}
In order to simplify the math, we will set $\hat{q}(y)=-\infty$ if $y$ is not in the domain of $\hat{q}$.

But, a final contribution to the set of breakpoints of $m$ is still missing. We need to add the start/end-bounds for contributed by $u$ and $l$, as the start/end of our feasible area $l(y) \leq u(y)$ if exists contributes another breakpoints. As two linear functions can only intersect at most once, and this intersection can be either the start or end-point of our feasible area, the possible values of the set $C_{\mathrm{bounds}}$ are straightforward to enumerate:
\begin{equation*}
C_{\mathrm{bounds}} {=}
\begin{cases}
\{y_s, +\infty\}, \!\!&\!\! \text{if } \exists y_s{:}l(y_s){=}u(y_s) \land y_s \text{ is start}\\
\{-\infty, y_s\}, \!\!&\!\! \text{if } \exists y_s{:}l(y_s){=}u(y_s) \land y_s \text{ is end}\\
\emptyset, \!\!&\!\! \text{if } \forall y: u(y) < l(y) \text{ \,\,\,\, (parallel)}\\
\{-\infty, +\infty\}, \!\!&\!\! \text{otherwise.}
\end{cases}
\end{equation*}

Together, these sets form the breakpoints of our piecewise function $m(y)$: $C_{\mathrm{break}} = C_{\hat{e}} \cup C_{\mathrm{switch}} \cup C_{\mathrm{bounds}}$. We are now ready to formalize our $\maxout^{\mathsf{PP}}$-algorithm, as provided in \cref{alg:max_out_pp}.

\begin{algorithm}
    \caption{$\bm\maxout^{\mathsf{PP}}$$(q, l, u)$}
    \label{alg:max_out_pp}
    \begin{algorithmic}[1]
        \INPUT $q$: piecewise polynomial, $l$: affine lower bound, $u$ affine upper bound
        \OUTPUT $m$: piecewise polynomial %
        \STATE $V_q, C_{\hat{e}}, C_{\mathrm{switch}}, C_{\mathrm{bounds}} \gets \preparebreaks(q,l,u)$
        \STATE $m \gets \emptypiecewise()$
        \FOR{$[i_1, i_2] \in \intervals(C_{\hat{e}} \cup C_{\mathrm{switch}} \cup C_{\mathrm{bounds}})$}
            \STATE $\hat{q} \gets \dominatingpoly((q \circ l), (q \circ u), i_1, i_2)$

            \STATE $\hat{e} \gets \innermax(q, l, u, i_1, i_2, V_q)$
            \alglinelabel{line:max_out_e_hat}
            \STATE $\$ \gets \exclusiveorinclusivestart(\hat{q})$\alglinelabel{line:max_out_exclusive_inclusive_start} 
            \STATE \COMMENT{$\$$ can be either ``$[$'' or ``$($''}
            \IF{$\max\{\hat{q}(i_1),\hat{q}(i_2)\} \leq \hat{e}$}
                \IF{$\hat{e} = - \infty$}
                    \STATE \textbf{continue} \alglinelabel{line:max_out_continue_undefined}
                \ELSE
                    \STATE $m|_{\$i_1, i_2]} \gets \hat{e}$ \COMMENT{defines $m$ on interval $\$i_1, i_2]$}\alglinelabel{line:max_out_assign_e_1}
                \ENDIF
            \ELSIF{$\min\{\hat{q}(i_1),\hat{q}(i_2)\} \geq \hat{e}$}
                \STATE $m|_{\$i_1, i_2]} \gets \hat{q}$\alglinelabel{line:max_out_assign_q_1}
            \ELSIF{$\hat{q}(i_1) < \hat{e} < \hat{q}(i_2)$}\alglinelabel{line:max_out_intersection_1}
                \STATE $i_\mathit{break} \gets \intersection(\hat{q}, \hat{e}, i_1, i_2)$
                \STATE $m|_{\$i_1, i_\mathit{break})} \gets \hat{e}$\alglinelabel{line:max_out_assign_e_2}
                \STATE $m|_{[i_\mathit{break}, i_2]} \gets \hat{q}$ \alglinelabel{line:max_out_assign_q_2}
            \ELSE\alglinelabel{line:max_out_intersection_2}
                \STATE $i_\mathit{break} \gets \intersection(\hat{q}, \hat{e}, i_1, i_2)$
                \STATE $m|_{\$i_1, i_\mathit{break})} \gets \hat{q}$ \alglinelabel{line:max_out_assign_q_3}
                \STATE $m|_{[i_\mathit{break}, i_2]} \gets \hat{e}$
            \ENDIF
        \ENDFOR
        \IF{$C_\mathit{bounds} \text{ has a finite element}$}
            \STATE $i_{b} \gets \getfiniteelem(C_\mathit{bounds})$
            \STATE $m|_{[i_{b},i_{b}]} \gets (q \circ l)(i_b)$ \COMMENT{start or end}
        \ENDIF
        \RETURN $\simplify(m)$ 
    \end{algorithmic}
\end{algorithm}

\FloatBarrier

\subsection{Correctness of the \maxout{} Algorithm}
\label{app:mp-map-maxout-correctness}
Before proving the correctness of the $\maxout^{\mathsf{PP}}$-algorithm, we will need a few propositions later used in the proof \cref{th:max_out_correctness}.

Consider a univariate piecewise polynomial $q$ with finitely many (not necessarily consecutive) pieces with discontinuous points at the breakpoints, and let $C_{\hat{e}}, C_{\mathrm{switch}}, C_{\mathrm{bounds}}, V_q$ be defined as in \cref{alg:max_out_pp}.  
Let $l$ and $u$ be univariate affine functions, and define $I(y) \coloneq [l(y), u(y)]$. We further define $q^\prime(x)$ to be the maximum between the left and right hand limit ($\lim_{x' \rightarrow x^-}q(x)$ and $\lim_{x' \rightarrow x^+}q(x)$), if they exist, otherwise the existing left or right. With $\mathsf{int}(A)$ we denote the interior of a set $A$.

We will treat $\max\emptyset$ and $\sup\emptyset$ as undefined.

\begin{proposition}
\label{prop:general_pointwise_extremum_theorem}
For any $y$ with $l(y) < u(y)$, we have

\[
\sup_{x\in I(y)} q(x)
=
\max\!\Big(
\!\begin{aligned}[t]
&\{\!\max(q(u(y)),\lim_{x\to u(y)^{-}}\!q(x)) | u(y) \in \mathsf{dom}(q)\!\},\\
&\{\!\max(q(l(y)),\lim_{x\to l(y)^{-}}\!q(x)) | u(y) \in \mathsf{dom}(q)\!\},\\
&\{\,q'(e)\mid e\in(\mathrm{E}_q\cup\mathrm{B}_q)\cap\mathsf{int}(I(y))\,\}
\Big).
\end{aligned}
\]
\end{proposition}

\begin{proof}
The restriction of $q$ to $I(y)$ decomposes into finitely many continuous polynomial pieces. 
On each (closed relative to $I(y)$) piece the continuous polynomial attains its maximum either at an interior critical point or at an endpoint of that piece. 
Interior critical points lie in $\mathrm{E}_q \cap \mathsf{int}(I(y))$; 
endpoints are either interior breakpoints (via $q'$ over $\mathrm{E}_q \cap \mathsf{int}(I(y))$) 
or the interval endpoints $l(y), u(y)$. 
Any approach to an endpoint from inside $I(y)$ produces only the interior-directed one-sided limit 
($\lim_{x \to u(y)^-} q(x)$ at $u(y)$, $\lim_{x \to l(y)^+} q(x)$ at $l(y)$), or the point value, over which we take the maximum.
Taking the maximum over this finite candidate set yields the stated equality.
\end{proof}

\begin{proposition}
For any $y$ contained in an interval $[i_1,i_2]$ spanned by the breakpoints 
$C_{\hat{e}} \cup C_{\mathrm{switch}} \cup C_{\mathrm{bounds}}$, we have
\[\{\mathrm{E}_q \cup \mathrm{B}_q\} \cap (\min\{l(i_1), l(i_2)\}, \max\{l(i_1), l(i_2)\} \setminus S  = \emptyset \]
and
\[\{\mathrm{E}_q \cup \mathrm{B}_q\} \cap (\min\{u(i_1), u(i_2)\}, \max\{u(i_1), u(i_2)\} \setminus S = \emptyset \]
with $S = \{\mathrm{E}_q \cup \mathrm{B}_q\} \cap [\max\{l(i_1), l(i_2), \min\{u(i_1), u(i_2)\}]$
\end{proposition}

\begin{proof}
Since $l$ and $u$ are affine, they are either strictly monotone or constant on $[i_1,i_2]$. 

If $l$ (resp.\ $u$) is strictly monotone, then it defines a bijection between $[i_1,i_2]$ and $[l(i_1),l(i_2)]$ (resp.\ $[u(i_1),u(i_2)]$), and thus any break-point or extreme point of $q$ in this image would correspond to one of the known points in $C_{\hat{e}} \cup C_{\mathrm{switch}}$. 
If $l$ (resp.\ $u$) is constant, then its image is a singleton possibly coinciding with a break-point or extremum of $q$, in which case the point is already contained in $S$. 
\end{proof}

With the above propositions, we are now ready to prove correctness.

\begin{theorem}[$\maxout^{\mathsf{PP}}$ is correct]
\label{th:max_out_correctness}
For any piecewise polynomials $q$, which may contain non-continuous breakpoints, for all affine functions $l$, $u$ and any real number $y$, we have that 
$\sup_{x \in [l(y), u(y)]} q(x) = \maxout^{\, \mathsf{\mathit{PP}}}(q,l,u)(y)$.
\end{theorem}

\par{\textit{Ansatz.}} For the proof of correctness we fill proof that for some arbitrary piecewise polynomial $q$, upper bound $l$, lower bounds $u$ and point $y$, if $\sup_{x \in [l(y), u(y)]} q(x) {=} z$, then $\maxout^{\mathsf{PP}}(q,l,u)(y){=}z$. Since $\sup_{x \in [l(y), u(y)]} q(x)$ is a total function in $y$ with the same domain as $\maxout^{\mathsf{PP}}(q,l,u)$, their equality implies that the inverse relation also holds, establishing correctness. We do this proof over $\maxout^{\mathsf{PP}}(q,l,u)(y)$ with the $\simplify$-call omitted, as it only joins the pieces of the same consecutive polynomial it results in the same function point-wise.

\begin{proof}
\textbf{We start the proof by considering the case of $l(y) > u(y)$.} As in this case $\sup_{x \in [l(y), u(y)]} q(x) = \max \emptyset$ and we treat $\max \emptyset$ as undefined, we start to the proof by checking that $\maxout^{\mathsf{PP}}(q,l,u)=:m$ is undefined whenever $l(y) > u(y)$. There are essentially $4$ cases to check here, as $l$ and $u$ being linear functions they can only intersect at most once:
\begin{enumerate}[topsep=0pt, partopsep=0pt, itemsep=1pt, parsep=0pt, label=\arabic*.]
    \item if $[l(y), u(y)]$ is feasible for all $y$ (no intersection of $l$ and $u$), then there is nothing to check as $\max \emptyset$ never occurs
    \item if $[l(y), u(y)]$ is infeasible for all $y$ (no intersection of $l$ and $u$), then $C_{\hat{e}} \cup C_{\mathrm{switch}} \cup C_{\mathrm{bounds}}=\emptyset$, as we $C_{\hat{e}}$ and $C_{\mathrm{switch}}$ only contain feasible points and $C_{\mathrm{bounds}}$ is also defined as the empty set for this case. As the domain of $m$ is defined by the intervals spanned by these points it is also empty
    \item if $[l(y), u(y)]$ is bounded from below by $y_{start}$  (intersection at $y_{start}$), then $\min C_{\hat{e}} \cup C_{\mathrm{switch}} \geq y_{start}$ and $\min C_{\mathrm{bounds}} = y_{start}$, therefore $m$ is undefined for $y < y_{start}$
    \item if $[l(y), u(y)]$ is bounded from above by $y_{end}$  (intersection at $y_{end}$), then $\max C_{\hat{e}} \cup C_{\mathrm{switch}} \leq y_{end}$ and $\max C_{\mathrm{bounds}} = y_{end}$, therefore $m$ is undefined for $y > y_{end}$
\end{enumerate}
This settles the case. \checkmark

\smallskip
\textbf{We will now check the case for $l(y)=u(y)$.} Let $y$ be some real number such that $l(y)=u(y)$. In this case we have $\sup_{x \in I(y)}q(x)$ = $(q\circ u)=(q \circ l)=z$. As $l$ and $u$ are linear functions, they can only intersect once: at the start or the end of the feasible set (in terms of $y$). Therefore, in case $l(y)=u(y)$, then $y$ is the only finite element in the set $C_\mathit{bounds}$. This element is called $i_b$ in the $\maxout^{\mathsf{PP}}$ algorithm, for which we set $m$ to be $(q \circ l)(i_b)$, which is the value of $\sup_{x \in I(y)}q(x)$. \checkmark

\medskip
\textbf{We can now assume that $l(y) < u(y)$.} We can therefore turn our attention to Proposition~\ref{prop:general_pointwise_extremum_theorem}:
\[
\sup_{x\in I(y)} q(x)
=
\max\!\Big(
\!\begin{aligned}[t]
&\colorbox{yellow!20}{$\{\!\max(q(u(y)),\lim_{x\to u(y)^{-}}\!q(x)) | u(y) \in \mathsf{dom}(q)\!\},$}\\[2pt]
&\colorbox{yellow!20}{$\{\!\max(q(l(y)),\lim_{x\to l(y)^{-}}\!q(x)) | u(y) \in \mathsf{dom}(q)\!\},$}\\[2pt]
&\colorbox{cyan!15}{$\{\,q'(e)\mid e\in(\mathrm{E}_q\cup\mathrm{B}_q)\cap\mathsf{int}(I(y))\,\}$}\!\Big)
\end{aligned}
\]
This expression is composed of two groups, the \colorbox{yellow!25}{\textcolor{black}{(1) endpoints}} (first two line)
 and \colorbox{cyan!15}{\textcolor{black}{(2) interior points}} (the last line). We will continue this proof by case distinction on these two groups.

\label{eq:prop:general_pointwise_extremum_theorem_numbered}
\textbf{Case I: $q$ is undefined on $I(y)=[l(y),u(y)]$.} Before considering the maximum to be either from set \colorbox{yellow!25}{$(1)$} or \colorbox{cyan!15}{$(2)$}, we have to check an additional case: if both sets are empty because $q$ is undefined on $I(y)$, then $\sup_{x \in I(y)} q(x)$ is undefined as well. We therefore have to check that $m({=:}\maxout^{\mathsf{PP}}(q,l,u))$ is also undefined in this instance. As $\min C_{\mathrm{bounds}}\leq y \leq \max C_{\mathrm{bounds}}$, we are guaranteed to find an $[i_1,i_2] \ni y$ in the set of breakpoints. As $q$ is entirely undefined on $I(y)$, both $q \circ l$ and $q \circ u$ have to be undefined on $[i_1,i_2]$, as no breakpoints for the respective functions can lie in the interval. As a consequence, $\dominatingpoly((q \circ l), (q \circ u), i_1, i_2)$ has to return $y \mapsto -\infty (= \hat{q})$. As $V_q \cap I(y) = \emptyset$, by assumption, it follows that $V_q \cap I(y) \supseteq V_q \cap [\max\{l(i_1), l(i_2), \min\{u(i_1), u(i_2)\}] = \emptyset$. Therefore, $\hat{e}$ in line \ref{line:inner_max:hat_e_assignment} (\cref{alg:inner_max}) takes the value $-\infty$. Therefore:
\begin{align}
                &\max\{\hat{q}(i_1),\hat{q}(i_2)\} \leq \hat{e}\\
    \Rightarrow &\max\{-\infty,-\infty\} \leq -\infty\\
    \Rightarrow &-\infty \leq -\infty %
\end{align}
As $\hat{e}=-\infty$, we do not assign $m$ any value in this interval (following from line \ref{line:max_out_continue_undefined}), and therefore $m$ is undefined. $\checkmark$\\

As we can now assume that $q$ is at least partially defined on $I(y)$, \textbf{we know that either set \colorbox{yellow!25}{$(1)$} or set \colorbox{cyan!15}{$(2)$} must be non-empty.} We now have to prove that for either cases, $\maxout^{\mathsf{PP}}$ returns the correct result.

\medskip
\textbf{Case II: The maximum is in set \colorbox{yellow!25}{$(1)$}} (so we assume the supremum occurs at the boundary). Let $\hat{q}$ be either $q \circ l$ or $q \circ u$ depending on from which set the supremum came from (or an arbitrary choice for a tie). Let w.l.o.g.\ the upper bound be the winning bound, we therefore have to analyze the expression:
\[
\max(q(u(y)),\lim_{x\to u(y)^{-}}\!q(x))
\]
We will now have have another case-distinction on type of $u(y)$ for $q$ under the assumption that the supremum came from $q \circ u$.

\textbf{Case II.I: Assuming $u(y)$ is not a break point of $m$}, the limit coming from the interior will coincide with the point-wise evaluation $q(u(y))=\hat{q}(y)$. Let $[i_1,i_2]$ the currently active breakpoints for $y$. If it is a tie, the chosen polynomial returned by $\dominatingpoly$ (alg. \ref{alg:mp_map_dominating_poly}) does not matter. If $q(u(y))$ dominates $q(l(y))$, as we assume, it must do so on the whole interval, therefore on $i_1$, $q(u(y^\prime)) - q(l(y^\prime))|_{y^\prime=i_1}$ must be positive or in case it is zero, the highest non-vanishing gradient of $q(u(y^\prime)) - q(l(y^\prime))$ must be positive for $q(u(y^\prime)) - q(l(y^\prime))|_{y^\prime = y} > 0$ to hold. This follows from Taylor's theorem with remainder \cite{apostol1991calculusremainder} around $i_1$ and generalized to the interval using the fact that the two polynomials can not intersect inside the interval. Therefore, $\dominatingpoly$ (\cref{alg:mp_map_dominating_poly}) in all cases returns $q \circ u$ and it is assigned as $\hat{q}$. We need to collect all the breakpoints/extreme points inside $[l(y), u(y)]$, but since we know that between $i_1$ and $i_2$ there are no breakpoints/extreme points, $ [\max\{l(i_1),l(i_2), \min\{u(i_1),u(i_2)\}\}] \subseteq [l(y), u(y)]$ contains all the relevant points and does not change for $y\in (i_1,i_2)$. In case the set is non-empty, $\max$ over \colorbox{cyan!15}{$(2)$} is the same as $\hat{e}$ (line \ref{line:inner_max:hat_e_assignment}, \cref{alg:inner_max}). As $\hat{e}$ and $\hat{q}$ can not intersect inside the interval, as we also split by intersection (if intersection occurs, and there can only be one, cases line \ref{line:max_out_intersection_1} and \ref{line:max_out_intersection_2} handle them by splitting the interval), therefore we know that $\hat{e} \leq \hat{q}(i_1)$ and $\hat{e} \leq \hat{q}(i_2)$ following from $\hat{e} \leq \hat{q}(y)$. So, we can either be in case line \ref{line:max_out_assign_q_1}, \ref{line:max_out_assign_q_2} or \ref{line:max_out_assign_q_3}. In conclusion, $m(y)$ is assigned $\hat{q}$ and $m(y)=\hat{q}(y)=q(u(y))=\sup_{x \in [l(u),u(y)]}q(x)$. \checkmark

\textbf{Case II.II: Assuming $u(y)$ is a break point of $m$}, there might be a difference between $q(u(y))$ and $\lim_{x\to u(y)^{-}}\!q(x)$. A first observation is that, in case we look at a breakpoints where a polynomial starts closed (and the other polynomial ends open), we have $\lim_{x\to u(y)^{-}}\!q(x)$ to be in the values picked but by $\hat{e}$, as here we have an inclusive comparison (line \ref{line:inner_max:hat_e_assignment} \cref{alg:inner_max}). In this case, the point-wise $q(u(y))$ can be strictly greater than $\lim_{x\to u(y)^{-}}\!q(x)$ or otherwise. We therefore have two cases:

\textbf{We will additionally assume $\lim_{x\to u(y)^{-}}\!q(x) < q(u(y))$.} As this can only happen if we look at a break-point that ends open and starts closed, in order for $q(u(y))$ to be extracted by $\dominatingpoly$, we have to check that the break-point is correctly assigned to the right interval. As we keep track of the open/closeness properties at the start of the interval (line \ref{line:max_out_exclusive_inclusive_start}) and assign it correctly, potentially overriding the previous interval, as every assignment starts with $\$$. We therefore look at the correct segment, so $m$ starts with a segment at $y$ that starts closed. Therefore $q(u(y))$ is extracted by $\dominatingpoly$. As $q(u(y)) \geq \colorbox{cyan!15}{(2)}$ if set \colorbox{cyan!15}{$(2)$} is not empty, and $q(u(y)) > \lim_{x\to u(y)^{-}}\!q(x) $, we have $q(u(y))=\hat{q}(i_1) \geq \hat{e}$, where $i_1 \in C_{\hat{e}}$. We can therefore enumerate all the possibilities for $i_2$ and check whether $m$ returns the correct value for this: In case $\hat{q}(i_2) \geq \hat{e}$, we have line \ref{line:max_out_assign_q_1} and $m(i_1)=m(y)=\hat{q}(y)$ (as ``$\$=[$''). In case $\hat{q}(i_2) < \hat{e}$, we can either have line \ref{line:max_out_assign_q_2} in case $\hat{q}(i_1) = \hat{e}$, then $m(i_1)=m(y)=\hat{e}=\hat{q}(y)$. Or we have $\hat{q}(i_1) > \hat{e}$, then we have line \ref{line:max_out_assign_q_3} and again $m(i_1)=m(y)=\hat{q}(y)$ (as ``$\$=[$''). \checkmark

\textbf{We will now assume $\lim_{x\to u(y)^{-}}\!q(x) \geq q(u(y))$}. It is important to note that in case the break-point is starts open/ends closed we actually have $\lim_{x\to u(y)^{-}}\!q(x) = q(u(y))$. So let us quickly consider this case, so when $y$ is the end of the segment ($y=i_2$). So from our assumptions of the maximum being in set \colorbox{yellow!25}{$(1)$}, we know that $\hat{q}(i_2) \geq \hat{e}$. As for all possibilities (line \ref{line:max_out_assign_q_1} and \ref{line:max_out_assign_q_2}) the assignment is closed to the right, so after running this iteration we have $m(y)=m(i_2)=\hat{q}=\lim_{x\to u(y)^{-}}$. In order for the break-point to belong to the previous interval, the next assignment must start open (``$\$=($''), which is respect in all the possible next assignments (so line \ref{line:max_out_assign_e_1}, \ref{line:max_out_assign_q_1}, \ref{line:max_out_assign_e_2} and \ref{line:max_out_assign_q_3}). Therefore, we have that $m$ at $y$ falls into the previous interval and therefore $m(y)=\lim_{x\to u(y)^{-}}\!q(x)$. In the other case, so we start closed and end open ($y=i_1$), the argument is the same as in the $\lim_{x\to u(y)^{-}}\!q(x) < q(u(y))$ but with the assumption that $q(u(y))=\hat{q}(i_1) \leq \hat{e}$ as $\lim_{x\to u(y)^{-}}\!q(x)$ falls into $\hat{e}$.\checkmark

\medskip
\textbf{Case III: Assuming the maximum is in set \colorbox{cyan!15}{$(2)$}} (more detailed: we assume the supremum \textbf{strictly} occurs inside the interval either as a left/right limit at a break-point or at an extreme-point). Let $i_1,i_2$ be the left/right boundaries of the interval of $m$ where $y$ falls into. Let $e^\prime$ be the $\max$ over set \colorbox{cyan!15}{$(2)$}, so $e^\prime=\max \{\,q'(e)\mid e\in(\mathrm{E}_q\cup\mathrm{B}_q)\cap\mathsf{int}([l(y),u(y)])\,\}$ (which therefore must be non-empty).

Furthermore, we will do another-case distinction based on whether $l$,$u$ are non-constant or not:

\textbf{Case III.I: $l$ and $u$ are non-constant} As the breakpoints for $m$ contain the set of breakpoints/extreme-points for $q \circ l$ and $q \circ u$, we know that
\begin{align}
    e^\prime &= \max \{\,q'(e)\mid e\in(\mathrm{E}_q\cup\mathrm{B}_q)\cap\mathsf{int}([l(y),u(y)])\,\}\\
    &= \max \{\,q'(e)\mid e\in(\mathrm{E}_q\cup\mathrm{B}_q)\cap[l(y'),u(y')]\,\}\\
    &= \hat{e}
\end{align}
for $y' \in (i_1,i_2)$, $l' = \max\{l(i_1),l(i_2)\}$ and $u' = \min \{u(i_1),u(i_2)\}$. Then the last equality is trivially true if $y$ is not at the boundary of $i_1,i_2$ (as no points of $\mathrm{E}_q\cup\mathrm{B}_q$ can lie inside this interval). If $y$ is at the boundary of the interval $i_1,i_2$, we have to be careful as we know could have potentially add the limit from outside $[l(y), u(y)]$. Fortunately, all is fine. We know that on $[l(y), u(y)]$, $e^\prime$ is dominant by our assumption of Case III (the maximum is in set \colorbox{cyan!15}{$(2)$}). But what if $l'=l(y)$ or $u'=u(y)$, then we would have added, via $q'$, the limit from outside. So the limit from the right towards our $y$, which is not inside $[l(y), u(y)]$. This can only lead to a a problem if we have a closed end/open start at $y$, as the right at $y$ does not necessarily coincide with $q(y)$ (it does otherwise). But in this case the picked up interval $i_1,i_2$ must be one where $y$ must be on the right boundary, so it is closed to the right (it would not have been picked up for the left boundary). But in this case, we can use our assumption of Case III (max must lie in the $int(I(y))$) as $\hat{q}(y)$ (point wise) would evaluate to the right-limit and is dominated by the interior. So $e^\prime=\hat{e}$ for both cases. Now let $\hat{q}$ be the dominant polynomial on $i_1,i_2$ as chosen by $\dominatingpoly$, so either $q \circ l$ or $q \circ u$ or the function that maps to minus infinity. As we know that $\hat{q}$ is monotonic, we can either have the case of line $19$, in which case $m(y)=\hat{e}=e^\prime$, or $\hat{q}$ has to cross $\hat{e}$ for $\hat{q}(y) < \hat{e}$ to hold (assumption Case III). 

Then we either have a growing (case if-statement line $23$) or falling (case if-statement $27$) $\hat{q}$ on the segment $i_1,i_2$. By assumption, we must be in the case with $\hat{q}(y) < \hat{e}$, so one of the bounds is actually the $i_{break}$ and we are in the range covered by line $25$ or $30$ with $y \neq i_{break}$ by our assumption. If $y$ is inside the interval, we now have $m(y)=\hat{e}=e^\prime$. If $y$ is the boundary $i_1$ or $i_2$, we now have $\hat{q}$ to be, in case it is a breakpoint, either the ending or starting polynomial. This depends on whether the interval starts closed or open. But by our assumption of set $(2)$ dominating set $(1)$, we have that $\hat{e}$ is dominating the ending polynomial (in case we have a closed end) or both (in case we have an open end). So we still fall into either case $19$, $25$ or $30$ and therefore $m(y)=\hat{e}=e^\prime$. \checkmark

\end{proof}

\subsection{Complexity Results for $\wfamily{\mathsf{PP}}$ and $\wfamily{\mathsf{PEP}}$}
\label{app:mp-map-complexity-wfamilies}

The idea behind analyzing the computational complexity behind the message-passing algorithm is to focus on the number of pieces generated during the message-passing algorithm, as the symbolic maximum is tractable, which follows from the tractability of the approximate root-enumeration \cite{schonhage1982fundamental} even when the degree of the polynomial is an input (runtime $\tilde{O}(d^3log(b))$ for approximation error $e=\frac{1}{2^b}$). As extending the $\max$ to the $\argmax$ results also in a still tractable complexity (the symbolic $\argmax$ piecewise polynomial has as many pieces as the $\max$ polynomial but only a degree of at most $1$), the theoretical analysis boils down to analyzing the total number of pieces sent.

In order to simplify the theoretical analysis, we will first analyze the number of pieces resulting for computing the symbolic max over a piecewise polynomial and symbolic, affine, upper and lower bound.

\begin{proposition}
\label{prop:complexity_max_out}
Computing the symbolic maximum $m(y)=\max_{x \in [l(y) \coloneq a\cdot y+b, u(y)\coloneq a^\prime\cdot y +b^\prime]}p(x)$ for a polynomial $p$ of degree $q$ with $m \geq 1$ pieces results in at most $8mq + 4m + 4$ pieces.
\end{proposition}
\begin{proof}
We can focus on the number of breakpoints enumerated by the $\maxout^{\mathsf{PP}}$ algorithm when applied to the polynomial $p$. We can focus on the algorithm, especially the $\preparebreaks$-function. Here we have:
\begin{align}
    |E_q| & \leq mq & \text{(extreme points of $p$)}\\
    |B_q| & \leq (m+1) & \text{(breakpoints points of $p$)}\\
    |C_{\tilde{e}}| &\leq 2(mq + (m + 1)) & \text{(after $l^{-1}(E_1 \cup B_q)$ and $u^{-1}(E_1 \cup B_q)$)}\\
    |C_{switch}| & \leq (2m)q & \text{(we have $2m$ pieces to between $p \circ l$ and $p \circ u$ with max. $q$ roots)}\\
    |C_{bounds}| & \leq 1 & \text{(we can have only one intersection between $l$ and $u$)}\\
    |C_{total}| & \leq 2(2(mq + (m + 1)) + 2mq + 1) & \text{(at most one additional break between those points $\maxout^{\mathsf{PP}}$)}\\
    &= 8mq + 4m + 4 
\end{align}
\end{proof}

In order to derive the complexity, we will take a similar approach to \citet{ZengICML20}. We will start with adapting proposition $22$ from \citet{ZengICML20} to our setting:

\begin{proposition}
\label{prop:mp_map_pieces_compute_msg}
Suppose the variables $X_i$ and $X_j$ are connected in the factor graph by factor $\factor{ij}$ associated to $\formula_{ij}$ of size $c$. Then:
\begin{enumerate}
    \item the number of pieces in $m_{X_i \rightarrow \factor{ij}}$ is bounded by $\sum_S m_S$, with $m_S = |m_{f_S \rightarrow X_i}|$ and $S \in \mathsf{neigh}(X_i)/\factor{ij}$.
    \item \textbf{for $\wfamily{\mathsf{PP}}$}: the number of pieces in $m_{\factor{ij} \rightarrow X_j}$ is bounded by $(2c+c^2)(c\cdot (8mq + 4m + 4)c^2\cdot q)$, with $m=|m_{X_i \rightarrow \factor{ij}}|$, $q = q_i{+}\max_{h \in [1, \dots, m]}(\mathsf{deg}(\mathsf{pieces}(m_{X_i \rightarrow \factor{ij}})[h])$ with $q_i$ being the max. degree of the polynomial factor associated with $X_i$ in the weight-function attached to $\formula_{ij}$.
    \item \textbf{for $\wfamily{\mathsf{PEP}}$}: the number of pieces in $m_{\factor{ij} \rightarrow X_j}$ is bounded by $(2c+c^2)(c\cdot (8mq' + 4m + 4)c^2\cdot q')$ for $q' = \max\{q_i,\max_{h \in [1, \dots, m]}(\mathsf{deg}(\mathsf{pieces}(m_{X_i \rightarrow \factor{ij}})[h])\}$,
\end{enumerate}
\end{proposition}

\begin{proof}
Statement $1$ is directly taken from proposition $22$ from \citet{ZengICML20}. It holds for the intersections between the messages due to the overall product resulting in the number of pieces being at most the sum over the number of individual pieces. \checkmark\\
The second statement is more challenging and analyzes the number of pieces returned by the algorithm $\computemsgs$. A key difference between \citet{ZengICML20} and our approach is that we compute the symbolic maximum over a piecewise functions, whereas \citet{ZengICML20} integrates a single polynomial at a time. Therefore the critical points $\mathcal{P}$ in \citet{ZengICML20} is both a function of the size of the formula $\formula_{ij}'$ ($=c$, defined in line \ref{line:mp_map_compute:formula_and_piecewise} in \cref{alg:mp_map_compute_pp}) and the number of pieces $m$. Whereas we only take into account the global bounds (if existing), so treat $m_{X_i \rightarrow \factor{ij}}$ as a single piece made of piecewise functions when constructing $\mathcal{P}$ and later perform the $\maxout^{\mathsf{PP}}$ over the piecewise polynomial. We therefore have $|\mathcal{P}| \leq (2c + c^2)$ following from proposition $22$ from \citet{ZengICML20} for a single piece message. We turn our attention towards the inner loop in $\computemsgs$. Inside an critical interval, we can only have $c$ different pieces for the polynomial factor associated the $\formula_{ij}$. We therefor have $c$ times at most $(8mq + 4m + 4)$ pieces (following proposition \ref{prop:complexity_max_out}). In order to compute the point wise max, we first need to compute the number of intersections between the pieces, which bounded by the sum, so we have at most $c(8mq + 4m + 4)$ pieces. On each piece, we now have at most $c$ polynomials of degree $q$ to compare. As we need to do a pairwise-comparison, with each comparison resulting in at most $q$ pieces, we have $c^2q$ pieces resulting from the symbolic point wise maximum between the polynomials. Putting all the factors together, we arrive at $(2c+c^2)(c\cdot (8mq + 4m + 4)c^2\cdot q)$ \checkmark\\
For weights in $\wfamily{\mathsf{PEP}}$, the derivation is analogous to $\wfamily{\mathsf{PP}}$ except that the degree of the polynomial does not grow as a sum of the two degrees but via the maximum over $q_i$ and $\{\mathsf{deg}(\mathsf{pieces}(m_{X_i \rightarrow \factor{ij}})[h]) | h \in [1, \dots, m] \}$. This is due to the product of the exponentiated polynomials resulting in the exponential of the sum of the polynomials and $\mathsf{deg}(p_i + p_h)$ is $\max\{\mathsf{deg}(p_i), \mathsf{deg}(p_h)\}$.
\end{proof}

\textbf{It is important to note here that these worst-case bounds are really a worst case scenario that can play out with significantly less complexity in practice.} For example, it assumes that every comparison between two polynomials $q_1$,$q_2$ of order $\tilde{q}$ in the point-wise maximum will always explode into $q$ pieces, the maximum number of roots of $q_1 - q_2$, which all have to lie inside the interval we are currently looking at. And every further comparison has to again result in $1$ pieces and so on the recursion goes until all comparisons are made, so in total the $\frac{c (c-1)}{2}q\leq c^2\cdot q$ pieces.

We now construct the adjacency matrix $M$ for the graph $G$ that represents the DAG used to run our message passing. So $M_{f_S, f_S^\prime} = 1$ if we compute $m_{f_S \rightarrow f_S^\prime}$ during message passing. One can easily see that $M$ is nilpotent as $M$ is the adjacency matrix over a DAG and since $m$ matrix-powers of an adjacency matrix represent $m$-step reachability in the DAG. The order of the nilpotent matrix $M$ can therefore only be at most the diameter of the factor graph, so the longest path between any two vertices \citep[Prop.~21]{ZengICML20}.

First, we will focus on $\wfamily{\mathsf{PP}}$, as later the proof for $\wfamily{\mathsf{PEP}}$ is analogous. We now introduce a vector $v^{(t)}$ that captures the max. number of pieces in the message per factor at time-step $t$. Calculating $\sum_{t=0}^d \sum_s v^{(t)}_s$ therefore bounds the number of pieces in the messages in MP-MAP. In order to approximate this, we need the maximum degree of the univariate polynomial, which we will call $q_{max}$; in particular, for a single bivariate piece $q_1^1(x_1)q_1^2(x_2)$ we compute $q_{max}$ via $\max(\mathsf{deg}(p_1), \mathsf{deg}(p_2))$. We first focus on the relationship between $v^{(t)}$ and $v^{(t-1)}$. We start by reducing $v^{(1)}$ to $v^{(0)}$:

\begin{align}
    v^{(1)}_j \leq \sum_i M_{ij}(2c+c^2)(c\cdot (8v^{(0)}_i (2q_{max}) + 4 v^{(0)}_i + 4)c^2\cdot (2q_{max}))\\
    = (2c+c^2)c\cdot c^2\cdot 4 \sum_i M_{ij}(2M_{ij} v^{(0)}_i (2q_{max}) + M_{ij} v^{(0)}_i + 1)\cdot (2q_{max})
\end{align}

with $c$ being the maximum size over all $\formula_{ij}$. We have the outer sum as we combine all the incoming messages using a product, leading a bound over the number of pieces of the sum of the individual pieces. The degree of the polynomial, at the point of being passed to $\maxout^{\mathsf{PP}}$, is $2q_{max}$ as we have a polynomial of degree $q_{max}$ coming from the leaves associated with an univariate formula $\formula_{i}$ and then we combine it with another polynomial associated to $\formula_{ij}$, also with at most a degree of $q_{max}$. So we arrive at a total degree of $2q_{max}$. This can be generalized:
\begin{align}
    ||v^{(t)}||_1 \leq (2c+c^2)c\cdot c^2\cdot 4 (2t(n{-}d)q_{max})||M(2v^{(t-1)}2t(n{-}d)q_{max} + v^{(t-1)} + 1)||_1
\end{align}
We have the total degree $2tq_{max}$ as we associate a factorizable polynomial $\formula_{ij}$ with at most a degree of $q_{max}$ per variable, of which we have two. We also have the multiplicative factor of $(n{-}d)$ due to the multiplication in \gathermsgs (of which we can have max. $(n{-}d)$).

Denote with $s$ the cardinality of our set of factors $\mathcal{F}$, and $d$ the diameter of $\mathcal{G}$. We can now use the following inequalities $||M||_1 \leq s$ and $||v^{(0)}||_1 \leq cs$ to derive a bound for $||\sum_{t=0}^d v^{(t)}||_1$.

\begin{theorem}
    \label{th:mp_complexity_pp}
    For weights in $\wfamily{\mathsf{PP}}$, we have an overall bound of $||\sum_{t=0}^d v^{(t)}||_1 = \mathcal{O}(c^{5d+1} (n{-}d)^{2d} q_{max}^{2d} (2n)^{d+1} (d!)^2)$ for the number of generated messages during the message-passing. 
\end{theorem}

\begin{proof}
\begin{align}
    ||v^{(t)}||_1 &\leq (2c+c^2)c\cdot c^2\cdot 4 (2t(n{-}d)q_{max})||M(2v^{(t-1)}2t(n{-}d)q_{max} + v^{(t-1)} + 1)||_1\\
    &\leq (2c+c^2)c\cdot c^2\cdot 4 (2t(n{-}d)q_{max})(||M(2v^{(t-1)}2t(n{-}d)q_{max} + v^{(t-1)})||_1 + s)\\
    &\leq (2c+c^2)c\cdot c^2\cdot 4 (2t(n{-}d)q_{max})(s||(2v^{(t-1)}2t(n{-}d)q_{max} + v^{(t-1)})||_1 + s)\\
    &\leq (2c+c^2)c\cdot c^2\cdot 4 (2t(n{-}d)q_{max})(s||v^{(t-1)}(4t(n{-}d)q_{max} + 1)||_1 + s)\\
    &\leq (2c+c^2)c\cdot c^2\cdot 4 (2t(n{-}d)q_{max})(4t(n{-}d)q_{max} + 1)(s||v^{(t-1)}||_1 + s)\\
    &\leq (2c+c^2)c^3\cdot 8 (4t^2(n{-}d)^2q_{max}^2 + q_{max}^2)(s||v^{(t-1)}||_1 + s)\\
    &\leq ((2c+c^2)c^3\cdot 8\cdot5(n{-}d)^2q_{max}^2s) t^2(||v^{(t-1)}||_1 + 1)\\
    &\leq ((2c+c^2)c^3\cdot 8\cdot5(n{-}d)^2q_{max}^2s) t^2 2||v^{(t-1)}||_1\\
    &\leq K t^2||v^{(t-1)}||_1
\end{align}
with $K=(2c+c^2)c^3\cdot 5\cdot8\cdot2 (n{-}d)^2q_{max}^2s$.

Using this recurrence, we focus on the overall number of messages:

\begin{align}
    ||\sum_{t=0}^d v^{(t)}||_1 &\leq \sum_{t=0}^d||v^{(t)}||_1\\
    &\leq \sum_{t=0}^d[K t^2||v^{(t-1)}||_1]\\
    &\leq \sum_{t=0}^d[K^t (t!)^2(||v^{(0)}||_1]\\
    &\leq \sum_{t=0}^dK^t (t!)^2cs\\
    &\Rightarrow ||\sum_{t=0}^d v^{(t)}||_1 = \mathcal{O}(cs K^d (d!)^2) \quad \text{(final term dominates, see ratio of the terms in the sum)}
\end{align}

We also know that $K=\mathcal{O}(c^5 q_{max}^2s)$, therefore:
\begin{align}
    ||\sum_{t=0}^d v^{(t)}||_1 &= \mathcal{O}(cs (c^5 (n{-}d)^{2} q_{max}^2s)^d (d!)^2)\\
    &= \mathcal{O}(c^{5d+1} (n{-}d)^{2d} q_{max}^{2d} s^{d+1} (d!)^2)
\end{align}
As we know that $s \leq 2n$, we have $||\sum_{t=0}^d v^{(t)}||_1 = \mathcal{O}(c^{5d+1} (n{-}d)^{2d} q_{max}^{2d} (2n)^{d+1} (d!)^2)$.

\end{proof}

After deriving the computational complexity for weights in $\wfamily{\mathsf{PP}}$, we will now focus on weights in $\wfamily{\mathsf{PEP}}$. The general approach is the same as in theorem \ref{th:mp_complexity_pp} with an important difference: the bound on the degree of the exponentiated polynomial stays constant and does not increase in depth. This is due to the product of the exponentiated polynomials turning into the exponential of the sum of the polynomials, which does not increase the degree. 

\begin{theorem}
    \label{th:mp_complexity_pep}
    For weights in $\wfamily{\mathsf{PEP}}$, we have an overall bound of $||\sum_{t=0}^d v^{(t)}||_1 = d c^{5d+1}(2n)^{d+1}q_{max}^{2d}$ 
    for the number of generated messages during the message-passing. 
\end{theorem}

\begin{proof}
We first start with the expression of the recursion
\begin{align}
    ||v^{(t)}||_1 &\leq (2c+c^2)c\cdot c^2\cdot 4 \cdot q_{max}||M(2v^{(t-1)}q_{max} + v^{(t-1)} + 1)||_1\\
    &\leq (2c+c^2)c\cdot c^2\cdot 4 (q_{max})(s||(2v^{(t-1)}q_{max} + v^{(t-1)})||_1 + s)\\
    &\leq (2c+c^2)c\cdot c^2\cdot 4 (q_{max})(s||v^{(t-1)}(2q_{max} + 1)||_1 + s)\\
    &\leq (2c+c^2)c\cdot c^2\cdot 4 (q_{max})(2q_{max} + 1)(s||v^{(t-1)}||_1 + s)\\
    &\leq (2c+c^2)c\cdot c^2\cdot 4 (2q_{max}^2 + 1)(s||v^{(t-1)}||_1 + s)\\
    &\leq (2c+c^2)c\cdot c^2\cdot 4 (2q_{max}^2 + q_{max}^2)(s||v^{(t-1)}||_1 + s)\\
    &\leq (2c+c^2)c\cdot c^2\cdot 4 (3q_{max}^2)(s||v^{(t-1)}||_1 + s)\\
    &\leq (2c+c^2)c\cdot c^2\cdot 4 (3q_{max}^2)s(||v^{(t-1)}||_1 + 1)\\
    &\leq (2c+c^2)c\cdot c^2\cdot 4 (3q_{max}^2)s2||v^{(t-1)}||_1\\
    &\leq (2c+c^2)c\cdot c^2\cdot 8 (3q_{max}^2)s||v^{(t-1)}||_1\\
    &= K'||v^{(t-1)}||_1
\end{align}

with $K'=(2c+c^2)c\cdot c^2\cdot 8 (3q_{max}^2)s$.

Using this recurrence, we focus on the overall number of messages:

\begin{align}
    ||\sum_{t=0}^d v^{(t)}||_1 &\leq \sum_{t=0}^d||v^{(t)}||_1\\
    &\leq \sum_{t=0}^d[K' ||v^{(t-1)}||_1]\\
    &\leq \sum_{t=0}^d[(K')^t (||v^{(0)}||_1]\\
    &\leq \sum_{t=0}^d(K')^t cs\\
    &\Rightarrow ||\sum_{t=0}^d v^{(t)}||_1 = \mathcal{O}(d cs (K')^d)
\end{align}

As also know that $K'=\mathcal{O}(c^5 q_{max}^2s)$ and $s \leq 2n$, we have:
\begin{align}
    ||\sum_{t=0}^d v^{(t)}||_1 &= \mathcal{O}(d cs (K')^t)\\
    &= \mathcal{O}(d c2n (c^5 q_{max}^22n)^d)\\
    &= \mathcal{O}(d c^{5d+1}(2n)^{d+1}q_{max}^{2d}).
\end{align}

\end{proof}

We will now prove \cref{th:mp_tractability}, which we will restate first:
\begin{theorem}[Tractability of \mapal{}]
\label{th:mp_tractability_app}
If the global graph of \mapal{} has treewidth one and bounded diameter, and the density fulfills the TMC (\cref{def:tractable_mp_conditions}), then \mapal{} can be solved tractably.
\end{theorem}

\begin{proof}
We will first look at tractability of the algorithms \gathermsgs{} and \computemsgs{} for arbitrary densities that fulfill the TMC. A first important observation is that while while the symbolic supremum is in itself is tractable, as required by the TMC, nesting the symbolic suprenum $n$-times is not tractable, as the output-size of the symbolic suprenum can scale polynomially. So:
\[
\label{eq:mp_map_proof_tractability_general}
\sup_{(x_1, x_2, \dots, x_n)} f(\vx) = \sup_{x_1 \in [a,b]} f'(x_1,x_2) \left( \sup_{x_x \in [l_1(x_1), u_1(x_2)]} f''(x_2,x_3)\left( \sup_{x_3 \in [l_2(x_2), u_2(x_2)]} f'''(x_3,x_4) \dots \right)\right)
\]
is not tractable in $n$. In comparison, this is not the case of the point-wise product.\\
As \gathermsgs{} (and \gatherroot{}) consists of the pointwise-product operation, repeated calling of these methods is not an issue. In comparison, \computemsgs{} is more costly. First, we have a loop over the intervals spanned by the critical points (line \ref{line:mp_map_compute_generic:I}). As we've seen in the proof of \cref{prop:mp_map_pieces_compute_msg}, these are polynomially many in the number of atoms $c$ in $\formula_{ij}$. Then, we do the symbolic suprenum for each interval, which is tractable and therefore has a polynomial output-size. We then do the pointwise maximum between the resulting functions (line \ref{line:mp_map_compute_generic:max-pieces}), again tractable but also polynomial output-size. Calling it once is therefore tractable, but the problem is that nesting these calls is not tractable anymore. This is due to the same reason as \cref{eq:mp_map_proof_tractability_general} is not tractable: the output size is polynomial in input size.\\
The question therefore is, on what depends the recursion-depth of \computemsgs{} (so repeated application of \computemsgs{} to output generated by \computemsgs{}) during the computation of \mpmap{}? This becomes obvious once we pivot our attention from the pseuocode to the mathmatical definition of the messages in equation \ref{eq:msg-var-to-fac}, \ref{eq:msg-fac-to-var} and \ref{eq:msg-final}. As we traverse the tree (or multiple trees in case of a forest) recursively from root to children, this number depends on the longest path found in the graph $\mathcal{G}$, also called the diameter of the graph. But as we assume boundedness of the diameter, the length of the path can not be arbitrarily long, even in varying dimensions. An example of this would be the diameter of the \emph{STAR}-problems in the experiments-section for \mpmap{} \ref{sec:experiments_mp_map}. Therefore, \mapal{} is tractable if the global graph of \mapal{} has treewidth one and bounded diameter, assuming the density fulfills the tractable map-conditions.
\end{proof}

\FloatBarrier

\subsection{Additional Routines for $\mpmap$}

\begin{algorithm}
    \caption{$\bm\gatherroot$($X_i$)
    }
    \label{alg:mp_map_marg}
    \begin{algorithmic}[1]
        \INPUT $X_i$: variable
        \OUTPUT $q$: piecewise polynomial
        \STATE $Q \gets \{ \msg{\factor{j',i}}{X_i}{} \mid \forall j' {\in} \neigh(i)\} \cup \{ \msg{\factor{i}}{X_i}{} \}$\\
        \RETURN $\prod_i Q_i$ \COMMENT{point-wise product}
    \end{algorithmic}
\end{algorithm}

\begin{algorithm}
    \caption{$\bm\getmsgpieces$($\msg{X_j}{\factor{ij}}{}, I, \wfun, \formula_{ij}'$)
    }
    \label{alg:mp_map_get_msg_pieces}
    \begin{algorithmic}[1]
        \INPUT $\msg{X_j}{\factor{ij}}{}$: univariate $\wfamily{}$ element, $I$: tuple start/end, $\wfun$ bivariate $\wfamily{}$ element
        \OUTPUT $\mathcal{Q}$: list of (lower-bound, upper bound, univariate function, univariate function)
        \STATE $\mathcal{I} \gets \findsimbolicboundsin(I, \formula_{ij}')$ \COMMENT{list of linear inequalities for $X_j$ wrt $X_i$}
        \STATE $r \gets []$
        \FOR{$k \in [1, \mathsf{len}(\mathcal{Q})]$}
            \STATE $q_k^i(x_i)\cdot q_k^j(x_j) \gets \piece(\omega, \mathcal{Q}_k, I)$ \COMMENT{gets the factorized, active piece}
            \STATE $(l_k, u_k) \gets \mathcal{Q}_k$ \COMMENT{Current formula in the form $l_k(X_i) \leq X_j \leq u_k(X_i)$}
            \STATE $\append(r, (l_k, u_k, q_k^i(x_i), q_k^j(x_j)\cdot \msg{X_j}{\factor{ij}}{}))$
        \ENDFOR
        \RETURN $r$
    \end{algorithmic}
\end{algorithm}

\begin{algorithm}
    \caption{$\bm\maxout(q, l, u)$}
    \label{alg:max_out}
    \begin{algorithmic}[1]
        \INPUT $q$: piecewise polynomial or piecewise exponentiated polynomial, $l$: affine lower bound, $u$ affine upper bound
        \OUTPUT $m$: piecewise polynomial  or piecewise exponentiated polynomial%
        \IF{$\ispiecewiseexponentianted(q)$}
            \STATE $q' \gets \log q$ \COMMENT{piecewise $\log$}
        \ELSE
            \STATE $q' \gets q$
        \ENDIF
        \STATE $m \gets \maxout^{\mathsf{PP}}(q', l, u)$
        \IF{$\ispiecewiseexponentianted(q)$}
            \STATE $m' \gets \exp m$ \COMMENT{piecewise $\exp$}
        \ELSE
            \STATE $m' \gets m$
        \ENDIF
        \RETURN $m'$
    \end{algorithmic}
\end{algorithm}

\FloatBarrier

\subsection{Additional Routines for $\maxout$ and $\maxout^{\mathsf{PP}}$}

\begin{algorithm}
    \caption{$\bm\dominatingpoly(q_1, q_2, i_1, i_2)$)
    }
    \label{alg:mp_map_dominating_poly}
    \begin{algorithmic}[1]
         \INPUT $q_1$, $q_2$: monotone polynomial or constant map to $-\infty$, $i_1$ lower bound interval, $i_2$ upper bound interval ($q_1$ and $q_2$ do not intersect on $(i_1,i_2)$ or are equal)
        \OUTPUT polynomial $\hat{q} \in \{q_1, q_2\}$ such that $q_1 \leq \hat{q}$ and $q_2 \leq \hat{q}$ point-wise on $[i_1, i_2]$
        \STATE $q_1^\prime \gets \extractpolypiece(q_1, i_1, i_2)\text{ if } (i_1,i_2) \in \dom(q_1) \text{ else } (y \mapsto -\infty)$
        \STATE $q_2^\prime \gets \extractpolypiece(q_2, i_1, i_2)\text{ if } (i_1,i_2) \in \dom(q_2) \text{ else } (y \mapsto -\infty)$
        \STATE $T \gets \{i \mid i \in \{\ i_1, i_2\}\}$
        \STATE $T \gets {0} \text{ if } T = \varnothing$
        \STATE $e_1 \gets \max\{q_1^\prime(i) \mid i \in T\}$
        \STATE $e_2 \gets \max\{q_2^\prime(i) \mid i \in T\}$
        \IF{$e_1 > e_2$}
            \RETURN $q_2$
        \ELSIF{$e_1 > e_2$}
            \RETURN $q_2$
        \ELSIF{$e_1 = -\infty$}
            \RETURN $q_1^\prime$ \COMMENT{does not matter}
        \ELSE
            \STATE \COMMENT{We have to check what happens inside the interval, both are defined but we have the same values for our points}
            \STATE $i_{min} \gets \min T$
            \STATE $r \gets q_1^\prime - q_2^\prime$
            \STATE $d \gets \highestnvderivativeat(r, i_{min})$
            \IF{$d \geq 0$}
                \STATE \COMMENT{either $q_1^\prime$ dominates or equal}
                \RETURN $q_1^\prime$
            \ELSE
                \RETURN $q_2^\prime$
            \ENDIF
        \ENDIF
    \end{algorithmic}
\end{algorithm}

\begin{algorithm}
    \caption{$\bm\preparebreaks(q, l, u)$)
    }%
    \label{alg:prepare_breakpoints}
    \begin{algorithmic}[1]
    \INPUT $q$: piecewise polynomial, $l$: affine function, $u$: affine function
    \OUTPUT $V_q$, $C_{\hat{e}}$, $C_{\mathrm{switch}}$, $C_{\mathrm{bounds}}$ sets of points
    \STATE $E_q \gets \extremepoints(q)$
    \STATE $B_q \gets \breakpoints(q)$
    \STATE $q'(x) \gets \max\max(\{\begin{array}[t]{@{}l@{}}
            \lim_{x \to b^-} q(x)\ \text{if defined},\\
            \lim_{x \to b^+} q(x)\ \text{if defined}
        \})\end{array}$
    \STATE $V_q \gets \{b \mapsto q'(b) \mid b \in E_q \cup B_q \}$
    \STATE $\Omega \gets \{y \in \mathbb{R} \mid l(y) {\leq} u(y)\}$
    \STATE $C_{\hat{e}} \gets \bigcup_{f \in {l,u}} \{ f^{-1}(y) \mid y \in \mathrm{E}_q {\cup} \mathrm{B}_q \land f {\neq} \text{const.}\}\cap \Omega$
    \STATE $C_{\mathrm{switch}} \gets \roots((q \circ u) {-} (q \circ l)) \cap \Omega$
    \STATE $C_{\mathrm{bounds}} \gets \bounds(l, u)$
    \RETURN $V_q$, $C_{\hat{e}}$, $C_{\mathrm{switch}}$, $C_{\mathrm{bounds}}$
    \end{algorithmic}
\end{algorithm}

\begin{algorithm}
    \caption{$\bm\innermax(l, u, i_1, i_2, V_q)$
    }%
    \label{alg:inner_max}
    \begin{algorithmic}[1]
     \INPUT $l$: affine function, $u$: affine function, $i_1$: lower bound interval, $i_2$: upper bound interval, $V_q$: map from point to value
     \OUTPUT $\hat{e}$: the maximum of $V_q$ inside $l([i_1,i_2])$ and $u([i_1,i_2])$ or $-\infty$
    \STATE $l^\prime, u^\prime \gets \max\{l(i_1), l(i_2)\}, \min\{u(i_1), u(i_2)\}$
            \STATE $
            \hat{e} \gets
            \max(\begin{array}[t]{@{}l@{}}

            \{V_q(b) \mid b \in V_q \land l^\prime \leq b \leq u^\prime \}\\
            \cup \{-\infty\})

            \end{array}
    $ \alglinelabel{line:inner_max:hat_e_assignment}
    \RETURN $\hat{e}$
    \end{algorithmic}
\end{algorithm}

\begin{algorithm}
    \caption{$\bm\simplify(m)$)
    }
    \label{alg:simplify_piecewise}
    \begin{algorithmic}[1]
    \INPUT $m$: piecewise polynomial
    \OUTPUT piecewise polynomial but with redundant pieces removed
    \IF{$\isempty(m)$}
        \RETURN $m$
    \ELSE
        \STATE $m^\prime \gets \emptypiecewise()$
        \STATE $\$^{\prime}, l_{last},u_{last}, \$^{\prime\prime},p_{last} \gets \topieces(m)[0]$
        \STATE \COMMENT{$\$^{\prime}$, $\$^{\prime\prime}$ denote whether the interval starts open/closed}
        \FOR{$\$^{start},l,u,\$^{end},p \in \gets \topieces(m)[1:]$}
            \IF{$p = p_{last}$}
                \STATE $u_{last} \gets u$
                \STATE $\$^{\prime\prime} \gets \$^{end}$
            \ELSE
                \STATE $m^\prime|_{\$^{\prime}l_{last},\,u_{last}\$^{\prime\prime}} \gets p^\prime$
                \STATE $\$^{\prime}, l_{last},u_{last}, \$^{\prime\prime},p_{last} \gets \$^{start},l,u,\$^{end},p$
            \ENDIF
        \ENDFOR
        \STATE $m^\prime|_{\$^{\prime}l_{last},\,u_{last}\$^{\prime\prime}} \gets p^\prime$
        \RETURN $m^\prime$
    \ENDIF
    \end{algorithmic}
\end{algorithm}

\FloatBarrier

\section{Particle, Constraint-Aware Adam Optimizer}%
\label{sec:app_pcadam}
In this section, we provide details on \pcadam{}, the particle-based, constraints-aware version of \adam{} that we present as a side contribution.
We use \pcadam{} both as a baseline to compare with, as well as a convex-polytope optimizer combined with \pamap{} in the data imputation experiments (\cref{sec:experiments}).
Given the constraints $\formula$ and a density $p$ to be maximized, and an initial batch of points $\mathcal{X}_0 = \{\vx_1, \dots, \vx_n\}$, \pcadam{} keeps track of the best feasible solution found so far, initialized to $-\infty$ (lines~\ref{line:pcadam:init_m}--\ref{line:pcadam:init_xstar}). Then, for a fixed number of iterations, it updates the batch of points using a gradient-based optimizer (line~\ref{line:pcadam:update}). After each update, it checks which of the new points satisfy the constraints (line~\ref{line:pcadam:check_sat}) and updates the best solution found so far accordingly (lines~\ref{line:pcadam:update_m}--\ref{line:pcadam:update_xstar}). Finally, it returns the best solution found (line~\ref{line:pcadam:return}).
\begin{algorithm}
    \caption{\textbf{\pcadam}$(\formula, p, \mathit{it}, \mathcal{X}_0)$}%
    \label{alg:pcadam}
    \begin{algorithmic}[1]
        \INPUT $\formula$: \smtlra{} formula, $p$: density function, $\mathcal{X}_0$: initial batch of points, $\mathit{it}$: iterations
        \OUTPUT $\vx^*$: global best point, $m^*$: maximum density found
        \STATE $m^* \gets -\infty$ \alglinelabel{line:pcadam:init_m}
        \STATE $\vx^* \gets \emptyset$ \alglinelabel{line:pcadam:init_xstar}
        \STATE $\mathcal{X} \gets \mathcal{X}_0$ 
        \FOR{$i=1$ to $\textit{it}$}
            \STATE $\mathcal{X} \gets \mathsf{update}(\mathbf{X}, \nabla p)$ \alglinelabel{line:pcadam:update} \COMMENT{Gradient-based step}
            \FOR{\textbf{each} $\vx \in \mathcal{X}$}
                \IF{$\vx\models\formula$ \textbf{and} $p(\vx) > m^*$} \alglinelabel{line:pcadam:check_sat}
                    \STATE $m^* \gets p(\vx)$ \alglinelabel{line:pcadam:update_m}
                    \STATE $\vx^* \gets \vx$ \alglinelabel{line:pcadam:update_xstar}
                \ENDIF
            \ENDFOR
        \ENDFOR
        \RETURN $\vx^*, m^*$ \alglinelabel{line:pcadam:return}
    \end{algorithmic}
\end{algorithm}

\newpage

\section{Experiments}%
\label{sec:app_experiments}
\subsection{Implementation}%
\label{sec:app_implementation}
We have implemented our algorithms in Python, building on top of several existing libraries. 
For \pcadam{}, we have used Pytorch's Adam optimizer for parallel unconstrained optimization.
For \pamap{}, we have used the SAE4WMI enumeration algorithm~\cite{spallitta2024enhancing} implemented in wmpy\footnote{\url{https://github.com/unitn-sml/wmpy}}; for numerical constrained optimization over convex polytopes, we have used SciPy's optimization routines; for Lasserre's method, we implemented the moment hierarchy on top of SumOfSquares.py\footnote{\url{https://github.com/yuanchenyang/SumOfSquares.py}}.
For MP-MAP, we have implemented our message-passing using SymPy for the symbolic computations over the polynomials and pySMT in order to query SMT solvers for logical operations. Parts of the code are adapted from MP-WMI \cite{ZengICML20}, which has a similar high-level structure.

\subsection{Details for STAR, SNOW and PATH}
\label{sec:app_details_tree_shaped_primal}

We start by sampling a random, $N$-variable SMT formula of the shape of either \emph{STAR} (star-shaped primal graph), \emph{SNOW} (ternary-tree shaped primal graph) or \emph{PATH} (linear-chain shaped primal graph). On the resulting support, we generate a non-negative, piecewise polynomial function by first randomly selecting literals to attach the polynomials to. Our weight-functions then looks like this: $\prod_{l \in \text{random-literals}} (\text{if $\vx \models l$ then $q_l(\vx)$ else 1})$. As each literal can at most mention two variables, the polynomial can also only be over those two variables, as it has to respect the scope of the literal it is attached to. Furthermore, as we want to have separable polynomials, we generate the polynomials by generating two univarate polynomials of degree $deg$ and form the product. Each univariate polynomial is generated as follows:

\begin{algorithm}
    \caption{$\bm\randompolyunivar(var, deg, bound_{lower}, bound_{upper}, pareto_{upper})$)
    }
    \label{alg:rand_poly_univar_tree_primal_graph}
    \begin{algorithmic}[1]
    \STATE $generateDegree \gets int(deg / 2)$
    \STATE $numRoots \gets generateDegree - 1$
    \STATE $roots \gets [r \sim U[bound_{lower}, bound_{upper}] \mid \_  \in [1, numRoots]]$
    \STATE $\mathsf{sample\text{-}coeff} \gets \lambda. \textbf{ return } s \sim \mathsf{TruncatedPareto}(min=2.0, max=pareto_{upper}, eps=0.01)$
    \STATE $mons \gets [(\mathsf{sample\text{-}coeff}()\cdot (var - r)) \mid r \in roots]$
    \STATE $poly^{unsquared}_{derivative} \gets \prod_{m \in mons} m$
    \STATE $poly^{unsquared} \gets \mathsf{integrate}(poly^{unsquared}_{derivative}, var)$
    \RETURN $(poly^{unsquared}(var))^2+1$
    \end{algorithmic}
\end{algorithm}

Additionally, for the literals specifying the global bounds (so $x_i \geq bound_{lower}$ and $x_i \leq bound_{upper}$), we attach the polynomial $q_{global}(x_i) = (x - bound_{lower})*(bound_{upper} - x)$, as otherwise the maximum is too commonly found at the global bounds.

We generate the inequalities by sampling two random points inside our global bounds, compute the connecting line which forms our decision boundary and randomly choosing either the left or right as the valid half-space.

We generate the datasets with the following configurations. For each shape $s \in \{\text{STAR}, \text{SNOW}, \text{PATH}\}$ we generate problems over the combination over the following parameters:
\begin{itemize}
    \item $N$: $2$, $4$, $6$, $8$, $10$
    \item $deg$ (squared, per univariate poly): $2$, $3$, $4$
    \item $n_{clauses}$: $2$, $3$
    \item $n_{literals}$: $2$, $4$
    \item $pareto_{upper}$: $N \leq 6 \Rightarrow 15$, $N \geq 8 \Rightarrow 2.5$, 
    \item $bound_{lower}$: $-1$
    \item $bound_{upper}$: $1$
\end{itemize}
Of each configuration, we generate $2$ random problems. 

Here, the number of random literals $n_{literals}$ is per clause (so times $n_{clauses}$). Additionally, we have the inequalities specifying the global bounds.

\begin{figure}[ht]
\centering
\begin{tabular}{ccc}
  \includegraphics[width=0.3\linewidth]{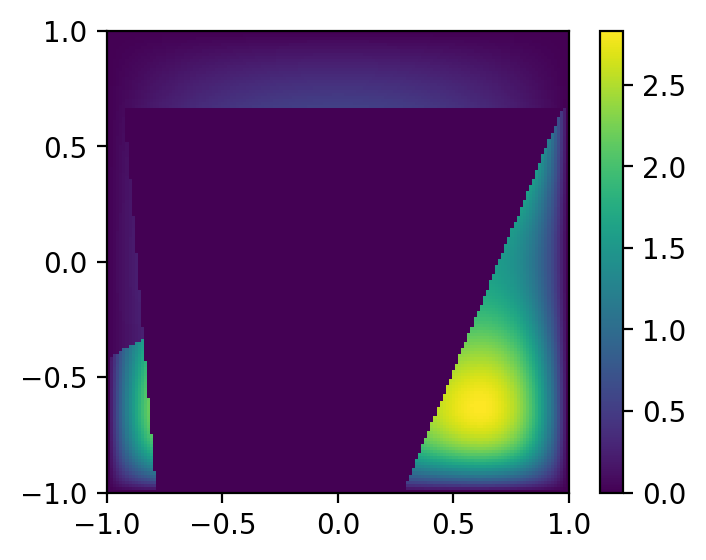} &
  \includegraphics[width=0.3\linewidth]{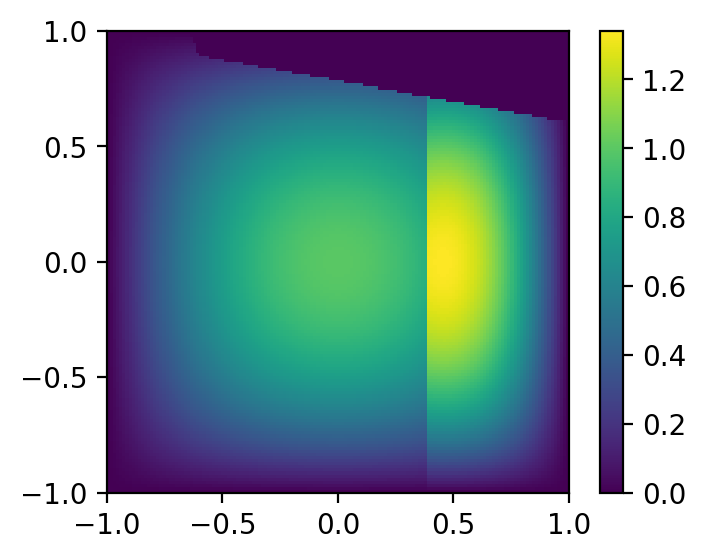} &
  \includegraphics[width=0.3\linewidth]{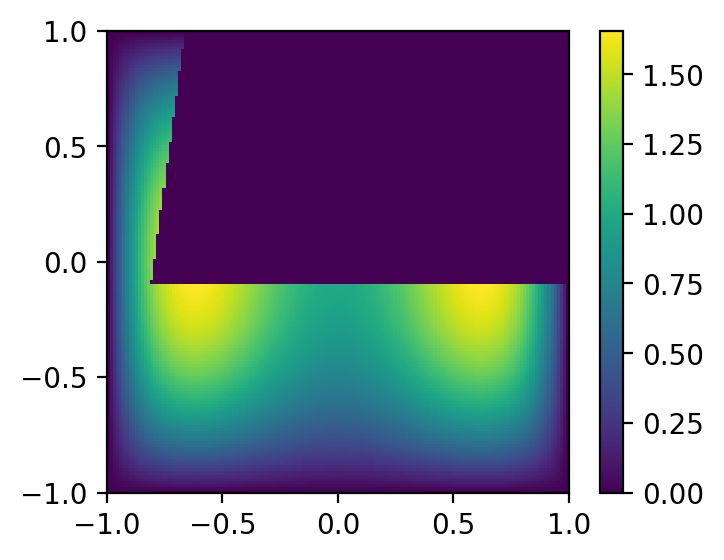} \\
  \includegraphics[width=0.3\linewidth]{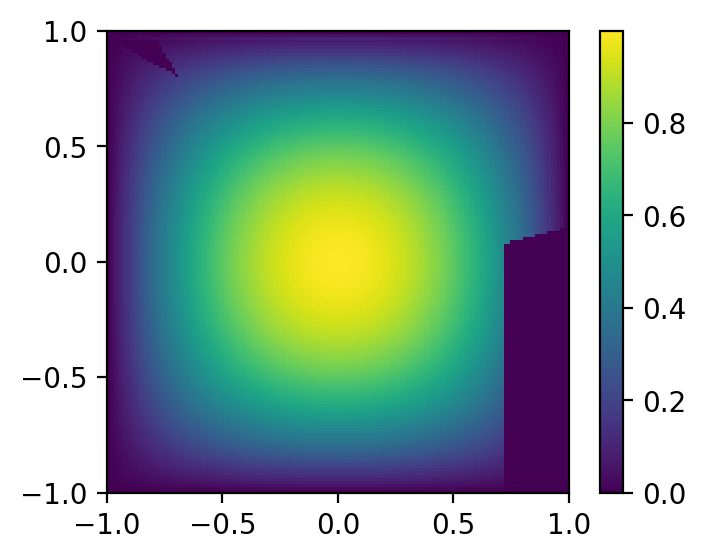} &
  \includegraphics[width=0.3\linewidth]{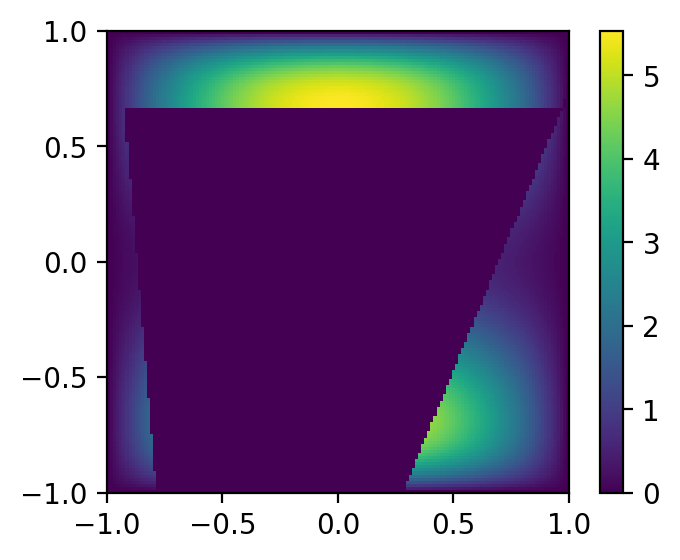} &
  \includegraphics[width=0.3\linewidth]{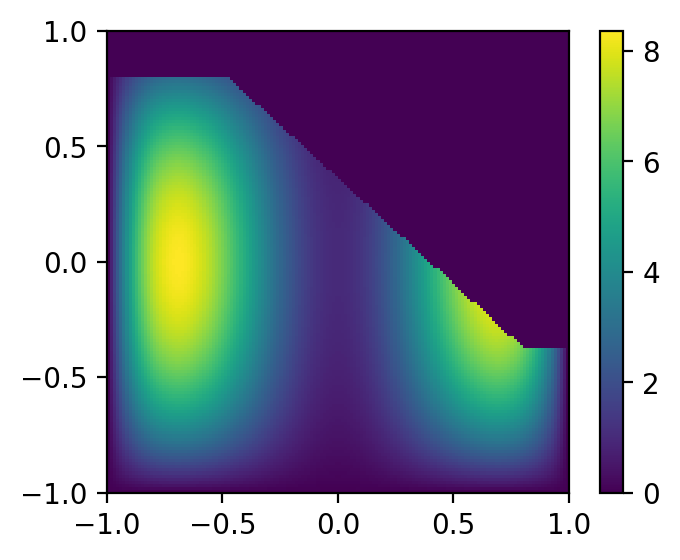} \\\includegraphics[width=0.3\linewidth]{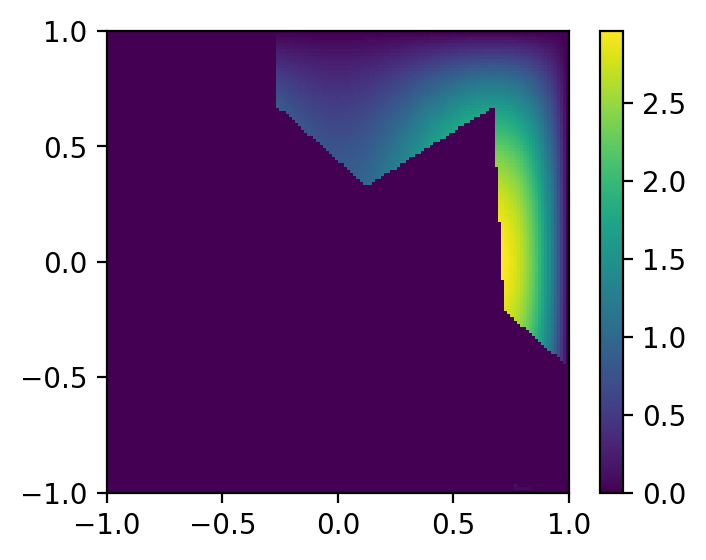} &
  \includegraphics[width=0.3\linewidth]{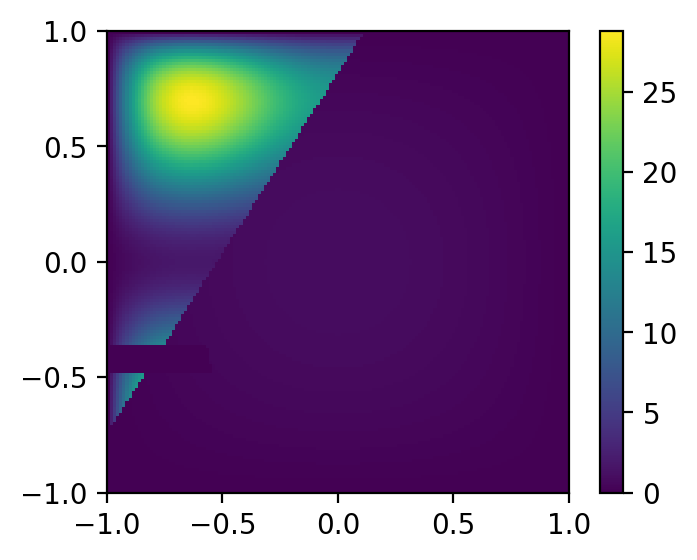} &
  \includegraphics[width=0.3\linewidth]{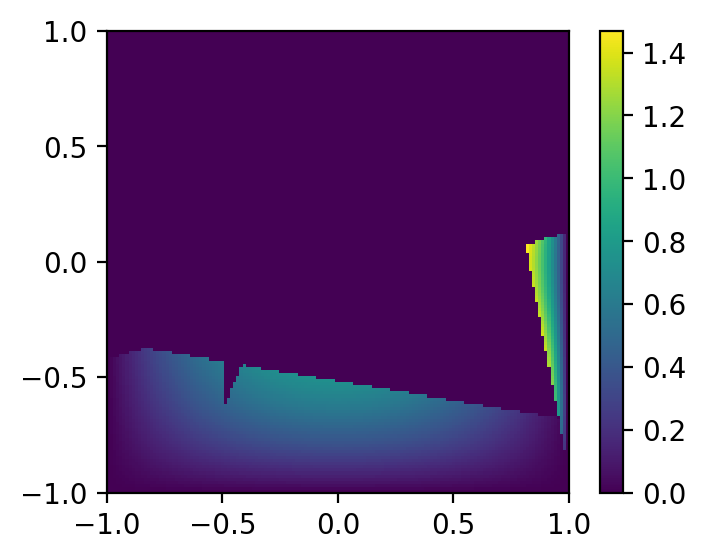} \\
\end{tabular}
\caption{Random examples of the resulting density for problems of shape \emph{PATH} in the $2d$-case.}
\label{fig:tree_primal_density_path}
\end{figure}

\begin{figure}[ht]
\centering
\begin{tabular}{ccc}
  \includegraphics[width=0.32\linewidth]{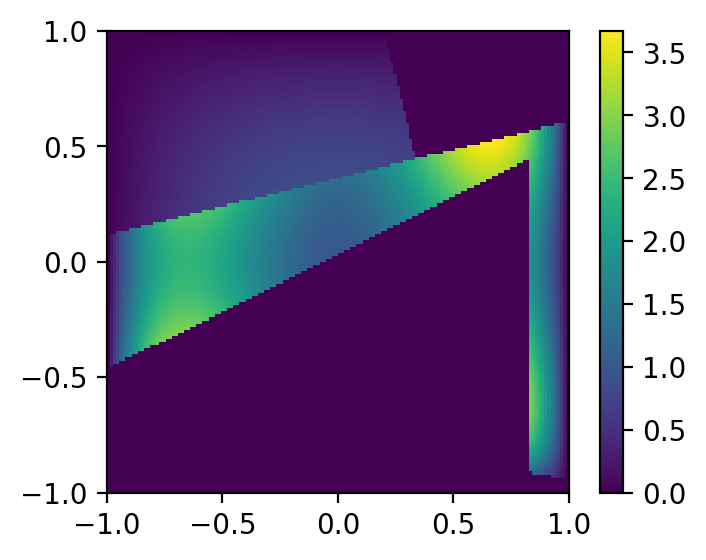} &
  \includegraphics[width=0.32\linewidth]{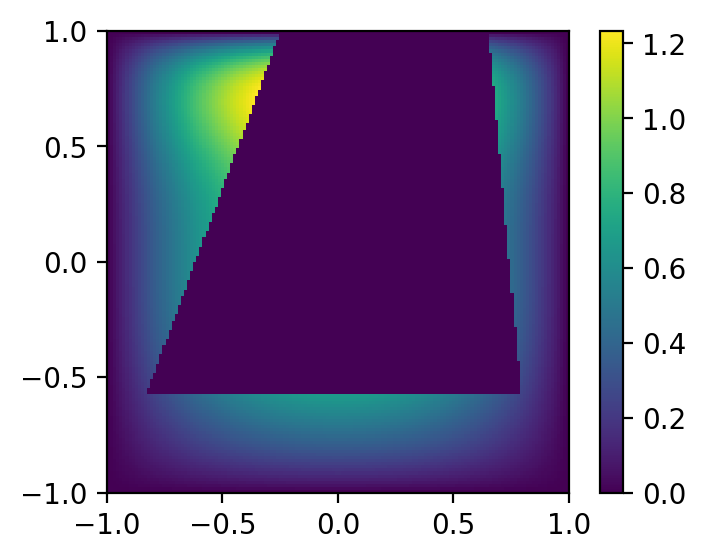} &
  \includegraphics[width=0.32\linewidth]{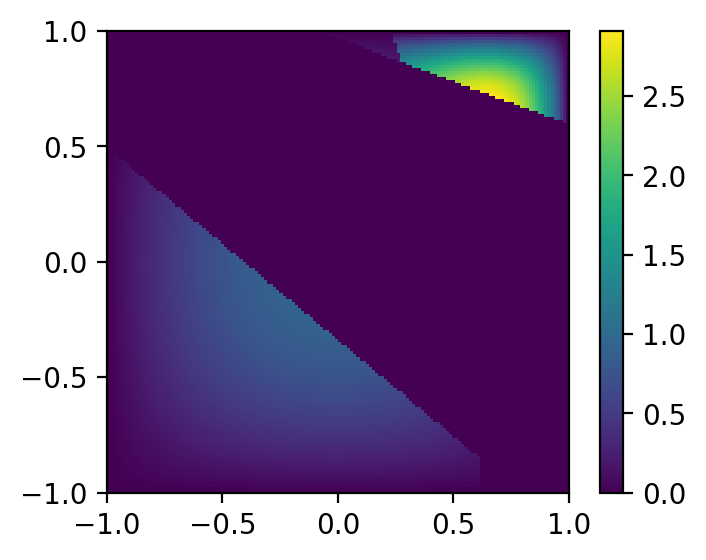} \\
  \includegraphics[width=0.32\linewidth]{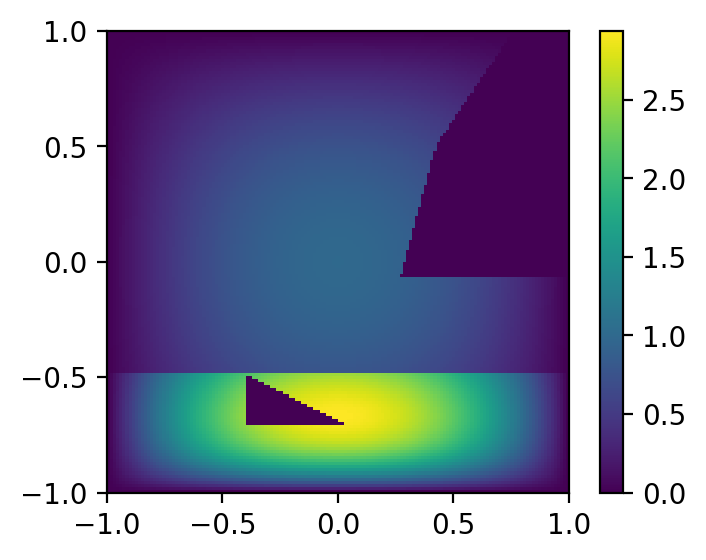} &
  \includegraphics[width=0.32\linewidth]{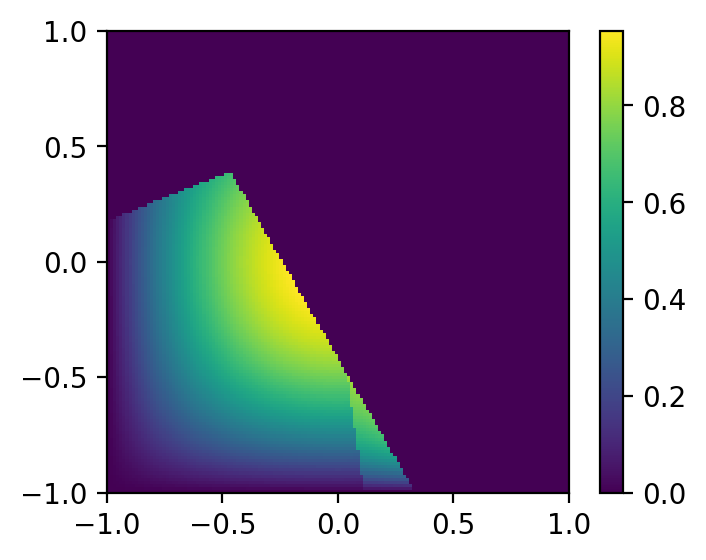} &
  \includegraphics[width=0.32\linewidth]{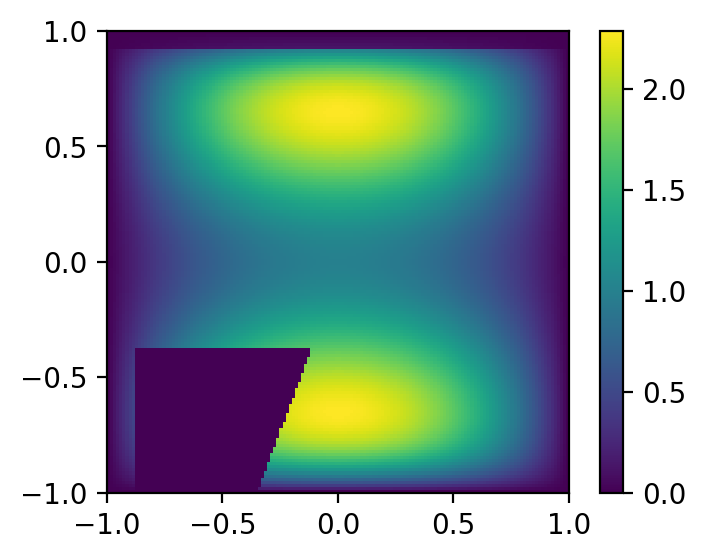} \\\includegraphics[width=0.32\linewidth]{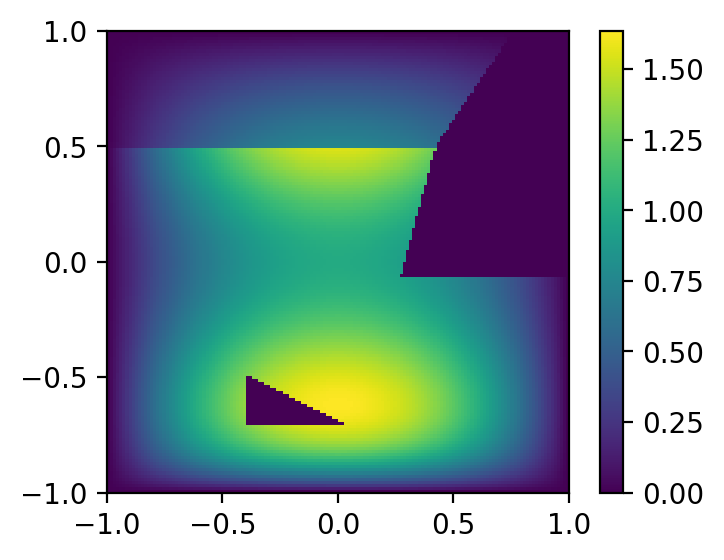} &
  \includegraphics[width=0.32\linewidth]{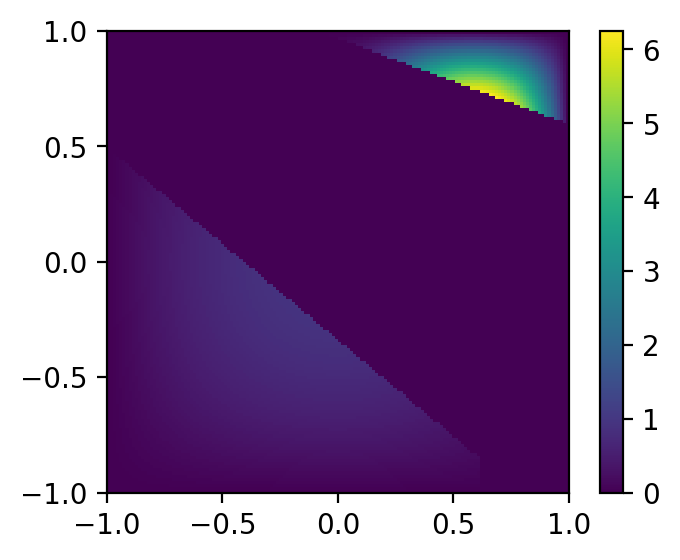} &
  \includegraphics[width=0.32\linewidth]{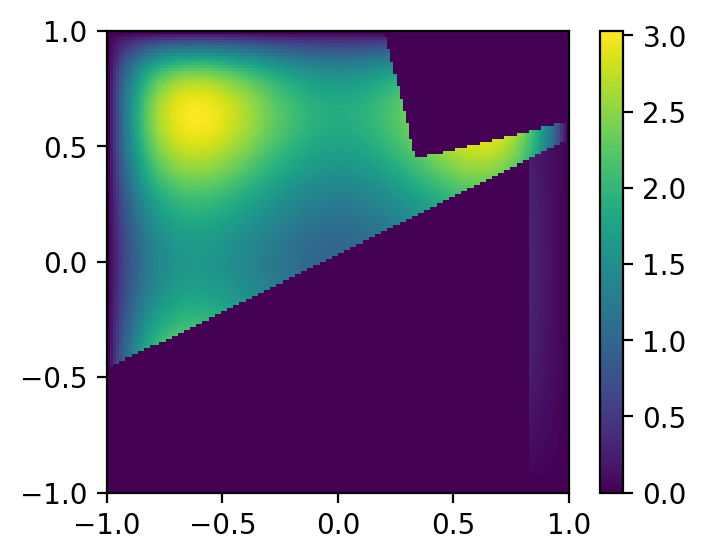} \\
\end{tabular}
\caption{Random examples of the resulting density for problems of shape \emph{SNOW} in the $2d$-case.}
\label{fig:tree_primal_density_snow}
\end{figure}

\begin{figure}[ht]
\centering
\begin{tabular}{ccc}
  \includegraphics[width=0.32\linewidth]{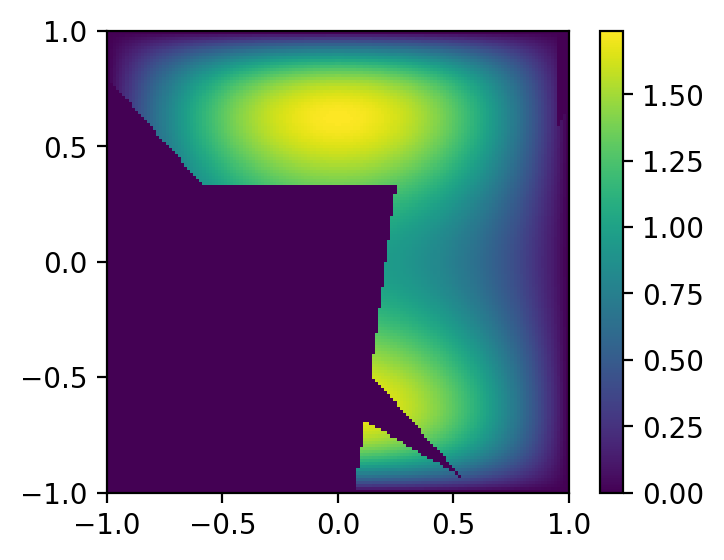} &
  \includegraphics[width=0.32\linewidth]{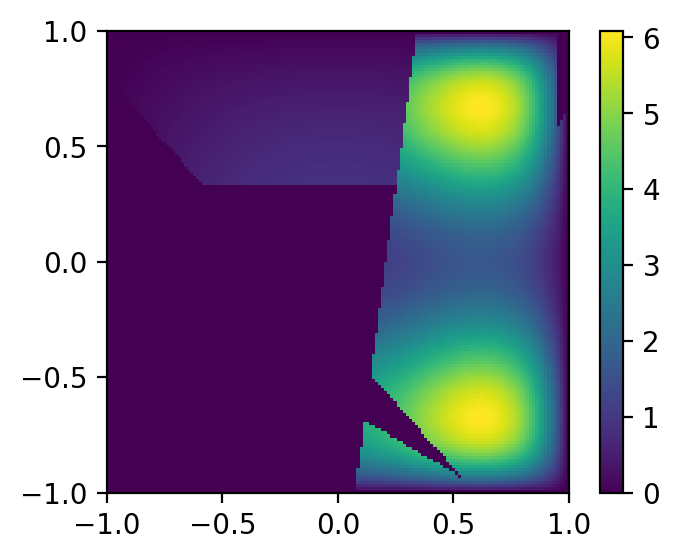} &
  \includegraphics[width=0.32\linewidth]{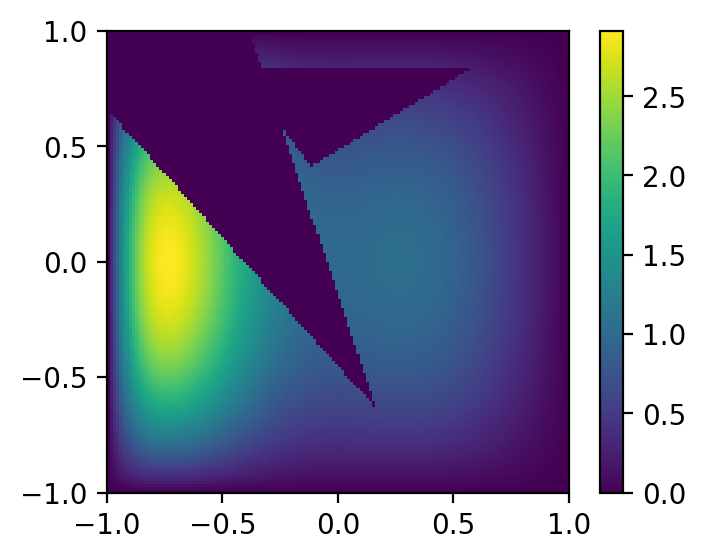} \\
  \includegraphics[width=0.32\linewidth]{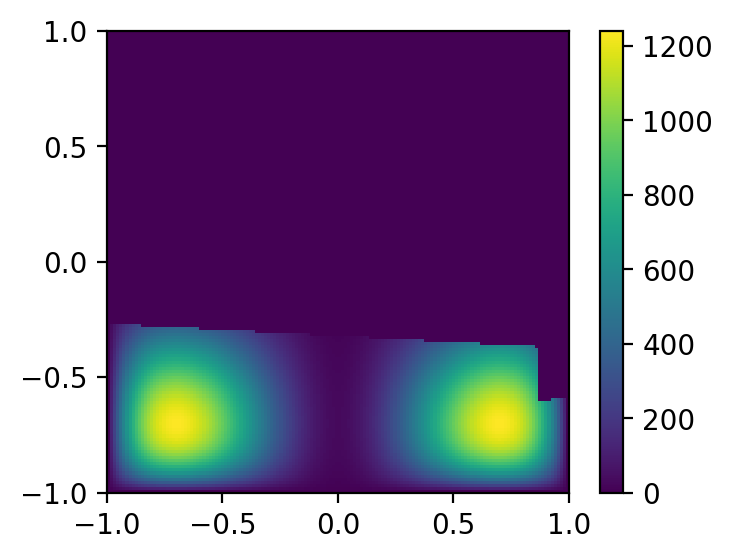} &
  \includegraphics[width=0.32\linewidth]{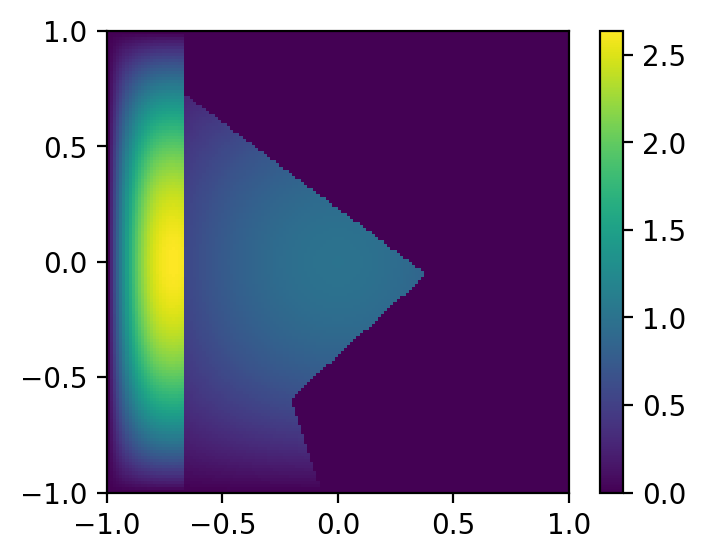} &
  \includegraphics[width=0.32\linewidth]{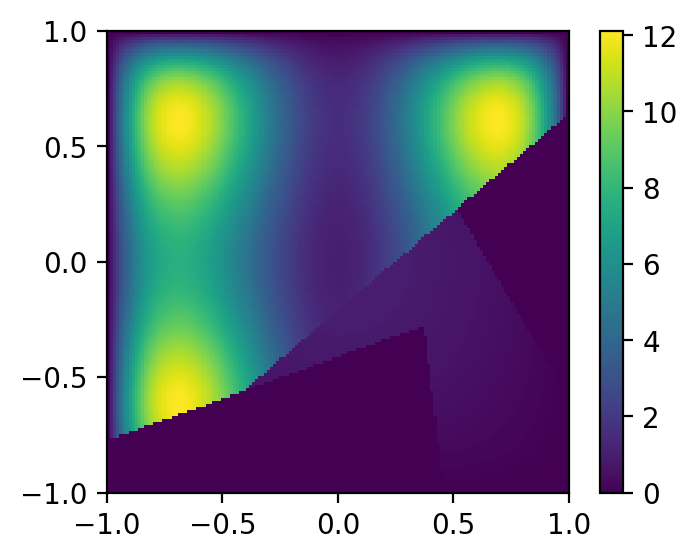} \\\includegraphics[width=0.32\linewidth]{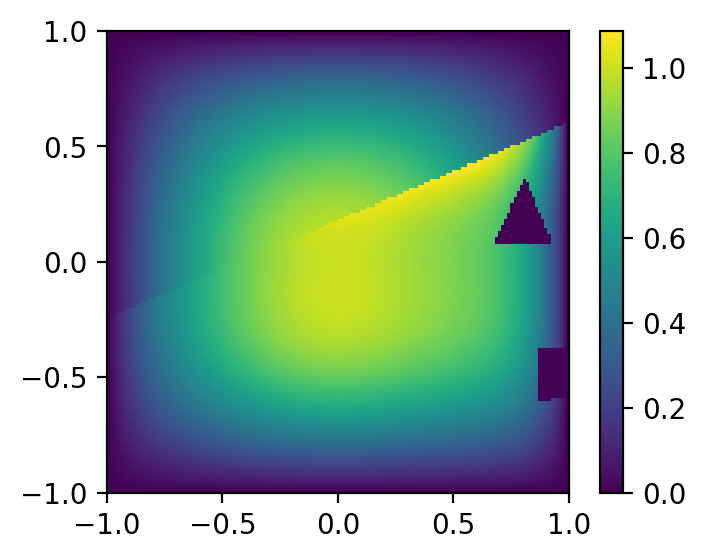} &
  \includegraphics[width=0.32\linewidth]{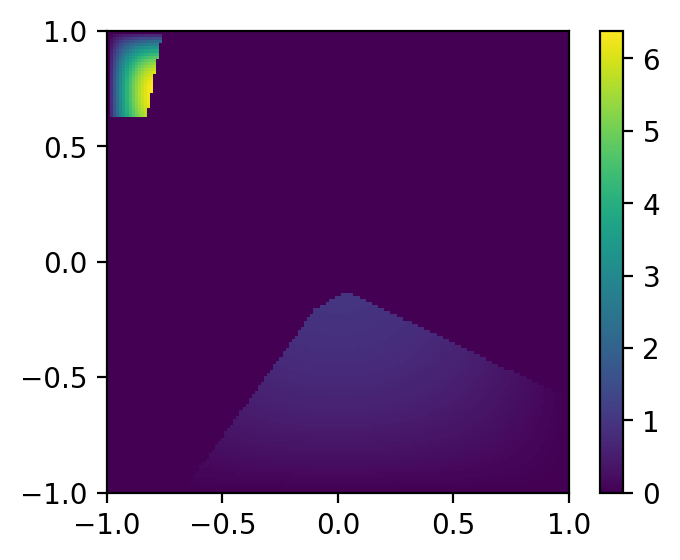} &
  \includegraphics[width=0.32\linewidth]{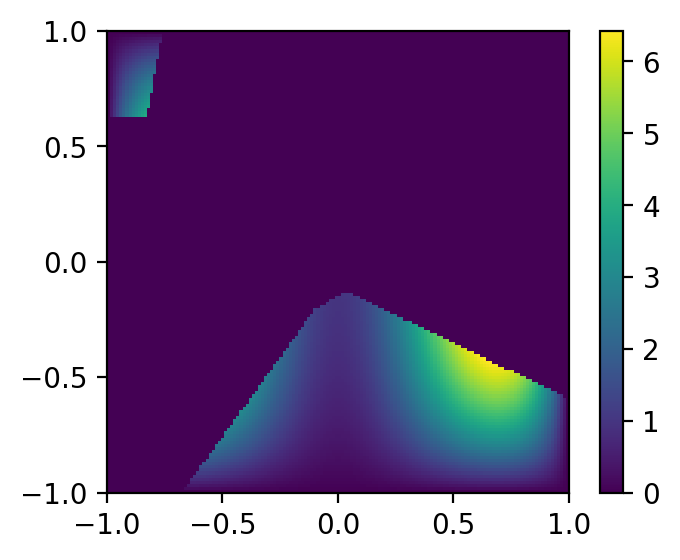} \\
\end{tabular}
\caption{Random examples of the resulting density for problems of shape \emph{STAR} in the $2d$-case.}
\label{fig:tree_primal_density_star}
\end{figure}

\FloatBarrier

\subsection{Hyperparameters for the optimizers for STAR, SNOW and PATH}
\label{sec:appendix_hyperparames_opt_tree_primal}

In general, the optimizers we compare the MP-MAP to are run with increasing budget until either we have a time-out or we have have a relative error of $0.01$ compared to the grid-search baseline.

\textbf{PA-SHGO} We run PA-SHGO with the following hyperparameters:
\begin{itemize}
    \item \textbf{local-optimiers}  cobyla
    \item \textbf{budget (iters)} 1 2 4 16 32 64 128 256
\end{itemize}

\textbf{ParticleAdam} We run ParticleAdam with the following hyperparameters:
\begin{itemize}
    \item \textbf{lr} 0.001 
    \item \textbf{max-iterations} 2500
    \item \textbf{budget (particles)} 100 1000 10000 100000 1000000
\end{itemize}

\textbf{Baseline Recursive Grid-Search} We run the baseline (grid search with a recursion into the best found point) with the following hyperparameters:
\begin{itemize}
    \item \textbf{grid size (per dim)} 10
    \item \textbf{precision} double
    \item \textbf{number of recursions} 30
    \item \textbf{recursion shrink factor} 0.2 (shrinks the domain in the recursion)
\end{itemize}
This grid search is run per polytope enumerted, via the domain of the bounding box of the poytope. It is therefure used as an local optimizer in \cref{alg:pa_map}. It constitutes a simple global search baseline that is intentionally easy to outperform. As \mpmap{} returns the global optimum, it always beats the this baseline and is therefore always accepted.

\subsection{Results on STAR, SNOW, PATH}

\begin{table}[t]
\centering
\begin{tabular}{llll}
\toprule
 & MP-MAP & PA-SHGO & ParticleAdam \\
Num Variables &  &  &  \\
\midrule
2 & 0.05 \(\pm\) 0.03 & \textbf{0.01 }\(\pm\) 0.00 & 5.03 \(\pm\) 9.44 \\
4 & \textbf{0.14 }\(\pm\) 0.12 & 0.33 \(\pm\) 0.36 & 11.90 \(\pm\) 12.80 \\
6 & \textbf{1.77 }\(\pm\) 6.05 & 5.07 \(\pm\) 9.57 & 22.74 \(\pm\) 11.49 \\
8 & \textbf{3.85 }\(\pm\) 8.25 & 15.51 \(\pm\) 14.51 & 18.79 \(\pm\) 12.20 \\
10 & \textbf{8.46 }\(\pm\) 12.48 & 23.27 \(\pm\) 11.40 & 23.28 \(\pm\) 10.47 \\
12 & \textbf{5.25 }\(\pm\) 9.55 & 30.00 \(\pm\) 0.00 & 23.45 \(\pm\) 9.72 \\
14 & \textbf{7.05 }\(\pm\) 10.27 & - & - \\
16 & \textbf{6.56 }\(\pm\) 10.32 & - & - \\
18 & \textbf{12.01 }\(\pm\) 12.84 & - & - \\
20 & \textbf{10.46 }\(\pm\) 12.87 & - & - \\
22 & \textbf{17.01 }\(\pm\) 14.14 & - & - \\
24 & \textbf{13.34 }\(\pm\) 13.01 & - & - \\
26 & \textbf{18.90 }\(\pm\) 13.25 & - & - \\
\bottomrule
\end{tabular}

\caption{Runtime comparison on the SNOW-dataset in mean and standard derivations for our approaches.}
\label{tab:tree_primal_results_snow_long}
\end{table}

\begin{table}[t]
\centering
\begin{tabular}{llll}
\toprule
 & MP-MAP & PA-SHGO & ParticleAdam \\
Num Variables &  &  &  \\
\midrule
2 & 0.07 \(\pm\) 0.06 & \textbf{0.02 }\(\pm\) 0.01 & 3.85 \(\pm\) 7.88 \\
4 & 1.39 \(\pm\) 5.97 & \textbf{0.29 }\(\pm\) 0.30 & 10.83 \(\pm\) 12.31 \\
6 & \textbf{1.53 }\(\pm\) 6.07 & 2.83 \(\pm\) 3.50 & 17.19 \(\pm\) 12.82 \\
8 & \textbf{1.70 }\(\pm\) 5.92 & 10.34 \(\pm\) 11.15 & 22.50 \(\pm\) 11.35 \\
10 & \textbf{1.83 }\(\pm\) 5.90 & 19.59 \(\pm\) 11.39 & 24.26 \(\pm\) 9.95 \\
12 & \textbf{2.01 }\(\pm\) 6.13 & 25.70 \(\pm\) 9.35 & 23.29 \(\pm\) 10.16 \\
14 & \textbf{0.84 }\(\pm\) 0.57 & - & - \\
16 & \textbf{7.56 }\(\pm\) 12.20 & - & - \\
18 & \textbf{3.56 }\(\pm\) 8.20 & - & - \\
20 & \textbf{2.68 }\(\pm\) 6.21 & - & - \\
22 & \textbf{11.08 }\(\pm\) 13.48 & - & - \\
24 & \textbf{10.05 }\(\pm\) 12.93 & - & - \\
26 & \textbf{13.69 }\(\pm\) 13.85 & - & - \\
\bottomrule
\end{tabular}

\caption{Runtime comparison on the STAR-dataset in mean and standard derivations for our approaches.}
\label{tab:tree_primal_results_star_long}
\end{table}

\begin{table}[t]
\centering
\begin{tabular}{llll}
\toprule
 & MP-MAP & PA-SHGO & ParticleAdam \\
Num Variables &  &  &  \\
\midrule
2 & 0.04 \(\pm\) 0.02 & \textbf{0.02 }\(\pm\) 0.00 & 9.79 \(\pm\) 12.97 \\
4 & 1.08 \(\pm\) 2.22 & \textbf{0.19 }\(\pm\) 0.16 & 14.75 \(\pm\) 13.37 \\
6 & \textbf{6.56 }\(\pm\) 9.54 & 8.18 \(\pm\) 10.70 & 19.04 \(\pm\) 12.66 \\
8 & \textbf{11.54 }\(\pm\) 12.55 & 18.81 \(\pm\) 11.81 & 17.78 \(\pm\) 12.23 \\
10 & \textbf{12.08 }\(\pm\) 12.50 & 21.57 \(\pm\) 11.62 & 23.40 \(\pm\) 10.28 \\
12 & \textbf{14.55 }\(\pm\) 12.91 & 30.00 \(\pm\) 0.00 & 21.07 \(\pm\) 10.58 \\
14 & \textbf{17.97 }\(\pm\) 13.14 & - & - \\
16 & \textbf{12.51 }\(\pm\) 12.55 & - & - \\
18 & \textbf{20.05 }\(\pm\) 12.93 & - & - \\
20 & \textbf{20.89 }\(\pm\) 11.91 & - & - \\
22 & \textbf{18.78 }\(\pm\) 11.85 & - & - \\
24 & \textbf{22.71 }\(\pm\) 10.17 & - & - \\
26 & \textbf{24.38 }\(\pm\) 8.60 & - & - \\
\bottomrule
\end{tabular}

\caption{Runtime comparison on the PATH-dataset in mean and standard derivations for our approaches.}
\label{tab:tree_primal_results_path_long}
\end{table}

\FloatBarrier

\subsection{Details for SDD experiments}%
\label{sec:app_details_sdd_experiments}

We provide additional experiments on the SDD dataset, and to the configuration of the optimization algorithms used in our experiments.
We recall that for this dataset, we learn the densities with PAL~\cite{kurscheidt2025probabilistic}. PAL densities are non-negative piecewise polynomials, each piece being an axis-aligned box 
$\Phi = [l_1, u_1] \times \cdots \times [l_N, u_N] \subseteq \mathbb{R}^N$
over which the density is defined as
\begin{equation}%
    \label{eq:pal_density}
    p_\Phi(\vx) = \sum_{i=1}^d \alpha_i \prod_{j=1}^N s_j(x_j)^2,
\end{equation}
where each $s_j(x_j)$ is a cubic polynomial. 
As such, they can be seen as simple squared probabilistic circuits (PCs) \citep{choi2020pc,vergari2021compositional} which have been recently investigated in the PC literature for their expressiveness properties
\citep{loconte2024subtractive,loconte2025sum,loconte2025square}.

\subsubsection{\pcadam}%
\label{sec:app_details_particle_adam}

In order to better understand the performance \pcadam{}, we conduct experiments over different hyperparameter configurations. 
We fix the learning rate to $0.1$ and vary:
\begin{itemize}
    \item $N$: number of parallel particles used in the optimization;
    \item $\mathit{it}$: number of iterations for each particle;
\end{itemize}
The results are shown in \cref{fig:sdd_ablation_optimizers}.
As expected, all configurations report feasible solutions, but the performance varies significantly. 
Since the SDD dataset is low-dimensional (2D) and the feasible region is broad, \pcadam{} can find good solutions if enough particles and iterations are used.
However, this comes at the cost of increased computation time.

\begin{figure}
    \centering
    \begin{subfigure}[T]{0.36\linewidth}
        \centering
        \includegraphics[width=.9\linewidth]{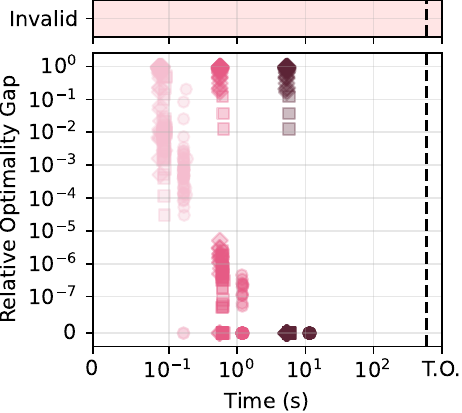}
        \includegraphics[width=\linewidth]{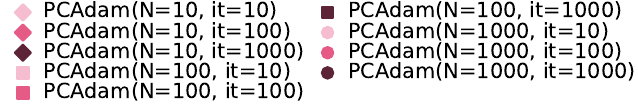}
        \caption{\pcadam{} configurations varying number of particles and iterations on SDD.}%
        \label{fig:sdd_pcadam_performance}
    \end{subfigure}
    \begin{subfigure}[T]{0.61\linewidth}
        \centering
        \includegraphics[width=.53\linewidth]{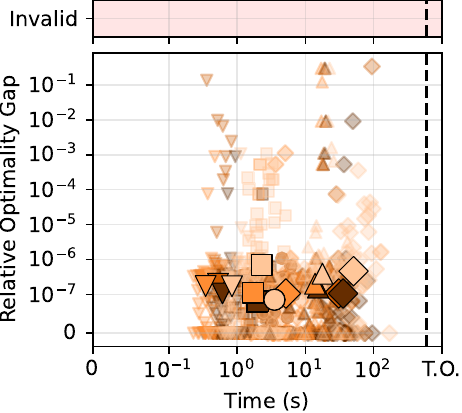}
        \includegraphics[width=\linewidth]{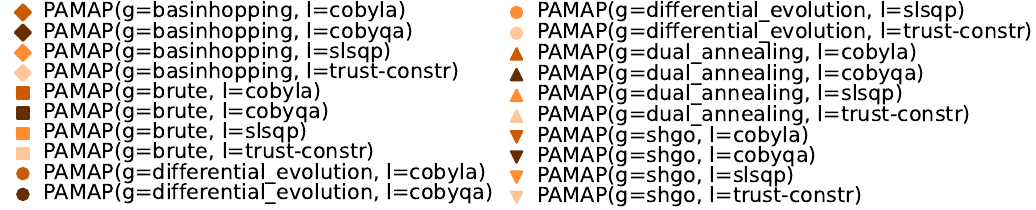}
        \caption{\pamap{} with differenc numerical optimizers from Scipy  on SDD. Since here the best configuration is not eye-catching, we also show the median point for each configuration.}%
        \label{fig:sdd_scipy_performance}
    \end{subfigure}
    \caption{Experiments on SDD dataset varying \pcadam{} hyperparameters (left) and \pamap{} numerical optimizers (right).}%
    \label{fig:sdd_ablation_optimizers}
\end{figure}

\subsubsection{\pamap}%
\label{sec:app_details_pa_map}

\paragraph{Enumeration and upper bounds.}
For these densities, we implement the enumerator as follows.
First, we compute the upper bound for each spline piece (see below).
Then we sort the pieces in decreasing order of their upper bounds, so to increase the chances of pruning suboptimal pieces early.
Finally, we enumerate polytopes corresponding to each piece in turn, terminating when the upper bound of the current piece is lower than the current best value.

An upper bound of the maximum of~\eqref{eq:pal_density} over a piece $\Phi$ can be computed as follows:
\begin{align}
\max_{\vx \models \Phi} p_\Phi(\vx)
&\leq \sum_{i=1}^d \alpha_i  \cdot \max_{\vx \models \Phi} \prod_{j=1}^N s_j(x_j)^2\\
&= \sum_{i=1}^d \alpha_i \prod_{j=1}^N \max_{x_j \in [l_j, u_j]} s_j(x_j)^2
\label{eq:bound_above_nd}
\end{align}

The critical points of $s_j(x_j)^2$ are the same as those of $s_j(x_j)$, which can be computed in closed form. 
Thus, we evaluate $s_j(x_j)$ at its critical points within $[l_j, u_j]$, as well as at the interval boundaries $l_j$ and $u_j$. 
The maximum squared value among these points provides the desired bound.
\cref{fig:trajectory_pruning_grid} shows how pruning works in practice on 10 sample trajectories from the SDD dataset.

\begin{figure*}[t]
    \centering
    \newcommand{\trajgroup}[2]{%
        \begin{minipage}{0.48\textwidth}
            \centering
            \begin{subfigure}{0.3\linewidth}
                \centering
                \includegraphics[width=\linewidth]{figures/upper-bound/trajectory_#1_densities.pdf}
            \end{subfigure}
            \begin{subfigure}{0.3\linewidth}
                \centering
                \includegraphics[width=\linewidth]{figures/upper-bound/trajectory_#1_upper_bound.pdf}
            \end{subfigure}
            \begin{subfigure}{0.3\linewidth}
                \centering
                \includegraphics[width=\linewidth]{figures/upper-bound/trajectory_#1_densities_shgo_polytopes.pdf}
            \end{subfigure}
            \caption*{Traj. #1. Enumerated polytopes: #2/257}
        \end{minipage}%
    }

    \trajgroup{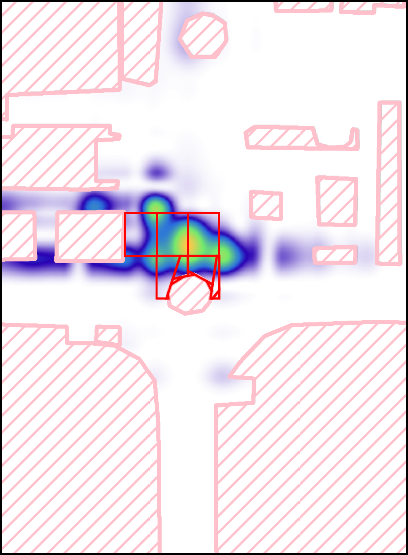}{12}\hfill\trajgroup{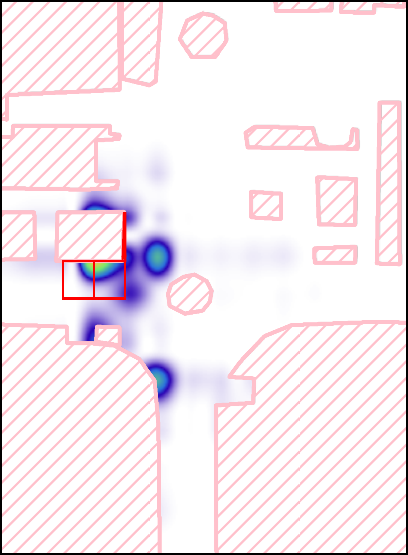}{4}
    \vspace{0.4cm}

    \trajgroup{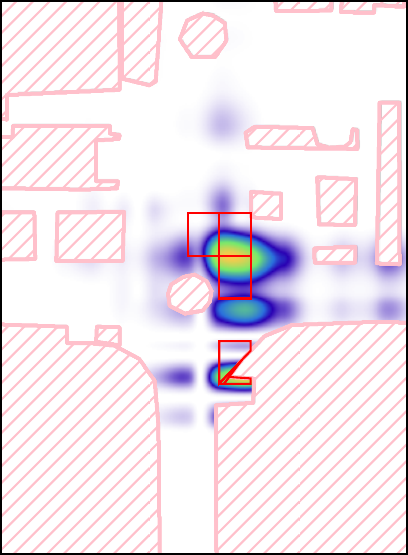}{6}\hfill\trajgroup{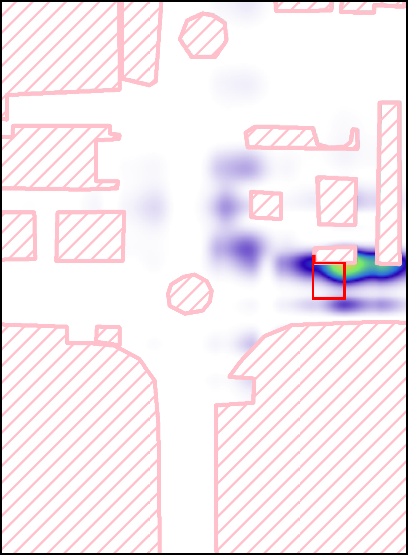}{2}
    \vspace{0.4cm}

    \trajgroup{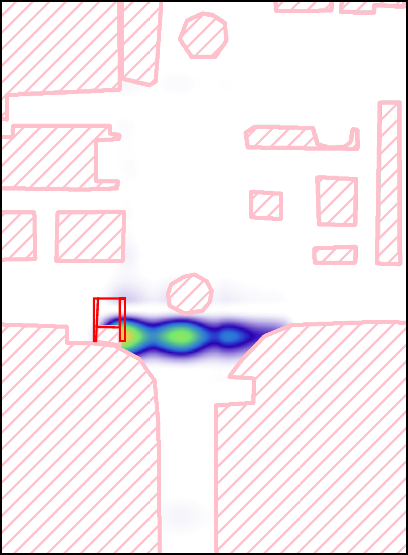}{3}\hfill\trajgroup{17000}{9}
    \vspace{0.4cm}

    \trajgroup{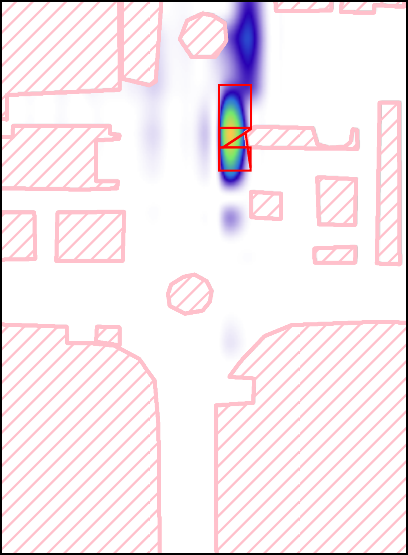}{6}\hfill\trajgroup{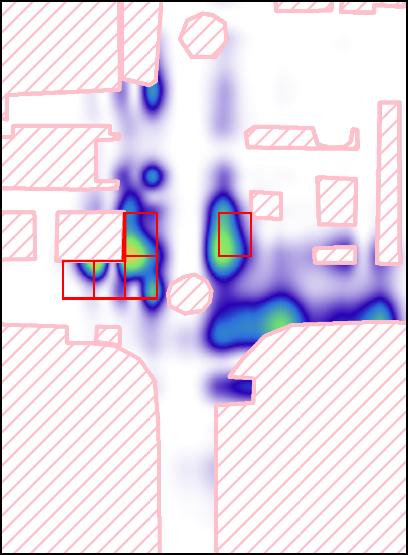}{7}
    \vspace{0.4cm}

    \trajgroup{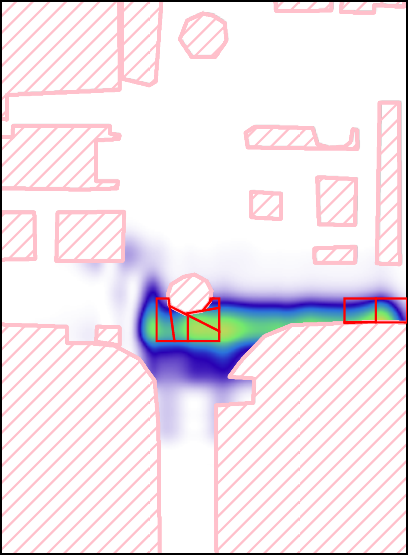}{10}\hfill\trajgroup{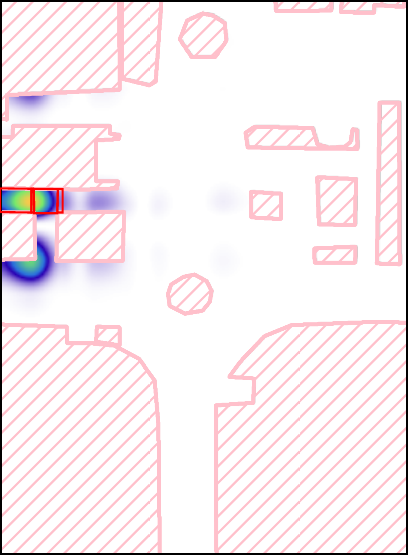}{4}

    \caption{Visualization of the \pamap{} pruning process across 10 sample trajectories. Each group shows the predictive density (left), the computed upper bounds used for pruning (center), and the polytopes analyzed by \pamap(\shgo) thanks to pruning (right).}%
    \label{fig:trajectory_pruning_grid}
\end{figure*}

\paragraph{SciPy's optimizers.}
For \pamap{}(SciPy), we tried all the combinations of global and local constrained optimizers available in SciPy's optimize library. The results are shown in \cref{fig:sdd_scipy_performance}.

From the plot, we can see that the fastest configuration is the combination of \shgo{} as global optimizer and SLSQP as local optimizer, achieving comparable relative error as the other configurations, hence the one we used in our experiments.

\subsubsection{OMT encoding}%
\label{sec:app_details_omt_encoding}
OMT(\Tnlarat) solvers require the optimization problem to be encoded as a pair of a logical formula and an objective function encoded as an \Tnlarat-{term}.
While the logical formula is directly given by the problem, multiple choices are possible for encoding the objective function~\eqref{eq:pal_density}.

A first possibility is to encode it as a nested if-then-else expression:
\begin{equation}
    \text{Ite}(\Phi_1, p_{\Phi_1}(\vx), \text{Ite}(\Phi_2, p_{\Phi_2}(\vx), \ldots))
\end{equation}
where each $\Phi_i$ is a piece of the PAL density.
An alternative is to encode the objective as a sum of if-then-else expressions:
\begin{equation}
    \sum_{\Phi_i} \text{Ite}(\Phi_i, p_{\Phi_i}(\vx), 0).
\end{equation}
To ensure that this encoding is well defined, the regions $\Phi_i$ must be mutually exclusive. We enforce this by defining each region using left-closed, right-open intervals, except for the final region along each dimension, which is closed on the right as well.

We tested both encodings using OptiMathSAT and CDCL-OCAC. The first encoding was challenging for both OMT solvers, and both timed out without finding a solution. With the second encoding, OptiMathSAT in anytime mode occasionally found a solution within the time limit, exiting with an error the remaing times. In contrast, CDCL-OCAC timed out on both encodings without finding any solution.

\FloatBarrier

\subsection{Details For Constrained MAP-prediction For Imputation On Tabular Data}
\label{app:details_tabular_dataset}

\subsubsection{The TVAE Optimization Objective}

In order to train a TVAE \cite{DBLP:conf/nips/XuSCV19}, our optimization objective is the common ELBO-style optimization objective extended to handle categorical data:

\begin{align}
    \mathsf{loss\text{-}item}(x_i, \hat{x}_i, \hat{\sigma}_i) &= 
    \begin{cases}
  \frac{1}{2\hat{\sigma_i^2}}(x_i - \hat{x}_i)^2 - \log \sigma_i, & \text{if } \neg (\mathsf{is\text{-}categorical}(x_i))\\
  \mathsf{CrossEntropy}(x_i, \hat{x}_i)
    \end{cases}\\
  \mathcal{L}_{rec}(\vx, \hat{\vx},\hat{\sigma}) &= \frac{1}{N}\sum_{b=1}^N \sum_i \mathsf{loss\text{-}item}(x_i^b, \hat{x}_i^b, \hat{\sigma}_i^b)\\
  \mathcal{L}_{KL}(\vmu_{latent}, \vsigma_{latent}) &= \frac{1}{N}(\sum_{b=1}^N - \frac{1}{2}\sum_i (1 + \log (((\sigma_{latent})^b_i)^2) - ((\mu_{latent})^b_i)^2 -  ((\sigma_{latent})^b_i)^2))\\
   \mathcal{L}(\vx) &= \mathcal{L}_{rec}(\vx, \hat{\vx}) + \mathcal{L}_{KL}(\vmu_{latent}, \vsigma_{latent})\\
   \text{where } & (\vmu_{latent}, \vsigma_{latent}) = \mathsf{encoder}_{NN}(\vx)\\
   & \vz \sim \mathcal{N}(\vmu_{latent}, \vsigma_{latent})\\
   & (\hat{\vx},\hat{\sigma}) = \mathsf{decoder}_{NN}(\vz)
\end{align}

In order compute $\mathit{cMAP}(p, \formula) = \argmax_{\vx\models\formula}{p(\vx)}$, we have to decide for a $p$ to optimize. We essentially have two options: We can directly optimize the ELBO-relaxation of $p$ that is detailed above, which is a stochastic objective, or we can optimize over $\vx$ and $\vz$ jointly and then discard $\vz$. The second objective is a common starting point to escape the stochastic nature of the first \cite{gonzalez2022solving}.

\subsubsection{Discrete Variables}

Similar to \citet{stoianbeyond}, we treat the discrete variables as continuous from the point of view of our model.

\FloatBarrier

\subsubsection{Results}
\label{app:tabular_results}

We benchmark our methods on the \emph{House-Price} prediction dataset with the constraints provided by \citet{stoianbeyond}. We train an TVAE-model according to the hyperparameters provided by \citet{stoianbeyond}, so with $150$ epochs, batch size $70$, l2scale $0.0002$, learning rate $0.0002$ (we use Adam) and loss-factor $2$.

In order to perform our MAP-prediction we use $100$-samples from the latent in order to estimate $p(x_m)$, and use a learning rate of $0.1$. We provide detailed results in table \ref{tab:imputation_stats}.

In order to generate the starting-points, for the unconstrained baselines we sample from the model and for Pa(PCAdam) we first sample unconstrained and then project into the current enumerated polytope. 

\begin{table}
\centering
\begin{tabular}{lcrrrrr}
\toprule
Method
& Particles
& Mean
& Median
& Std.\ Dev.
& Trimmed Mean (5\%)
& Sec./Sample \\
\midrule
Pa(PCAdam) - constrained
& 10
& 5.71
& 0.0023
& 51.71
& 0.16
& 20.33 \\

Adam - unconstrained
& 10
& 42.68
& 0.262
& 220.10
& 11.19
& 4.17 \\

Adam - unconstrained
& 100
& 33.30
& 0.253
& 196.33
& 7.72
& 4.48 \\
\bottomrule
\end{tabular}
\caption{Imputation error statistics (aggregated over the dataset). We use a two-sided $5\%$ trimmed mean, so with the highest $2.5\%$ and lowest $2.5\%$ of our results removed.}
\label{tab:imputation_stats}
\end{table}

\end{document}